\documentclass[pdflatex,iicol, sn-mathphys-num]{sn-jnl}

\usepackage{graphicx}%
\usepackage{multirow}%
\usepackage{amsmath,amssymb,amsfonts}%
\usepackage{amsthm}%
\usepackage{mathrsfs}%
\usepackage[title]{appendix}%
\usepackage[table]{xcolor}
\usepackage{textcomp}%
\usepackage{manyfoot}%
\usepackage{booktabs}%
\usepackage{algorithm}%
\usepackage{algorithmicx}%
\usepackage{algpseudocode}%
\usepackage{listings}%

\usepackage[caption=false,font=footnotesize]{subfig}%
\usepackage{stfloats}%
\usepackage{threeparttable}%
\usepackage{colortbl}%
\usepackage{siunitx}%
\usepackage{microtype}%
\usepackage{bm}%

\newtheorem{proposition}{Proposition}%

\begin{document}

\title[Article Title]{Generic Calibration: Pose Ambiguity/Linear Solution and Parametric-hybrid Pipeline}


\author[1]{\fnm{Yuqi} \sur{Han}}\email{hanyuqi@sjtu.edu.cn}
\equalcont{These authors contributed equally to this work.}

\author[1]{\fnm{Qi} \sur{Cai}}\email{qicaicn@gmail.com}
\equalcont{These authors contributed equally to this work.} 

\author*[1]{\fnm{Yuanxin} \sur{Wu}}\email{yuanx\_wu@hotmail.com}

\affil*[1]{\orgdiv{Shanghai Key Laboratory of Navigation and Location-based Services, School of Automation and Intelligent Sensing}, \orgname{Shanghai Jiao Tong University}, \orgaddress{\city{Shanghai}, \postcode{200240}, \country{China}}}








\abstract{Offline camera calibration techniques typically employ parametric or generic camera models. Selecting parametric models relies heavily on user experience, and an inappropriate camera model can significantly affect calibration accuracy. Meanwhile, generic calibration methods involve complex procedures and cannot provide traditional intrinsic parameters. This paper reveals a pose ambiguity in the pose solutions of generic calibration methods that irreversibly impacts subsequent pose estimation. A linear solver and a nonlinear optimization are proposed to address this ambiguity issue. Then a global optimization hybrid calibration method is introduced to integrate generic and parametric models together, which improves extrinsic parameter accuracy of generic calibration and mitigates overfitting and numerical instability in parametric calibration. Simulation and real-world experimental results demonstrate that the generic-parametric hybrid calibration method consistently excels across various lens types and noise contamination, hopefully serving as a reliable and accurate solution for camera calibration in complex scenarios.}




\maketitle

\section{Introduction}\label{sec1}

As a foundational task in visual computing, camera calibration establishes the relationship between 3D points in camera coordinates and their corresponding 2D image points, and significantly influences all subsequent computational processes \cite{Schops2020}. Consequently, the accuracy and robustness of camera calibration are crucial for 3D vision computations.

Camera models are generally classified into parametric \cite{Zhang2000,Brown1971,Kannala2006,Scaramuzza2006,Urban2015,Mei2007,Usenko2018,Khomutenko2016} and non-parametric \cite{Pan2022,Schops2020,Dunne2007Efficient,Rosebrock2012Generic,Rosebrock2012Complete} types. Parametric models, characterized by concise mathematical representations, are less susceptible to variations in coverage, density, and detection accuracy of image corner points \cite{Hagemann2022} and thus have been implemented in popular platforms such as OpenCV and MATLAB. Parametric calibration methods sequentially consist of selecting a model, estimating initial extrinsic parameters, and minimizing reprojection error on the image plane by Bundle Adjustment (BA) for intrinsic and extrinsic parameters. The performance of parametric methods largely depends on the compatibility between the camera model and the actual lens \cite{Hagemann2022}. Although the incompatibility can be remedied by high-order polynomials or complex functions, they complicate BA and lead to overfitting and numerical instability \cite{Guoqing1994,Hartley2007,Lochman2021}. Expertise in the field is necessary to select suitable models for mitigating systematic errors and instability. Calibration results may vary significantly across different image sets even from a fixed-focus camera \cite{Hagemann2022}. A number of methods have been developed to evaluate the impact of different parametric models on calibration quality and to investigate model selection strategies \cite{Tang2017,Larsson2019,Polic2020}. For example,  the stability of parametric calibration methods was assessed in \cite{Davison1997,Hartley2003,Beck2018,Hagemann2022} and statistical guidance was used to assist users in capturing optimal images for enhanced calibration reliability \cite{Rojtberg2018,Peng2019,Ren2021}.

In recent years, non-parametric models, particularly generic camera models, have gained significant attention due to their greater parameter flexibility. Generic camera models can provide more accurate and unbiased estimation for a variety of cameras \cite{Schops2020}, particularly for highly distorted cameras \cite{Dunne2007comparison}. Generic calibration methods do not rely on any specific parametric model. Instead, they use lookup tables to establish the mapping between control point image coordinates and observed rays, and employ interpolation methods such as B-splines \cite{Beck2018} to estimate rays for non-control points. However, the calibration pipeline and representation of intrinsic parameters in generic calibration are relatively complex. In 2020, Schöps \cite{Schops2020} released the first open-source generic calibration toolbox, which involves complex polynomial calculations and relies on iterative incremental optimization.

Through this work, we attempt to combine advantages of both parametric and generic calibration methods. The main technical contributions include: 
\begin{itemize}
\item{Identifying a pose ambiguity problem in generic calibration, analyzing its impact on multi-view pose estimation, and proposing a linear solver and a nonlinear optimization to resolve it, thereby enabling the accurate application of generic calibration in subsequent applications such as Structure-from-Motion (SfM), Simultaneous Localization and Mapping (SLAM), and Perspective-n-Point (PnP)}.
\item{Contributing a pose graph-based global optimization approach for solving extrinsic parameters with reduced complexity. Proposing a hybrid calibration framework that integrates the flexibility of generic calibration into parametric calibration to prevent mutual compensation.}
\item{Evaluating, against a reference ground truth for the first time, the performance of calibration and reconstruction among various parametric models and generic camera models, providing guidance for model selection for different lenses and application scenarios.}
\end{itemize}

\section{Related Work}\label{sec2}
\subsection{Parametric Model Calibration}
Research on parametric camera calibration is extensive. In 1919, Conrady \cite{Conrady1919} introduced a distortion model incorporating both radial and tangential components. In 1966, Brown \cite{Brown1966} proposed the thin prism model and demonstrated its projective equivalence to Conrady's model. Numerous studies, based on Brown's model or assuming negligible camera distortions, improved calibration accuracy using methods such as plumb line techniques \cite{Brown1971}, vanishing point constraints \cite{Caprile1990}, and pure rotation constraints \cite{Hartley1994}. In 2000, Zhang \cite{Zhang2000} introduced a planar calibration pattern, which significantly improved calibration flexibility and accuracy, and was subsequently widely adopted by platforms such as OpenCV, MATLAB, and OpenMVG. Camera models with two or three radial distortion coefficients are commonly used for narrow and some wide-angle lenses, but are not suitable for fisheye lenses. Consequently, many researchers have focused on precise modeling of real fisheye lenses. Kannala and Brandt \cite{Kannala2006} proposed a set of models and corresponding calibration methods for fisheye and wide-angle cameras, which have been adopted by OpenCV and OpenMVG and later extended by subsequent studies \cite{Khomutenko2016}. Scaramuzza \cite{Scaramuzza2006} used the Taylor series expansion to describe the imaging function and performed nonlinear optimization based on the maximum likelihood criterion. Urban \cite{Urban2015} significantly improved the accuracy of Scaramuzza's calibration by replacing the residual function. The methods by Scaramuzza and Urban \cite{Scaramuzza2006,Urban2015} are now provided in the MATLAB platform. Mei \cite{Mei2007} introduced a calibration method for single-view omnidirectional cameras based on a  precise theoretical projection function.

With advancements in manufacturing processes, modern cameras are predominantly affected by radial distortion, while tangential distortion is often neglected or only considered during nonlinear optimization \cite{Larsson2019}. In 1987, Tsai \cite{Tsai1987} introduced a radial alignment constraint that only considered radial distortion. This constraint is typically used to estimate the camera's rotation and part of the translation components for self-calibration techniques and PnP. Larsson \cite{Larsson2019,Lochman2021} subsequently estimated the remaining translation parameters and intrinsic parameters based on various parametric models. Fitzgibbon \cite{Fitzgibbon2001} proposed a concise division model with a single radial distortion parameter in the denominator for self-calibration scenarios. Brito \cite{Brito2013,Brito2017} used the division model to implement a linear self-calibration method. Bujnak \cite{Bujnak2010} applied the division model to PnP problems with unknown intrinsics and proposed a minimal solver. Kukelova \cite{Kukelova2013,Kukelova2015} used the division model to propose minimal solvers for two-view imaging geometry with unknown camera intrinsics. Additionally, they \cite{Kukelova2013} employed a division model with three radial distortion parameters for PnP problems.

\subsection{Non-parametric Model Calibration}
Non-parametric camera calibration methods have garnered  significant attention in recent years. These models are primarily classified into one-dimensional (1D) and two-dimensional (2D) models. In particular, the 1D model initially determines all extrinsic parameters except for the third translation component using the radial alignment constraint. Subsequently, the remaining extrinsic parameters are identified without relying on any specific parameterized camera model. The 1D model typically establishes a one-dimensional mapping relationship between the radial distance of the image points and the normalized image coordinates. Hartley \cite{Hartley2007} assumed local linearity to estimate the radial distortion function, which makes it applicable to a wider range of cameras, from narrow-angle to fisheye lenses. Camposeco \cite{Camposeco2015} used an ordering constraint to find the translation of the camera center. Pan \cite{Pan2022} proposed an implicit distortion model and replaced the parametric model with a regularization term, which ensures smooth variation of the distortion mapping across the entire image.

To comprehensively capture the multi-dimensional characteristics of lens distortion, numerous studies have focused on the two-dimensional non-parametric model. Unlike the 1D model, the 2D model establishes a mapping relationship between two-dimensional image coordinates and two-dimensional observed rays, with control point information typically stored in lookup tables. The observed rays corresponding to image coordinates outside the control points are usually determined using B-spline surface models \cite{Miraldo2013,Beck2018,Rosebrock2012Generic}. This two-dimensional non-parametric model is often referred to as the generic camera model. In 2001, Grossberg and Nayar \cite{Grossberg2001} first introduced the generic camera model, which requires known extrinsic parameter information. Following this, Sturm and Ramalingam \cite{Sturm2004} proposed an extrinsic parameter initialization method and extended it to various types of lenses in \cite{Ramalingam2017}. Specifically, the work \cite{Ramalingam2017} requires the camera to be fixed while capturing the calibration pattern, ensuring that the observed rays corresponding to the same image point coordinates remain constant across different images. This method computes the camera's position and the relative pose between calibration patterns using the collinearity constraint of multiple 3D points along the same observed ray. The method by Sturm and Ramalingam \cite{Sturm2004} is often regarded as the standard generic calibration method. Schöps \cite{Schops2020} proposed a generic calibration pipeline using the standard generic calibration and open-sourced the first generic calibration toolbox. However, the initialization process of the standard generic calibration involves complex polynomial calculations \cite{Sturm2003,Ramalingam2006}. Dunne \cite{Dunne2007Efficient} introduced a linear extrinsic parameter estimation method by taking advantage of the geometric constraint derived from a single center of projection. This method reduces computational complexity and improves the accuracy and stability of initial values, although it calibrates only a portion of the image and requires the acquisition of active grids. Dunne's approach also requires the camera to be fixed during calibration. Rosebrock \cite{Rosebrock2012Generic,Rosebrock2012Complete} improved upon Dunne's method and expanded the calibration area to the entire image region. It is worth noting that both the standard generic calibration method and the linear extrinsic parameter estimation method utilize incremental optimization, necessitating multiple iterations that may introduce the risk of error accumulation. Brousseau \cite{Brousseau2019} treated each ray as an independent optimization variable and alternated optimization of the camera's intrinsic and extrinsic parameters. Nevertheless, the approach requires a greater number of pattern poses to overlap the same region in the image \cite{Bergamasco2017,Uhlig2020}.

\section{Generic Calibration Methods and Ambiguity Issues}\label{sec3}
\subsection{Generic Calibration Methods}\label{sec3.1}
Generic calibration methods include the standard generic calibration method \cite{Sturm2004, Ramalingam2017,Schops2020} and the linear extrinsic parameter estimation method \cite{Dunne2007Efficient,Rosebrock2012Generic,Rosebrock2012Complete}. Typically, these methods begin with initializing relative poses among several patterns, and then include additional patterns by incremental optimization. Subsequently, a lookup table that maps image coordinates to normalized image coordinates (or bearing vectors) is constructed using the B-spline surface, where the mapping represents the intrinsic parameters of the generic camera model. Finally, the intrinsic and extrinsic parameters are jointly optimized.

During the generic calibration, images of a calibration pattern placed in multiple distinct poses are captured with the camera fixed. For notational clarity and without causing confusion, we refer to the calibration pattern at the $i$-th pose as the ``$i$-th pattern'' throughout this paper. Let $R_i$ and $\mathbf{t}_i$ denote the transformation from the $i$-th pattern's coordinate system to the camera coordinate system. In this paper, we refer to these pattern-to-camera transformations as the extrinsic parameters, following the common convention in camera calibration literature, although they differ from the camera-to-world extrinsic parameters typically used in structure-from-motion contexts. The 3D point $\mathbf{X}_{il}=\left(X_{il}, Y_{il}, 0\right)^\top$, denoting the $l$-th point in the $i$-th pattern's coordinate system, can be expressed in the camera coordinate frame as:
\begin{equation}
\label{eq3}
\mathbf{X}_{il}^C=R_i \mathbf{X}_{il}+\mathbf{t}_i.
\end{equation}
Denote the corresponding normalized homogeneous coordinates as $\mathbf{x}_{il}=\left(x_{il}, y_{il}, 1\right)^\top$, and the distorted image coordinates as $\mathbf{u}_{il}=\left(u_{il}, v_{il}\right)^\top$.

Consider two different points $\mathbf{X}_{il}$ and $\mathbf{X}_{jp}$ that both project to the same $(u,v)$, as depicted in Fig. \ref{fig_1}. They are collinear with the optical center, i.e., $\mathbf{X}_{jp}^C=\lambda \mathbf{X}_{il}^C$, where $\lambda$ is a non-zero scalar. Using the matrix representation of Eq. \eqref{eq3}, we obtain
\begin{equation}
\label{eq16}
\begin{aligned}
\mathbf{X}_{il}^C &= \begin{pmatrix}\mathbf{r}_{i,1} & \mathbf{r}_{i,2} & \mathbf{t}_i\end{pmatrix}
\begin{pmatrix}X_{il}\\Y_{il}\\1\end{pmatrix}, \\
\mathbf{X}_{jp}^C &= \begin{pmatrix}\mathbf{r}_{j,1} & \mathbf{r}_{j,2} & \mathbf{t}_j\end{pmatrix}
\begin{pmatrix}X_{jp}\\Y_{jp}\\1\end{pmatrix},
\end{aligned}
\end{equation}
where $\mathbf{r}_{i,w}$ denotes the $w$-th column of $R_i$. The above collinearity implies a homography relationship between the points in their respective pattern coordinate systems:
\begin{equation}
\label{eq17}
\begin{pmatrix}
X_{jp} \\
Y_{jp} \\
1
\end{pmatrix} = \lambda\begin{pmatrix}\mathbf{r}_{j,1} & \mathbf{r}_{j,2} & \mathbf{t}_{j}\end{pmatrix}^{-1}
\begin{pmatrix}\mathbf{r}_{i,1} & \mathbf{r}_{i,2} & \mathbf{t}_{i}\end{pmatrix}
\begin{pmatrix}
X_{il} \\
Y_{il} \\
1
\end{pmatrix}.
\end{equation}
Defining $H_{ij}=\lambda\begin{pmatrix}\mathbf{r}_{j,1} & \mathbf{r}_{j,2} & \mathbf{t}_{j}\end{pmatrix}^{-1}
\begin{pmatrix}\mathbf{r}_{i,1} & \mathbf{r}_{i,2} & \mathbf{t}_{i}\end{pmatrix}$, Eq. \eqref{eq17} simplifies to
\begin{equation}
\label{eq18}
\begin{pmatrix}
X_{jp} \\
Y_{jp} \\
1
\end{pmatrix}=H_{ij}\begin{pmatrix}
X_{il} \\
Y_{il} \\
1
\end{pmatrix}.
\end{equation}
This homography $H_{ij}$ encodes the relative pose between the $i$-th pattern and the $j$-th pattern, exhibiting both invertibility and transitivity. As long as there is an overlap between the corresponding image regions of any two calibration patterns, a relationship as defined by Eq. \eqref{eq18} can be established between the two calibration patterns.

The linear extrinsic parameter estimation method \cite{Dunne2007Efficient,Rosebrock2012Generic,Rosebrock2012Complete} utilizes the homography relationship but introduces a specific reference frame. For example, let the $0$-th pattern be designated as the reference pattern, and correspondingly, the image captured when the calibration pattern is at this reference pose is termed the reference image. Denote the camera center in the reference pattern's coordinate system (i.e., $-R_0^\top \mathbf{t}_0$) as $\mathbf{O}$, and the transformation from the $i$-th pattern to the reference pattern as $(R_{i,0}, \mathbf{t}_{i,0})$. With this notation, the transformation from the $i$-th pattern to the camera is $R_i = R_0 R_{i,0}$ and $\mathbf{t}_i = R_0(\mathbf{t}_{i,0} - \mathbf{O})$. Substituting these specific forms, we obtain the specialized representation of Eq. \eqref{eq16}:
\begin{equation}
\label{eq4}
\begin{aligned}\mathbf{X}_{il}^C &=  R_0  \begin{pmatrix}    \mathbf{r}_{i,0,1} & \mathbf{r}_{i,0,2} & \mathbf{t}_{i,0}-\mathbf{O}  \end{pmatrix}  \begin{pmatrix}    X_{il}\\[2pt] Y_{il}\\[2pt] 1  \end{pmatrix},\\[4pt]
\mathbf{X}_{0p}^C &=  R_0  \begin{pmatrix}    \mathbf{e}_1 & \mathbf{e}_2 & -\mathbf{O}  \end{pmatrix}  \begin{pmatrix}    X_{0p}\\[2pt] Y_{0p}\\[2pt] 1  \end{pmatrix},
\end{aligned}
\end{equation}
where $\mathbf{r}_{i,0,w}$ is the $w$-th column of $R_{i,0}$ and $\mathbf{e}_q$ is the $q$-th column vector of the $3 \times 3$ identity matrix. Since $\mathbf{X}_{0p}^C=\lambda \mathbf{X}_{il}^C$, it results in a specific homography relationship by canceling out $R_{0}$,
\begin{equation}
\label{eq5}
\resizebox{.95\linewidth}{!}{$
\begin{aligned}
\begin{pmatrix}
X_{0p}\\[1pt] Y_{0p}\\[1pt] 1
\end{pmatrix}
&=\lambda\,\begin{pmatrix}\mathbf{e}_1 & \mathbf{e}_2 & -\mathbf{O} \end{pmatrix}^{-1}
\begin{pmatrix}\mathbf{r}_{i,0,1} & \mathbf{r}_{i,0,2} & \mathbf{t}_{i,0}-\mathbf{O} \end{pmatrix}
\begin{pmatrix}X_{il}\\[1pt] Y_{il}\\[1pt] 1 \end{pmatrix}  \stackrel{\triangle}{=}H_{i0}\begin{pmatrix} X_{il}\\[1pt] Y_{il}\\[1pt] 1 \end{pmatrix}.
\end{aligned}$}
\end{equation}
The resulting homography $H_{i0}$ relates points on the $i$-th pattern to corresponding points on the reference pattern, mediated by their shared projections in the image plane. This $H_{i0}$ is a special case of the general $H_{ij}$ in Eq. \eqref{eq18}, where the target pattern $j$ is the reference pattern.

\begin{figure}[!t]
\centering
\includegraphics[width=0.45\textwidth]{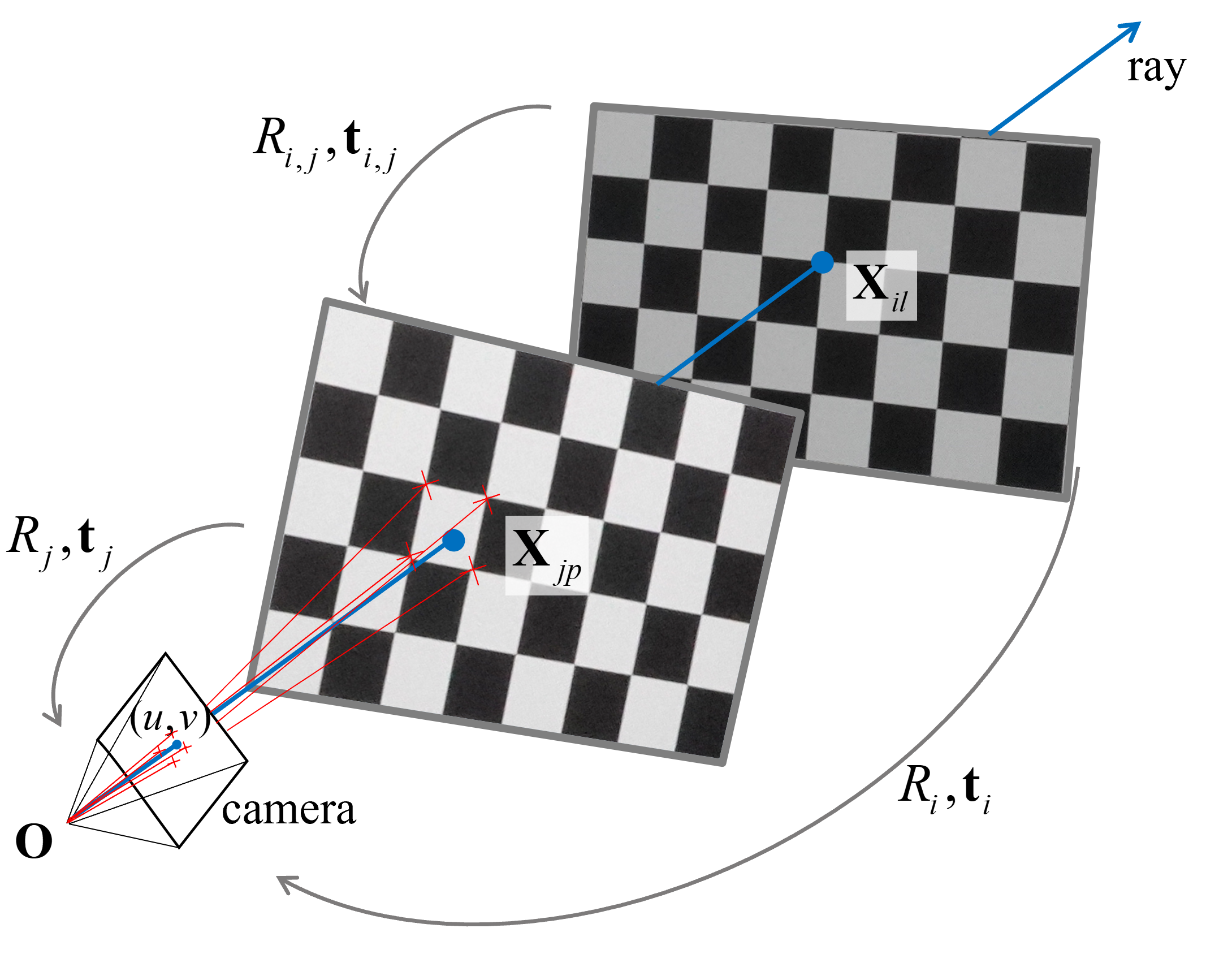}
\caption{Example of collinear 3D points $(\mathbf{X}_{il}, \mathbf{X}_{jp})$ on a calibration pattern in two distinct poses with camera center. $\mathbf{X}_{jp}$ is obtained by interpolation using four adjacent corner points (red cross) of image point $\left(u,v\right)$, the projection of $\mathbf{X}_{il}$ \cite{Ramalingam2017}.}
\label{fig_1}
\vspace{-5pt}
\end{figure}


The linear extrinsic parameter estimation method uses at least four point correspondences $\left\langle\mathbf{X}_{0p}, \mathbf{X}_{il}\right\rangle$ to compute the homography matrix $H_{i0}$. Then, the studies \cite{Dunne2007Efficient,Rosebrock2012Generic,Rosebrock2012Complete} linearly estimate $R_{i,0}, \mathbf{t}_{i,0}$ and $\mathbf{O}$ by utilizing at least two images that overlap with the reference image and combining with the properties of $\left\|\mathbf{r}_{i,0,1}\right\|=\left\|\mathbf{r}_{i,0,2}\right\|=1$ and $\mathbf{r}_{i,0,1}^\top \mathbf{r}_{i,0,2}=0$. It should be noted that $R_0$ remains unknown in the above development, resulting in a pose ambiguity in extrinsic parameter estimation.

Based on the collinearity constraint (see Fig. \ref{fig_1}), the standard generic calibration method \cite{Ramalingam2017,Schops2020} constructs a system of equations \cite{Ramalingam2006} to initialize $R_{i,0}, \mathbf{t}_{i,0}$ and $\mathbf{O}$. Similarly to the linear extrinsic parameter estimation method, the standard generic calibration also suffers from the pose ambiguity characterized by the unknown $R_0$.

Regarding the optimization procedure, generic calibration methods typically use the B-spline surface to construct a lookup table that maps image points to the corresponding observed rays \cite{Schops2020,Rosebrock2012Generic,Rosebrock2012Complete}. Taking forward mapping as an example, let $\boldsymbol{\beta}_{il} = \mathbf{X}_{il}^C / \left\| \mathbf{X}_{il}^C \right\|$ denote the bearing vector of an observed ray, where $\mathbf{X}_{il}^C$ is given in Eq. \eqref{eq4}. Then a sufficiently large set of $\left\langle \boldsymbol{\beta}_{il}, \mathbf{u}_{il} \right\rangle$ can define a lookup table $F_{\mathbf{\boldsymbol{\theta}}}$ that satisfies $\boldsymbol{\beta}_{il} = F_{\boldsymbol{\theta}}(\mathbf{u}_{il})$, where $\boldsymbol{\theta}$ is the set of B-spline surface control points \cite{Rosebrock2012Generic}. It follows from Eq. \eqref{eq4} that $\boldsymbol{\beta}_{il}$ depends upon the unknown rotation $R_0$. For a given $\hat{R}_0$, generic calibration typically initializes $\boldsymbol{\theta}$ via
\begin{equation}
\label{eq7}
    \underset{\boldsymbol{\theta}}{\arg \min} \sum_i \sum_l \left\| F_{\boldsymbol{\theta}}(\mathbf{u}_{il}) - \boldsymbol{\beta}_{il}(\hat{R}_0) \right\|.
\end{equation}
Both control points and extrinsic parameters are then simultaneously optimized by minimizing either reprojection error or ray-point distance, for instance,
\begin{equation}
\label{eq8}
\underset{\{R_i,\mathbf{t}_i\},\,\boldsymbol{\theta}}{\operatorname*{arg\,min}}
\sum_i \sum_l \bigl\|F_{\boldsymbol{\theta}}^{-1}(\boldsymbol{\beta}_{il})-\mathbf{u}_{il}\bigr\|.
\end{equation}

\subsection{Pose Ambiguity and Its Attributes}
As shown above, the extrinsic parameters estimated by the generic calibration approaches, $\tilde{R}_{i}$ and $\tilde{\mathbf{t}}_{i}$ for the $i$-th pattern, are inherently subject to a pose ambiguity. This subsection aims to formally describe the attributes of the ambiguity when recovering poses from such homography-based solutions.

\vspace{0.5\baselineskip}
\begin{proposition}
\label{prop1}
For any rotation matrix $\Lambda$, $\tilde{R}_{i}=\Lambda R_{i}$ and $\tilde{\mathbf{t}}_{i}=\Lambda \mathbf{t}_{i}$ (or $\Lambda R_{i} \operatorname{diag}(1,1,-1) R_{i}^{T} \mathbf{t}_{i}$) satisfy the homography relationship in Eq. \eqref{eq17}.
\end{proposition}

\begin{proof}
Given the homography relationship as Eq. \eqref{eq17} and \eqref{eq18}, for any invertible matrix $A$, we have
\begin{equation}
\label{eq19}
\begin{aligned}
H_{ij} &= \lambda\begin{pmatrix}\mathbf{r}_{j,1} & \mathbf{r}_{j,2} & \mathbf{t}_{j}\end{pmatrix}^{-1} A^{-1}A
\begin{pmatrix}\mathbf{r}_{i,1} & \mathbf{r}_{i,2} & \mathbf{t}_{i}\end{pmatrix}\\
&\stackrel{\triangle}{=} \lambda\begin{pmatrix}\tilde{\mathbf{r}}_{j,1} & \tilde{\mathbf{r}}_{j,2} & \tilde{\mathbf{t}}_{j}\end{pmatrix}^{-1}
\begin{pmatrix}\tilde{\mathbf{r}}_{i,1} & \tilde{\mathbf{r}}_{i,2} & \tilde{\mathbf{t}}_{i}\end{pmatrix},
\end{aligned}
\end{equation}
with the pose components $\tilde{\mathbf{r}}_{i,w} = A\mathbf{r}_{i,w}$ and $\tilde{\mathbf{t}}_{i} = A\mathbf{t}_{i}$ for $w=1,2$. Considering the $i$-th pattern as an example, due to the inherent constraints of extrinsic parameters solved by the homography matrix, the transformed vectors $\tilde{\mathbf{r}}_{i,1}$ and $\tilde{\mathbf{r}}_{i,2}$ must satisfy $\|\tilde{\mathbf{r}}_{i,1}\|=\|\tilde{\mathbf{r}}_{i,2}\|=1$ and $\tilde{\mathbf{r}}_{i,1}^{T}\tilde{\mathbf{r}}_{i,2}=0$ (the same constraint applies to the $j$-th pattern). Therefore, $A$ must be an orthogonal matrix. Then, the recovered $\tilde{R}_{i} \in SO(3)$ from $H_{ij}$ can be expressed as $\tilde{R}_i = \left(\tilde{\mathbf{r}}_{i,1} \quad \tilde{\mathbf{r}}_{i,2} \quad \tilde{\mathbf{r}}_{i,3}\right) = \left(A\mathbf{r}_{i,1} \quad A\mathbf{r}_{i,2} \quad (A\mathbf{r}_{i,1})\times(A\mathbf{r}_{i,2})\right)$. If $A$ is a rotation matrix, then $\tilde{R}_{i}=A R_{i}$, $\tilde{\mathbf{t}}_{i}=A \mathbf{t}_{i}$. Conversely, if $A$ is a reflection matrix, then $\tilde{R}_{i}=A(\mathbf{r}_{i,1} ~ \mathbf{r}_{i,2} ~ -\mathbf{r}_{i,3})=(A R_{i} \operatorname{diag}(1,1,-1) R_{i}^{T}) R_{i}$, $\tilde{\mathbf{t}}_{i}=A \mathbf{t}_{i}$.
\end{proof}

Proposition \ref{prop1} reveals two types of pose ambiguity: a rotation ambiguity, where the transformation is purely rotational ($\tilde{R}_{i}=\Lambda R_{i}, \tilde{\mathbf{t}}_{i}=\Lambda \mathbf{t}_{i}$ with $\Lambda \in SO(3)$), and a reflection ambiguity, where it involves a reflection ($\tilde{R}_{i}=\Lambda R_{i}, \tilde{\mathbf{t}}_{i}=\Lambda R_{i} \operatorname{diag}(1,1,-1) R_{i}^{T} \mathbf{t}_{i}$). 

\begin{proposition}
\label{prop2}
If the pose ambiguity in Proposition \ref{prop1} is restricted to the form $\tilde{R}_{i}=\Lambda_{i} R_{i}$ and $\tilde{\mathbf{t}}_{i}=\Lambda_{i} \mathbf{t}_{i}$, with $\Lambda_i \in SO(3)$, then the ambiguity matrix $\Lambda_i$ is unique and consistent across pattern at all poses, i.e., $\Lambda_i = \Lambda_j = \Lambda$ for any $i$-th and $j$-th patterns.
\end{proposition}

\begin{proof}
Two arbitrary poses from calibration pattern coordinate systems to the camera coordinate system satisfy $\tilde{R}_{i}=\Lambda_{i} R_{i}$, $\tilde{\mathbf{t}}_{i}=\Lambda_{i} \mathbf{t}_{i}$, $\tilde{R}_{j}=\Lambda_{j} R_{j}$, $\tilde{\mathbf{t}}_{j}=\Lambda_{j} \mathbf{t}_{j}$, where $\Lambda_{i}, \Lambda_{j} \in SO(3)$. The homography between any two patterns is geometrically unique (up to scale). Thus, $\tilde{H}_{ij}$ constructed from the above poses and the true $H_{ij}$ must satisfy
\begin{equation}
\label{eq23}
\tilde{H}_{ij}=\lambda_{ij} H_{ij},
\end{equation}
which can be rewritten as
\begin{equation}
\label{eq24}
\resizebox{.9\linewidth}{!}{$
\begin{aligned}
\begin{pmatrix}\tilde{\mathbf{r}}_{j,1} & \tilde{\mathbf{r}}_{j,2} & \tilde{\mathbf{t}}_{j}\end{pmatrix}^{-1}\begin{pmatrix}\tilde{\mathbf{r}}_{i,1} & \tilde{\mathbf{r}}_{i,2} & \tilde{\mathbf{t}}_{i}\end{pmatrix} = \lambda_{ij} \begin{pmatrix}\mathbf{r}_{j,1} & \mathbf{r}_{j,2} & \mathbf{t}_{j}\end{pmatrix}^{-1}\begin{pmatrix}\mathbf{r}_{i,1} & \mathbf{r}_{i,2} & \mathbf{t}_{i}\end{pmatrix},
\end{aligned}$}
\end{equation}
where $\lambda_{ij}$ is a nonzero scalar, leading to $\Lambda_j^{-1}\Lambda_i = \lambda_{ij}I_{3\times3}$. Taking determinants on both sides yields $\lambda_{ij}=1$. Thus, the two ambiguity matrices $\Lambda_{i}$ and $\Lambda_{j}$ for calibration pattern at any two poses satisfy
\begin{equation}
\label{eq25}
\Lambda_{i}=\Lambda_{j}.
\end{equation}
Therefore, the ambiguity matrix relating all pattern poses is consistent and unique.
\end{proof}

\begin{proposition}
\label{prop3}
For $\Lambda=\operatorname{diag}(1,1,\pm1)$, $\tilde{R}_{i,0}=\Lambda R_{i,0}$ and $\tilde{\mathbf{t}}_{i,0}=\Lambda \mathbf{t}_{i,0}$ satisfy the homography relationship in Eq. \eqref{eq5}. 
\end{proposition}

\begin{proof}
Given the homography relationship as Eq. \eqref{eq16} and \eqref{eq17}, for any invertible matrix $A$, we have
\begin{equation}
\label{eq19}
\begin{aligned}
H_{i0} &= \lambda\begin{pmatrix}\mathbf{e}_1 & \mathbf{e}_2 & -\mathbf{O} \end{pmatrix}^{-1} A^{-1}A
\begin{pmatrix}\mathbf{r}_{i,0,1} & \mathbf{r}_{i,0,2} & \mathbf{t}_{i,0}-\mathbf{O} \end{pmatrix}\\
&\stackrel{\triangle}{=} \lambda\begin{pmatrix}\mathbf{e}_1 & \mathbf{e}_2 & -\tilde{\mathbf{O}} \end{pmatrix}^{-1}
\begin{pmatrix}\tilde{\mathbf{r}}_{i,0,1} & \tilde{\mathbf{r}}_{i,0,2} & \tilde{\mathbf{t}}_{i,0}-\tilde{\mathbf{O}} \end{pmatrix},
\end{aligned}
\end{equation}
with the pose components $\tilde{\mathbf{r}}_{i,0, w} = A\mathbf{r}_{i,0,w}$ ($w=1,2$) and $\tilde{\mathbf{t}}_{i,0} = A\mathbf{t}_{i,0}$. Therefore $A\mathbf{e}_1=\mathbf{e}_1$, $A\mathbf{e}_2=\mathbf{e}_2$. $A$ must be the identity matrix or $\operatorname{diag}(1,1,-1)$. 
\end{proof}



From Section \ref{sec3.1} and Propositions \ref{prop1}$\sim\ref{prop3}$, the linear extrinsic parameter estimation method inherently suffers from pose ambiguity. This ambiguity can be reduced to rotation ambiguity by applying a chiral constraint \cite{Rosebrock2012Generic}. In summary, both the standard generic calibration method and the linear extrinsic parameter estimation method inherently suffer from a rotation ambiguity.


The consequent joint optimization process in Eq. \eqref{eq8} simultaneously refines the B-spline surface control points $\boldsymbol{\theta}$ and the extrinsic parameters $\{R_i, \mathbf{t}_i\}$. A critical question then arises: Can this comprehensive optimization effectively compensate for or eliminate the pose ambiguity's influence, leading to a unique set of intrinsic parameters $\boldsymbol{\theta}$ regardless of the initial $\Lambda$? The answer is NO.

\begin{proposition}
\label{prop4}
Given $\Lambda \neq \Lambda'$, the optimal control points $\boldsymbol{\theta}$ of the B-spline surface satisfy $\boldsymbol{\theta}(\Lambda) \neq \boldsymbol{\theta}(\Lambda')$.
\end{proposition}

\begin{proof}
The 2D representation of $\boldsymbol{\beta}_{il}$ is utilized to simplify this proof. A B-spline surface function $F_{\boldsymbol{\theta}}(\mathbf{u}_{il})$ can be represented as
\begin{equation}
\label{eq9}
F_{\boldsymbol{\theta}}(\mathbf{u}_{il})=
\sum_{a=0}^{n}\sum_{b=0}^{m}
N_{a,p}(\mu)\,N_{b,p}(\nu)\,\overline{\boldsymbol{\beta}}_{a,b},
\end{equation}
where $N_{a,p}(\mu)$ and $N_{b,p}(\nu)$ are B-spline basis functions of degree $p$. The parameters $\mu$ and $\nu$ are computed from the image coordinates $\mathbf{u}_{il}$ \cite{Rosebrock2012Generic}. The terms $\overline{\boldsymbol{\beta}}_{a,b}$, indexed by $a\!\in\![0,n]$ and $b\!\in\![0,m]$, are the control points that define the B-spline surface. Define $\boldsymbol{\theta}$ as the $L \times 2$ matrix containing all control points vectors $\overline{\boldsymbol{\beta}}_{a,b}$ arranged in sequence, i.e., $\boldsymbol{\theta} = \left[\overline{\boldsymbol{\beta}}_{0,0},\ldots,\overline{\boldsymbol{\beta}}_{0,m},\ldots,\overline{\boldsymbol{\beta}}_{n,m}\right]^T$, where $\quad L=(n+1)(m+1)$. By arranging the basis functions into a $1 \times L$ row vector
\begin{equation}
    \label{eq:arranged basis function}
    \resizebox{.95\linewidth}{!}{$
    \mathbf{N}_{il}(\mu,\nu)=\bigl[N_{0,p}(\mu)N_{0,p}(\nu),\ldots,N_{0,p}(\mu)N_{m,p}(\nu),\ldots,N_{n,p}(\mu)N_{m,p}(\nu)\bigr],$}
\end{equation}
 we derive $F_{\boldsymbol{\theta}}(\mathbf{u}_{il})=(\mathbf{N}_{il}(\mu,\nu)\,\boldsymbol{\theta})^T$. Substituting into Eq. \eqref{eq7} yields
\begin{equation}
\label{eq10}
\underset{\boldsymbol{\theta}}{\operatorname*{arg\,min}}
\sum_i\sum_l \bigl\|\bigl(\mathbf{N}_{il}(\mu,\nu)\,\boldsymbol{\theta}\bigr)^T-\boldsymbol{\beta}_{il}\bigr\|^2.
\end{equation}
An $M\!\times\!L$ matrix $\mathbf{G}(\mu,\nu)$ is defined by stacking $\mathbf{N}_{il}(\mu,\nu)$, where $M$ is the total number of points $\mathbf{u}_{il}$, and each row of $\mathbf{G}(\mu,\nu)$ corresponds to a B-spline basis function vector $\mathbf{N}_{il}(\mu,\nu)$ from $\mathbf{u}_{il}$. Similarly, stacking the corresponding $\boldsymbol{\beta}_{il}^T$ forms an $M \times 2$ matrix $\mathbf{B}(\Lambda)$. The optimization problem in Eq. \eqref{eq10} is thus equivalent to seeking the optimal $\boldsymbol{\theta}$ in
\begin{equation}
\mathbf{G}(\mu,\nu)\boldsymbol{\theta}  + \mathbf{E} = \mathbf{B}(\Lambda)
\end{equation}
by minimizing the squared Frobenius norm of residual matrix $\mathbf{E}$. If the matrix $\mathbf{G}(\mu,\nu)$ has full column rank (or via regularization), the solution to this linear least-squares problem is given by the normal equation 
\begin{equation}
\label{eq12}
\boldsymbol{\theta}(\Lambda)=\bigl(\mathbf{G}^T\!\mathbf{G}\bigr)^{-1}\mathbf{G}^T\mathbf{B}(\Lambda).
\end{equation}
Similarly, when given $\Lambda'$, we obtain 
\begin{equation}
\label{eq13}
\boldsymbol{\theta}(\Lambda')=\bigl(\mathbf{G}^T\!\mathbf{G}\bigr)^{-1}\mathbf{G}^T\mathbf{B}(\Lambda').
\end{equation}
Since the image coordinate $\mathbf{u}_{il}$ in Eq. \eqref{eq7} is consistent, $\mathbf{N}_{il}$ remains identical, hence $\mathbf{G}(\mu,\nu)$ in Eq. \eqref{eq12} and Eq. \eqref{eq13} is also identical. Note that different matrices $\Lambda\neq\Lambda'$ yield distinct observed rays $\boldsymbol{\beta}_{il}(\Lambda)\neq\boldsymbol{\beta}_{il}(\Lambda')$. Thus,
\begin{equation}
\label{eq14}
\mathbf{B}(\Lambda)\neq\mathbf{B}(\Lambda').
\end{equation}
Consequently, the optimal control points satisfy
\begin{equation}
\label{eq15}
\boldsymbol{\theta}(\Lambda)\neq\boldsymbol{\theta}(\Lambda').
\end{equation}
\end{proof}

Once the intrinsic parameter set $\boldsymbol{\theta}(\Lambda)$ has been determined through calibration, it establishes the bidirectional projection functions for subsequent applications. For example, the back projection maps a distorted image coordinate to its corresponding bearing vector via $\hat{\boldsymbol{\beta}} = F_{\boldsymbol{\theta}}(\mathbf{u})$. Conversely, the forward projection maps a bearing vector to the distorted image coordinate $\hat{\mathbf{u}}$ by solving the nonlinear minimization problem: $\arg \min_{\hat{\mathbf{u}}} \|F_{\boldsymbol{\theta}}(\hat{\mathbf{u}})-\boldsymbol{\beta}\|$. Different lookup table parameter sets $\boldsymbol{\theta}(\Lambda) \neq \boldsymbol{\theta}(\Lambda')$ result in varying back-projection results $F_{\boldsymbol{\theta}(\Lambda)}(\mathbf{u}_{il}) \neq F_{\boldsymbol{\theta}(\Lambda')}(\mathbf{u}_{il})$ for the same distorted image points $\mathbf{u}_{il}$. Similarly, the forward projection results vary as well, i.e., $F_{\boldsymbol{\theta}(\Lambda)}^{-1}(\boldsymbol{\beta}_{il}) \neq F_{\boldsymbol{\theta}(\Lambda')}^{-1}(\boldsymbol{\beta}_{il})$. 

Several previous studies \cite{Ramalingam2006,Dunne2007Efficient,Schops2020} acknowledged the unknown rotation matrix and adopted ad-hoc strategies to circumvent this issue. For example, Ramalingam noted in \cite{Ramalingam2006} that "Ideally, rays should be perpendicular to the plane." More precisely, studies \cite{Dunne2007Efficient,Ramalingam2006} determined $R_0$ by
\begin{equation}
\label{eq: R0 in open-source toolbox}
\arg\min_{R_0} \sum_{i} \|\boldsymbol{\beta}_i \times \boldsymbol{\xi}\|_2,
\end{equation}
where $\boldsymbol{\xi}$ denotes the unit normal vector of the reference pattern. However, these studies presupposed that the average direction of all observed rays and the camera's optical axis achieve directional consistency and left degrees of freedom unresolved. Schöps \cite{Schops2020} adjusted the camera's orientation to standardize the mean directions of observed rays across different calibration methods, thereby facilitating the comparative analysis of distortion correction and reconstruction results. However, this effort did not resolve the ambiguity issue. Generally speaking, generic calibration approaches failed to adequately recognize the ambiguity issue.

\subsection{Impact of Pose Ambiguity on Pose Estimation and 3D Reconstruction}
The pose ambiguity inherent in generic calibration methods introduces a rotation uncertainty to the representation of a 3D point within the camera frame:
\begin{equation}
\label{eq26}
\tilde{\mathbf{X}}_{il}^{C}=\tilde{R}_{i} \mathbf{X}_{il}+\tilde{\mathbf{t}}_{i} 
= \Lambda(R_{i} \mathbf{X}_{il}+\mathbf{t}_{i}) 
= \Lambda \mathbf{X}_{il}^{C}.
\end{equation}
The relationship between the ambiguous and true normalized coordinates is given by: $\tilde{\mathbf{x}}_{il} = z \Lambda \mathbf{x}_{il} / \tilde{z}$, where $z$ and $\tilde{z}$ represent the depths of $\mathbf{X}_{il}^{C}$ and $\tilde{\mathbf{X}}_{il}^{C}$, respectively. Any images captured by the same camera share the same ambiguity matrix $\Lambda$. Consequently, when undistorting an image point using the ambiguous intrinsic parameters, the resulting normalized image coordinates differ from the true coordinates:
\begin{equation}
\label{eq_normalized_coordinates_relationship}
\tilde{\mathbf{x}} = z \Lambda \mathbf{x} / \tilde{z}.
\end{equation}

The pose ambiguity would propagate to subsequent vision tasks. Consider a two-view scenario where the left camera defines the world frame and the true relative pose between the two views is $(R, \mathbf{t})$. The ambiguous normalized image coordinates of a 3D point $\mathbf{X}$ are $\tilde{\mathbf{x}}=z \Lambda \mathbf{x} / \tilde{z}$ (left view) and $\tilde{\mathbf{x}}^{\prime}=z^{\prime} \Lambda \mathbf{x}^{\prime} / \tilde{z}^{\prime}$ (right view), which differ from the true coordinates $\mathbf{x}$ and $\mathbf{x}^{\prime}$. Denote the reconstructed 3D points in the left and right camera frames as $\tilde{\mathbf{X}}$ and $\tilde{\mathbf{X}}^{\prime}$, respectively. From Eq. \eqref{eq_normalized_coordinates_relationship}, the reconstructed 3D points can be expressed as:
\begin{equation}
\begin{aligned}
\tilde{\mathbf{X}} &= \tilde{z} \tilde{\mathbf{x}} = z \Lambda \mathbf{x} = \Lambda \mathbf{X}, \\
\tilde{\mathbf{X}}^{\prime} &= \tilde{z}^{\prime} \tilde{\mathbf{x}}^{\prime} = z^{\prime} \Lambda \mathbf{x}^{\prime} = \Lambda (R \mathbf{X} + \mathbf{t}).
\end{aligned}
\label{eq:reconstruction ambiguity}
\end{equation}
The reconstructed 3D points are rotated versions of the true points by the ambiguity matrix $\Lambda$, leading to a  warped two-view imaging equation,
\begin{equation}
\label{eq29:view_under_ambiguity}
\tilde{z}^{\prime} \tilde{\mathbf{x}}^{\prime} = \tilde{z} \Lambda R \Lambda^{T} \tilde{\mathbf{x}} + \Lambda \mathbf{t},
\end{equation}
which reveals that the relative pose is warped from $(R, \mathbf{t})$ to $(\Lambda R \Lambda^T, \Lambda \mathbf{t})$.

The above problem caused by the pose ambiguity also applies to the multi-view scenarios. According to Eq. \eqref{eq:reconstruction ambiguity} and \eqref{eq29:view_under_ambiguity}, the estimated global poses $(\tilde{R}, \tilde{\mathbf{t}})$ are related to the ground truth by
\begin{equation}
\tilde{R} = \Lambda R \Lambda^{T}, \quad \tilde{\mathbf{t}} = \Lambda \mathbf{t}, \quad  \tilde{\mathbf{X}} = \Lambda \mathbf{X}.
\end{equation}
Unfortunately, the impacts of the pose ambiguity have largely been overlooked in previous literature, despite its potentially critical impact on practical applications, including but not limited to SfM, SLAM and PnP. 

The pose ambiguity brought about by the generic calibration is fundamentally distinct from the global-scale ambiguity \cite{Hartley2003}. The latter is a well-known property of multi-view geometry, independent of the calibration method. For example, the projection of a 3D point $\mathbf{X}$ from a world frame onto the camera's normalized image plane is given by ${\mathbf{x}} = (R^C{\mathbf{X}} + {\mathbf{t}}^C)/z$, where $(R^C, \mathbf{t}^C)$ denotes the world-to-camera pose. This projection relationship remains invariant under a global similarity transformation $\left\{R_s, \mathbf{t}_s, \alpha\right\}$ applied to the entire scene and all cameras: 
\begin{equation}
\label{eq2}
\begin{aligned}
R^C &\rightarrow R^C R_s^\top, \\
\mathbf{t}^C &\rightarrow \alpha\left(\mathbf{t}^C-R^C \mathbf{t}_s\right), \\
\mathbf{X} &\rightarrow \alpha R_s\left(\mathbf{X}+\mathbf{t}_s\right).
\end{aligned}
\end{equation}


\section{Camera Pose Estimation to Solve Pose Ambiguity}\label{sec4}
The above-identified pose ambiguity is inherent in current generic calibration methods. Fortunately, it can be solved by exploiting the radial alignment constraint \cite{Hartley2007,Kukelova2013, Camposeco2015, Lin2020,Lochman2021,Pan2022}, namely, the image point should lie on the radial line emanating from the principal point. This section proposes a linear method for estimating the rotation ambiguity matrix using the radial alignment constraint within a RANSAC framework. 

\subsection{Linear Solution for Pose Ambiguity}
Specifically, to resolve the pose ambiguity above, we stack 3D points from all pattern poses into a unified set and re-index them globally using a single index $k$. Let $\tilde{\mathbf{X}}_{k}^{C}=(\tilde{X}_{k}^{C}, \tilde{Y}_{k}^{C}, \tilde{Z}_{k}^{C})^{T}$. Let $f_{k}$ denote the point-wise focal length \cite{Pan2022}, which incorporates both focal length and radial distortion information, and $\mathbf{c} = (c_x, c_y)$ denote the principal point of the image. According to Eq. \eqref{eq26}, the projection of 3D points $\tilde{\mathbf{X}}_{k}^{C}$ can be expressed as
\begin{equation}
\label{eq31:projection_expression}
z_{k}\begin{pmatrix}
u_{k}-c_{x} \\
v_{k}-c_{y} \\
1
\end{pmatrix}=\begin{pmatrix}
f_{k} & 0 & 0 \\
0 & f_{k} & 0 \\
0 & 0 & 1
\end{pmatrix} \Lambda^{T} \tilde{\mathbf{X}}_{k}^{C}.
\end{equation}
By eliminating $z_{k}$ and $f_{k}$ using the first two equations, we obtain
\begin{equation}
\label{eq32}
\frac{u_{k}-c_{x}}{v_{k}-c_{y}}=\frac{\gamma_{11} \tilde{X}_{k}^{C}+\gamma_{12} \tilde{Y}_{k}^{C}+\gamma_{13} \tilde{Z}_{k}^{C}}{\gamma_{21} \tilde{X}_{k}^{C}+\gamma_{22} \tilde{Y}_{k}^{C}+\gamma_{23} \tilde{Z}_{k}^{C}},
\end{equation}
where $\gamma_{ij}$ represents the $i$-th row and the $j$-th column of $\Lambda^{T}$. Rearranging Eq. \eqref{eq32} for each point correspondence yields one constraint. Stacking these constraints for six correspondences results in the following 6x9 homogeneous linear system:
\begin{equation}
\label{eq33:linear_system_expanded}
\resizebox{.95\linewidth}{!}{$
\begin{aligned}
&\begin{pmatrix}
u_1\tilde{X}_1^C & u_1\tilde{Y}_1^C & u_1\tilde{Z}_1^C & v_1\tilde{X}_1^C & v_1\tilde{Y}_1^C & v_1\tilde{Z}_1^C & \tilde{X}_1^C & \tilde{Y}_1^C & \tilde{Z}_1^C \\
u_2\tilde{X}_2^C & u_2\tilde{Y}_2^C & u_2\tilde{Z}_2^C & v_2\tilde{X}_2^C & v_2\tilde{Y}_2^C & v_2\tilde{Z}_2^C & \tilde{X}_2^C & \tilde{Y}_2^C & \tilde{Z}_2^C \\
\vdots & \vdots & \vdots & \vdots & \vdots & \vdots & \vdots & \vdots & \vdots \\
u_6\tilde{X}_6^C & u_6\tilde{Y}_6^C & u_6\tilde{Z}_6^C & v_6\tilde{X}_6^C & v_6\tilde{Y}_6^C & v_6\tilde{Z}_6^C & \tilde{X}_6^C & \tilde{Y}_6^C & \tilde{Z}_6^C
\end{pmatrix} 
\begin{pmatrix}
\gamma_{21} \\
\gamma_{22} \\
\gamma_{23} \\
-\gamma_{11} \\
-\gamma_{12} \\
-\gamma_{13} \\
\delta_1 \\
\delta_2 \\
\delta_3
\end{pmatrix} = \mathbf{0},
\end{aligned}$}
\end{equation}
with $\delta_{1}$, $\delta_{2}$ and $\delta_{3}$ defined as
\begin{equation}
\label{eq34:delta_definition}
\begin{pmatrix}
\delta_{1} \\
\delta_{2} \\
\delta_{3}
\end{pmatrix}=\begin{pmatrix}
-\gamma_{21} & \gamma_{11} \\
-\gamma_{22} & \gamma_{12} \\
-\gamma_{23} & \gamma_{13}
\end{pmatrix}\begin{pmatrix}
c_{x} \\
c_{y}
\end{pmatrix}.
\end{equation}
The solution $\mathbf{q}$ to Eq. \eqref{eq33:linear_system_expanded} is derived from the singular vectors corresponding to the three smallest singular values, denoted as $\mathbf{q}_{1}$, $\mathbf{q}_{2}$, $\mathbf{q}_{3}$. Thus, $\mathbf{q}$ can be written as
\begin{equation}
\label{eq35:q_solution}
\mathbf{q}=a \mathbf{q}_{1}+b \mathbf{q}_{2}+c \mathbf{q}_{3},
\end{equation}
where $a$, $b$ and $c$ are unknown coefficients. To simplify the calculation, we temporarily fix the scale by setting $c=1$. Since the first six elements of $\mathbf{q}$ correspond to the first two rows of the rotation matrix $\Lambda^{T}$, substituting Eq. \eqref{eq35:q_solution} into the constraints of the rotation matrix results in a multivariate polynomial in terms of the unknown parameters $a$ and $b$.

We next solve this polynomial equation using the Gröbner basis method as detailed by Stewénius \cite{Stewenius2006} and Kneip \cite{Kneip2012}. Specifically, we first derive a Gröbner basis consisting of ten polynomial constraints from $\Lambda^{T} \Lambda=I_{3 \times 3}$ and $|\Lambda|=1$, which yields an equivalent set of 20 quadratic constraints:

\begin{equation}
\label{eq36:polynomial_constraints}
\left\{\begin{array}{l}
\gamma_{31}-\gamma_{12} \gamma_{23}+\gamma_{13} \gamma_{22}=0 \\
\gamma_{32}-\gamma_{13} \gamma_{21}+\gamma_{11} \gamma_{23}=0 \\
\gamma_{33}-\gamma_{11} \gamma_{22}+\gamma_{12} \gamma_{21}=0 \\
\gamma_{11}^{2}-\gamma_{22}^{2}-\gamma_{23}^{2}-\gamma_{32}^{2}-\gamma_{33}^{2}+1=0 \\
\gamma_{11} \gamma_{12}+\gamma_{21} \gamma_{22}+\gamma_{31} \gamma_{32}=0 \\
\gamma_{12}^{2}+\gamma_{22}^{2}+\gamma_{32}^{2}-1=0 \\
\gamma_{11} \gamma_{13}+\gamma_{21} \gamma_{23}+\gamma_{31} \gamma_{33}=0 \\
\gamma_{12} \gamma_{13}+\gamma_{22} \gamma_{23}+\gamma_{32} \gamma_{33}=0 \\
\gamma_{13}^{2}+\gamma_{23}^{2}+\gamma_{33}^{2}-1=0 \\
\gamma_{11} \gamma_{21}+\gamma_{12} \gamma_{22}+\gamma_{13} \gamma_{23}=0 \\
\gamma_{21}^{2}+\gamma_{22}^{2}+\gamma_{23}^{2}-1=0 \\
\gamma_{11} \gamma_{31}+\gamma_{12} \gamma_{32}+\gamma_{13} \gamma_{33}=0 \\
\gamma_{12} \gamma_{31}-\gamma_{11} \gamma_{32}-\gamma_{23}=0 \\
\gamma_{22}-\gamma_{11} \gamma_{33}+\gamma_{13} \gamma_{31}=0 \\
\gamma_{21} \gamma_{31}+\gamma_{22} \gamma_{32}+\gamma_{23} \gamma_{33}=0 \\
\gamma_{13}-\gamma_{21} \gamma_{32}+\gamma_{22} \gamma_{31}=0 \\
\gamma_{23} \gamma_{31}-\gamma_{21} \gamma_{33}-\gamma_{12}=0 \\
\gamma_{31}^{2}+\gamma_{32}^{2}+\gamma_{33}^{2}-1=0 \\
\gamma_{13} \gamma_{32}-\gamma_{12} \gamma_{33}-\gamma_{21}=0 \\
\gamma_{11}-\gamma_{22} \gamma_{33}+\gamma_{23} \gamma_{32}=0

\end{array}\right..
\end{equation}
By replacing $\{\gamma_{31}, \gamma_{32}, \gamma_{33}\}$ with $\{\gamma_{11}, \gamma_{12}, \gamma_{13}, \gamma_{21}, \gamma_{22}, \gamma_{23}\}$, and expanding the equations, we obtain 15 quadratic constraints in terms of $\{\gamma_{11}, \gamma_{12}, \gamma_{13}, \gamma_{21}, \gamma_{22}, \gamma_{23}\}$ that correspond precisely to the first six elements of the vector $\mathbf{q}$. Substituting Eq. \eqref{eq35:q_solution} into the above constraints yields
\begin{equation}
\label{eq37:matrix_equation}
D\mathbf{y} = \mathbf{0},
\end{equation}
where $D$ is a $15 \times 15$ coefficient matrix constructed from the elements of $\mathbf{q}_1$, $\mathbf{q}_2$ and $\mathbf{q}_3$. $\mathbf{y}=(a^4, b^4, ab^3, a^2b^2, a^3b, b^3, ab^2, a^2b, a^3, b^2, ab, a^2, b, a, 1)^T$. By applying Gaussian elimination to $D$, we transform its first five columns into an identity matrix and select the first five polynomial equations,
\begin{equation}
\label{eq38:reduced_system}
(I_{5 \times 5} \mid M_{5 \times 10})\mathbf{y} = \mathbf{0}.
\end{equation}
The structure of Eq. \eqref{eq38:reduced_system} is summarized in Table \ref{tab1:polynomial_equations}.

\begin{table}[!t]
\caption{System of polynomial equations}
\label{tab1:polynomial_equations}
\centering
\renewcommand{\arraystretch}{1.3}
\scriptsize
\setlength{\tabcolsep}{2pt}
\begin{tabular}{@{}c*{15}{c}@{}}
\toprule
 $a^{4}$ & $b^{4}$ & $ab^{3}$ & $a^{2}b^{2}$ & $a^{3}b$ & $b^{3}$ & $ab^{2}$ & $a^{2}b$ & $a^{3}$ & $b^{2}$ & $ab$ & $a^{2}$ & $b$ & $a$ & 1 \\ 
\midrule
1 &  &  &  &  & -- & -- & -- & -- & -- & -- & -- & -- & -- & -- \\
  & 1 &  &  &  & -- & -- & -- & -- & -- & -- & -- & -- & -- & -- \\
  &  & 1 &  &  & -- & -- & -- & -- & -- & -- & -- & -- & -- & -- \\
  &  &  & 1 &  & -- & -- & -- & -- & -- & -- & -- & -- & -- & -- \\
  &  &  &  & 1 & -- & -- & -- & -- & -- & -- & -- & -- & -- & -- \\
\bottomrule
\end{tabular}

\scriptsize
\noindent `--' denotes non-zero coefficients; empty spaces denote zero coefficients.
\vspace{-2em}  
\end{table}
\vspace{0.5em}

Let the lower-order monomials form the vector $\mathbf{g}=(b^3, ab^2, a^2b, a^3, b^2, ab, a^2, b, a, 1)^T$. As illustrated in Table \ref{tab1:polynomial_equations}, the first five elements of $\mathbf{y}$ can be linearly expressed in terms of the elements of $\mathbf{g}$. Therefore, we formulate the equation
\begin{equation}
\label{eq39:eigenvalue_problem}
Q\mathbf{g} = b\mathbf{g},
\end{equation}
where $Q=\left(-M' \quad \mathbf{e}_1 \quad \mathbf{e}_2 \quad \mathbf{e}_3 \quad \mathbf{e}_5 \quad \mathbf{e}_6 \quad \mathbf{e}_8\right)^\top$; $M'$ comprises the transpose of the 2nd $\sim$ 5th rows of $M_{5 \times 10}$; and $\mathbf{e}_i$ denotes the $i$-th column of a 10-dimensional identity matrix. As shown in Eq. \eqref{eq39:eigenvalue_problem}, $\mathbf{g}$ is an eigenvector of $Q$. Let $\mathbf{g}_{i,j}$ denote the $i$-th coefficient of $Q$'s $j$-th eigenvector. Then the solution for $a$ and $b$ in Eq. \eqref{eq35:q_solution} can be expressed as
\begin{equation}
\label{eq40:ab_solution}
\{(a, b) \mid a = \mathbf{g}_{9,j} / \mathbf{g}_{10,j}, b = \mathbf{g}_{8,j} / \mathbf{g}_{10,j}\}.
\end{equation}
The ten eigenvectors of $Q$ correspond to ten sets of candidate solutions for $(a, b)$. Subsequently, by leveraging the unit-length property of the 1st--3rd and 4th--6th elements of $\mathbf{q}$, the elements in $\mathbf{q}$ are proportionally scaled. Since the scaling ratio can be either positive or negative, there are a total of 20 candidate solutions to pose ambiguity. The corresponding principal point can be subsequently determined using Eq. \eqref{eq34:delta_definition}.

\subsection{Camera Pose Selection and Nonlinear Optimization with Regularization}
Eq. \eqref{eq40:ab_solution} provides 20 candidate solutions to pose ambiguity. In this subsection, we select the optimal solution among these candidates and further refine it through nonlinear optimization. Specifically, we apply the following criteria for rapid screening: 
\begin{itemize}
\item The estimated principal point must lie within the valid region of the image. 
\item Visible 3D points must have positive depth. 
\item The radial distances of visible points in the normalized image plane must lie within a reasonable range. 
\end{itemize}
Among solutions satisfying these conditions, we evaluate their corresponding reprojected radial distances and select the one with the largest number of inliers as the initial estimate for the pose ambiguity matrix and the principal point.

Previous studies on non-parametric camera calibration have employed various strategies: Lin \cite{Lin2020} minimized the reprojection distance between points and corresponding radial lines; Hartley \cite{Hartley2007} and Camposeco \cite{Camposeco2015} employed a sliding median filter to mitigate the impact of noise; and Pan \cite{Pan2022} replaced the parametric model with regularization to achieve smoothly varying distortion mappings. Similarly, our optimization approach for the pose ambiguity and principal point does not depend on any specific parametric camera model. Our method minimizes the reprojection error perpendicular to the radial lines and simultaneously introduces additional constraints that incorporate residual information along the radial direction. Specifically, given sufficient point correspondences $\Lambda^T\tilde{\mathbf{X}}_k^C$ and $\mathbf{u}_k$, we arrange these points according to their radial distances on the image plane $\|\mathbf{d}_k\|$, where $\mathbf{d}_k = \mathbf{u}_k - \mathbf{c}$. Denote $\Lambda^T\tilde{\mathbf{X}}_k^C = (X_k^C, Y_k^C, Z_k^C)^T$, and let $\boldsymbol{\chi}_k = (X_k^C/Z_k^C, \alpha Y_k^C/Z_k^C)^T$, where $\alpha$ denotes the aspect ratio. The optimization problem is then formulated as
\begin{equation}
\label{eq41:optimization_problem}
\resizebox{.85\linewidth}{!}{$
\begin{aligned}
\min_{\Lambda,\mathbf{c},\alpha} \left\{
\sum_{k} \rho\left(
\right.\right.&\|\pi_r(\boldsymbol{\chi}_k, \mathbf{d}_k)-\mathbf{d}_k\|^2 \\
&+ \bigl(\|\pi_r(\boldsymbol{\chi}_k,\mathbf{d}_k)\| - \|\hat{\mathbf{d}}_k\|\bigr)^2
\left.\left.
\right)
\right\},
\end{aligned}$}
\end{equation}
where $\rho$ is a robust loss function and $\pi_r$ is the projection function $\pi_r(a,b)=\frac{a^Tb}{a^Ta}a$, $\pi_r:\mathbb{R}^2 \times \mathbb{R}^2 \rightarrow \mathbb{R}^2$. The first term in Eq. \eqref{eq41:optimization_problem} corresponds to the reprojection error orthogonal to the radial line. The second term measures the error along the radial line, where $\|\hat{\mathbf{d}}_k\|$ is a reference radial distance introduced to enforce local smoothness. This reference distance is computed for each point $k$ via a local polynomial fitting procedure. Specifically, a polynomial $\mathcal{P}_k$ is fitted to the point and its $2n$ neighbors, mapping their normalized radial distances $\|\boldsymbol{\chi}_j\|$ to their observed image-plane radial distances $\|\mathbf{d}_j\|$ for $j \in [k-n, k+n]$. Then, the $k$-th reference distance is the polynomial's prediction, i.e., $\|\hat{\mathbf{d}}_k\| = \mathcal{P}_k(\|\boldsymbol{\chi}_k\|)$. This fitting procedure adopts a size-decreasing window strategy, initially using larger sliding windows (e.g., 100 points) for stable estimates and progressively reducing to smaller windows (e.g., 5 points) to improve local adaptability for various complex lens models.

\section{Pipeline of Proposed Generic-Parametric Hybrid Calibration}\label{sec5}
\subsection{Initial Estimation of Extrinsic Parameters}
Suppose $N$ images are captured of a calibration pattern containing $m$ grid points. Neglecting camera distortion, we initialize the extrinsic parameters, i.e., the transformation from each calibration pattern to the camera coordinate frame, using Zhang's approach \cite{Zhang2000}.

Ignoring lens distortion, the projection of a 3D point onto the image plane can be expressed as   
\begin{equation}
\label{eq20}
\begin{aligned}
\begin{pmatrix}
u_{il} \\
v_{il} \\
1
\end{pmatrix} &= \frac{1}{z} K\begin{pmatrix}\mathbf{r}_{i,1} & \mathbf{r}_{i,2} & \mathbf{t}_{i}\end{pmatrix}
\begin{pmatrix}
X_{il} \\
Y_{il} \\
1
\end{pmatrix}
\overset{\triangle}{=} H_{I}^{W}\begin{pmatrix}
X_{il} \\
Y_{il} \\
1
\end{pmatrix}.
\end{aligned}
\end{equation}
The homography matrix $H_I^W$ in Eq. \eqref{eq20} is derived under the assumption of negligible lens distortion when mapping between the calibration pattern and the image plane. This simplification is made for the initial estimation. In contrast, the homography $H_{ij}$ presented in Eq. \eqref{eq18} provides an exact mapping between pairs of calibration pattern planes, immune to any lens distortion.

Notably, the extrinsic parameters initialized by Zhang's method are unambiguous.

\begin{proposition}
\label{prop5}
The traditional parametric-model calibration approach \cite{Zhang2000} yields a unique solution for the transformation from the pattern coordinate system to the camera coordinate system without ambiguity.
\end{proposition}

\begin{proof}
We demonstrate through contradiction that the poses recovered by Eq. \eqref{eq20} are unique under the assumption of no distortion. Suppose there exist multiple pose solutions. Then, for any invertible matrix $B$, we have $H_{I}^{W}=\frac{1}{z} K B^{-1} B\begin{pmatrix}\mathbf{r}_{i,1} & \mathbf{r}_{i,2} & \mathbf{t}_{i}\end{pmatrix}$, that is, any combination of intrinsic matrix $K^{\prime}=K B^{-1}$ and pose solution $B\begin{pmatrix}\mathbf{r}_{i,1} & \mathbf{r}_{i,2} & \mathbf{t}_{i}\end{pmatrix}$ satisfies the above constraint. Additionally, the transformed vectors $B \mathbf{r}_{i,1}$ and $B \mathbf{r}_{i,2}$ for any pose solution must satisfy $\|B \mathbf{r}_{i,1}\|=\|B \mathbf{r}_{i,2}\|=1$ and $(B \mathbf{r}_{i,1})^{T} B \mathbf{r}_{i,2}=0$, therefore $B$ must be an orthogonal matrix. Both intrinsic matrices $K^{\prime}$ and $K$ are upper-triangular, so $B$ is orthogonal and upper-triangular simultaneously. Using $b_{i,j}$ to denote the $i$-th row and the $j$-th column of $B$, it follows that
\begin{equation}
\label{eq21}
\resizebox{.97\linewidth}{!}{$
\begin{array}{lll}
b_{1,1}^{2}+b_{2,1}^{2}+b_{3,1}^{2}=1 & \Rightarrow b_{1,1}^{2}=1 & \Rightarrow b_{1,1}= \pm 1 \\
b_{1,2}^{2}+b_{2,2}^{2}+b_{3,2}^{2}=1 & \Rightarrow b_{1,2}^{2}+b_{2,2}^{2}=1 & \Rightarrow b_{2,2}= \pm 1 \\
b_{1,3}^{2}+b_{2,3}^{2}+b_{3,3}^{2}=1 & \Rightarrow b_{1,3}^{2}+b_{2,3}^{2}+b_{3,3}^{2}=1 & \Rightarrow b_{3,3}= \pm 1 \\
b_{1,1} b_{1,2}+b_{2,1} b_{2,2}+b_{3,1} b_{3,2}=0 & \Rightarrow b_{1,1} b_{1,2}=0 & \Rightarrow b_{1,2}=0 \\
b_{1,1} b_{1,3}+b_{2,1} b_{2,3}+b_{3,1} b_{3,3}=0 & \Rightarrow b_{1,1} b_{1,3}=0 & \Rightarrow b_{1,3}=0 \\
b_{1,2} b_{1,3}+b_{2,2} b_{2,3}+b_{3,2} b_{3,3}=0 & \Rightarrow b_{1,2} b_{1,3}+b_{2,2} b_{2,3}=0 & \Rightarrow b_{2,3}=0 \\
b_{2,1}=0 & b_{2,1}=0 & b_{2,1}=0 \\
b_{3,1}=0 & b_{3,1}=0 & b_{3,1}=0 \\
b_{3,2}=0 & b_{3,2}=0 & b_{3,2}=0
\end{array}.$}
\end{equation}
Then we obtain $B$ as follows:
\begin{equation}
\label{eq22}
B=\begin{pmatrix} 
\pm 1 & 0 & 0 \\
0 & \pm 1 & 0 \\
0 & 0 & \pm 1
\end{pmatrix}.
\end{equation}
Moreover, a valid intrinsic matrix $K'$ must have positive principal point components, which constrains the matrix $B$ to be the identity matrix. Consequently, the pose solution derived from Eq. \eqref{eq20} is unambiguous.
\end{proof}

Proposition \ref{prop5} indicates that the unambiguous extrinsic parameters provided by \cite{Zhang2000} correspond precisely to the case $\tilde{R}_i=\Lambda R_i$ and $\tilde{\mathbf{t}}_i=\Lambda \mathbf{t}_i$, where $\Lambda$ is the identity matrix. In other words, these initial estimates correspond to the rotation ambiguity instance in Proposition \ref{prop1}. This initialization provides a well-defined starting point that improves the stability and convergence of subsequent global optimization.

\subsection{Global Optimization of Extrinsic Parameters}
We construct a connectivity graph to facilitate global optimization. Each pattern pose is represented as a node, with edges connecting nodes whose corresponding patterns exhibit overlapping regions on the image plane. Edge weights are determined by overlap quality, and the node with the highest cumulative edge weight serves as the root node, defining the world coordinate system. When multiple disconnected subgraphs exist, only the largest is utilized. Users may integrate the disconnected components by capturing additional calibration images.

\begin{figure}[!t]
\centering
\includegraphics[width=0.46\textwidth]{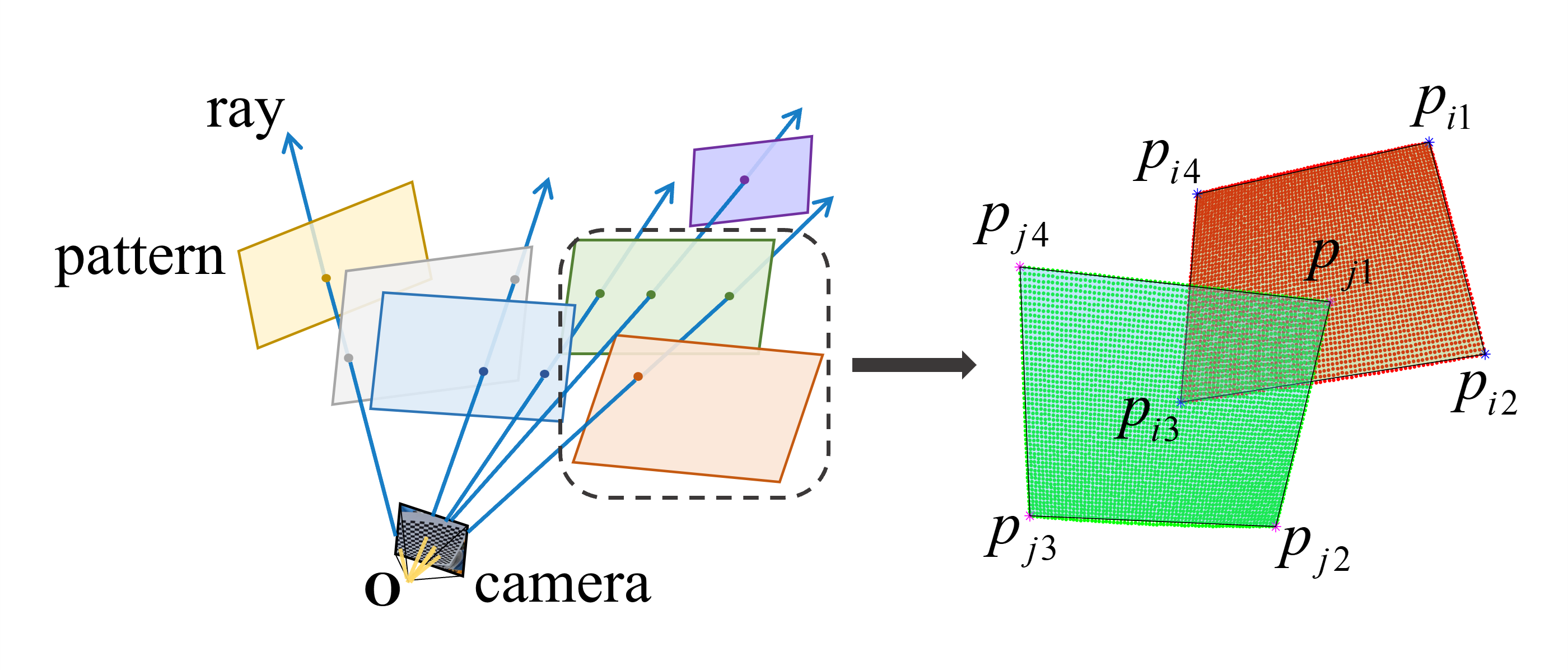}
\caption{An example of the simplified quadrilateral structure. \textit{Left:} Multiple poses of a single calibration pattern are captured by a fixed camera. \textit{Right:} To efficiently build the connectivity graph, each pattern's projection in the image plane is simplified to a bounding quadrilateral defined by its four corner points.}
\label{fig2:quadrilateral}
\end{figure}

A key step in building the connectivity graph is to identify which pattern poses have overlapping regions on the image plane, as these regions allow geometric constraints (e.g., collinearity as in Fig. \ref{fig_1}) to be exploited. However, determining pattern overlaps for $N$ images traditionally involves $C_N^2$ pairwise comparisons. Point-by-point comparison further escalates computational complexity to $O(C_N^2 m^2)$. While Schöps \cite{Schops2020} addressed this through image downsampling and direct pixel comparison, such an approach risks precision loss. To maintain accuracy while improving efficiency, we propose a strategy illustrated in Fig. \ref{fig2:quadrilateral}, where each calibration pattern region in the image plane is simplified to a quadrilateral. Let $\{\mathbf{p}_{i1}, \mathbf{p}_{i2}, \mathbf{p}_{i3}, \mathbf{p}_{i4}\}$ denote the four vertices of the $i$-th quadrilateral, and let linear inequalities $\mathbf{l}_{ik}: a_{ik}^1 u+a_{ik}^2 v\leq b_{ik}$ define the $k$-th edge of the $i$-th quadrilateral, where coefficients $a_{ik}^1, a_{ik}^2, b_{ik}$ ensure that image points $\mathbf{u}$ within the $i$-th quadrilateral satisfy $A_i \mathbf{u} \leq \mathbf{b}_i$, namely,
\begin{equation}
\label{eq43:quadrilateral_constraint}
\begin{pmatrix}
a_{i1}^1 & a_{i1}^2 \\
a_{i2}^1 & a_{i2}^2 \\
a_{i3}^1 & a_{i3}^2 \\
a_{i4}^1 & a_{i4}^2
\end{pmatrix}\begin{pmatrix}u\\v\end{pmatrix} \leq\begin{pmatrix}
b_{i1} \\
b_{i2} \\
b_{i3} \\
b_{i4}
\end{pmatrix}.
\end{equation}

The existence of an edge between nodes $i$ and $j$ can be efficiently verified by checking whether at least one vertex from either quadrilateral satisfies
\begin{equation}
\label{eq44:edge_verification}
\begin{pmatrix}A_i\\A_j\end{pmatrix}\begin{pmatrix}u\\v\end{pmatrix} \leq\begin{pmatrix}b_i\\b_j\end{pmatrix}.
\end{equation}
This formulation also enables rapid identification of image points within overlapping regions.

Using the constructed connectivity graph, we perform global optimization to refine the initial extrinsic parameters. According to Eq. \eqref{eq17} and \eqref{eq18}, when nodes $i$ and $j$ are connected by an edge in the connectivity graph, their homography relationship is $\mathbf{X}_{jp} = H_{ij}\mathbf{X}_{il}$. The loss function for minimizing reprojection error is formulated as
\begin{equation}
\label{eq46:loss_function}
L = \sum_{i,j} \sum_{l,p} \|H_{ij} \mathbf{X}_{il} - \mathbf{X}_{jp}\|.
\end{equation}
Notably, the optimization problem $\min_{\tilde{R}_i, \tilde{\mathbf{t}}_i} L$ is independent of any specific parametric camera model, which avoids the systematic errors often introduced by model-lens mismatches. This yields robust estimates for $\tilde{R}_i$ and $\tilde{\mathbf{t}}_i$, particularly across cameras with diverse lens distortion characteristics.

While the globally optimized poses are subject to the pose ambiguity characterized in Proposition \ref{prop1}, our proposed pipeline effectively resolves it. The optimization is initialized using the unambiguous extrinsic parameters from Zhang's method (see Proposition \ref{prop5}). This strong initialization reliably facilitates convergence towards the specific rotation ambiguity case ($\tilde{R}_i=\Lambda R_i$ and $\tilde{\mathbf{t}}_i=\Lambda \mathbf{t}_i$), rather than reflection ambiguity. Consequently, this constrained ambiguity can then be explicitly and reliably resolved by the method detailed in Section \ref{sec4}.

\subsection{Intrinsic Parameters Computation}
Our calibration framework supports both parametric and generic camera models. The procedure for determining the camera's intrinsic parameters using the generic camera model is not further elaborated in this work for the sake of brevity. For parametric models, the mapping between the normalized image coordinate $\mathbf{x}_k$ and the image coordinate $\mathbf{u}_k$ is established through distortion coefficients and distortion functions. The global index $k$ enumerates the 3D points aggregated from all pattern poses. For example, Zhang \cite{Zhang2000} utilized a set of intrinsic parameters (focal length, principal point, aspect ratio, and distortion coefficients) combined with radial and tangential distortion functions to uniquely map the normalized image coordinates onto the image plane.

\begin{figure}[!t]
\centering
\includegraphics[width=0.49\textwidth]{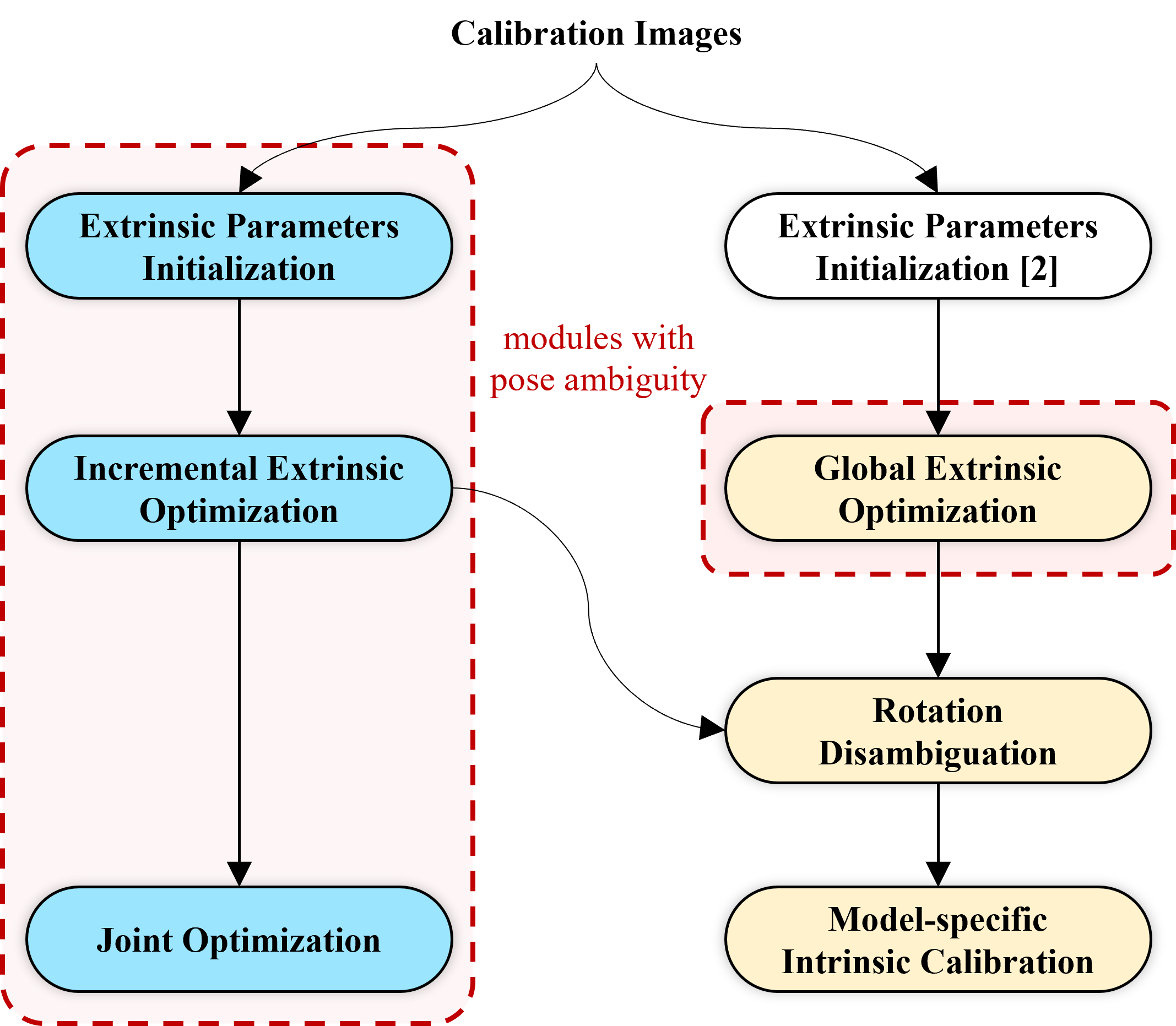}
\caption{Generic-parametric hybrid calibration framework. Blue blocks: existing generic calibration approaches; Yellow blocks: the proposed approach; Dashed boxes: modules with pose ambiguity.}
\label{fig3:calibration_framework}
\end{figure}

In this work, we employ the Levenberg-Marquardt optimization algorithm \cite{Levenberg1944,Marquardt1963} to compute the intrinsic parameters of the selected camera model by minimizing the reprojection error,
\begin{equation}
\min_{\boldsymbol{\theta}} \sum_{k}\|\Psi_{\boldsymbol{\theta}}(\mathbf{X}_{k}^{C})-\mathbf{u}_{k}\|^{2},
\label{eq:intrinsic_optimization}
\end{equation}
where $\Psi_{\boldsymbol{\theta}}(\mathbf{X}_{k}^{C})$ denotes the camera-specific projection function that maps 3D points $\mathbf{X}_{k}^{C}$ to corresponding distorted image coordinates $\mathbf{u}_{k}$, parameterized by the intrinsic parameters $\boldsymbol{\theta}$ of the selected camera model. Algorithm \ref{alg:hybrid_calibration} and Fig. \ref{fig3:calibration_framework} summarize the proposed complete calibration pipeline.

\begin{algorithm}[t]
\caption{Pipeline of Generic-Parametric Hybrid Calibration}
\label{alg:hybrid_calibration}
\begin{algorithmic}[1]
\Require Images of a calibration pattern at multiple poses.
\Ensure Intrinsic and extrinsic parameters of the camera.
\State Compute initial extrinsic parameters using Zhang's approach \cite{Zhang2000}.
\State Construct the connectivity graph and globally optimize the extrinsic parameters, yielding ambiguous estimates $(\tilde{R}_i, \tilde{\mathbf{t}}_i)$ using Eq. \eqref{eq46:loss_function}.
\State Linearly determine candidate solutions for $\Lambda$, as expressed in Eq. \eqref{eq40:ab_solution}, and select the optimal one based on geometric constraints.
\State Refine the selected solution via nonlinear optimization (Eq. \ref{eq41:optimization_problem}), and then rectify the extrinsic parameters: $R_i \leftarrow \Lambda^T \tilde{R}_i$, $\mathbf{t}_i \leftarrow \Lambda^T \tilde{\mathbf{t}}_i$.
\State Fix the rectified extrinsic parameters and determine the intrinsic parameters according to the selected camera model.
\end{algorithmic}
\end{algorithm}

\section{Experiments}\label{sec6}
We evaluate the accuracy and robustness of the proposed approach using four representative camera models:
\begin{itemize}
\item \textit{Model 1 (M1)}: Zhang \cite{Zhang2000} with two radial distortion terms.
\item \textit{Model 2 (M2)}: Zhang \cite{Zhang2000} with three radial distortion terms.
\item \textit{Model 3 (M3)}: Scaramuzza \cite{Scaramuzza2006}.
\item \textit{Model 4 (M4)}: The generic camera model \cite{Schops2020}, employing the B-spline surface for lens distortion.
\end{itemize}

Models $1\sim3$ are parametric models integrated into the MATLAB camera calibration toolbox, which are suitable for lenses ranging from pinhole-like cameras to fisheye cameras. Model 4 is taken from a state-of-the-art open-source generic calibration toolbox \cite{Schops2020}. Our approach uses the same B-spline control point density with \cite{Schops2020} to ensure a fair comparison. Two-view relative pose estimation and triangulation are solved using OpenGV \cite{Kneip2014}.

\begin{figure}[!t]
\centering
\includegraphics[width=0.45\textwidth]{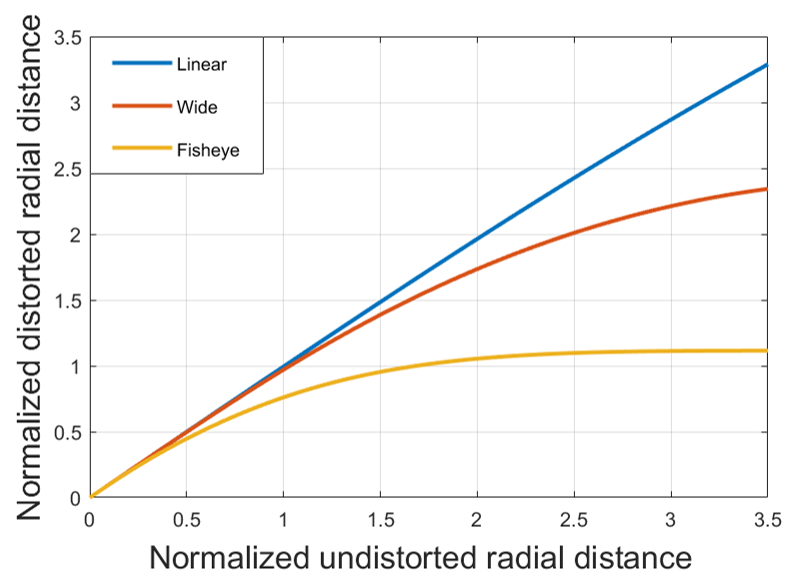}
\caption{Distortion characteristics of simulated linear, wide-angle, and fisheye lens models.}
\label{fig4:distortion_characteristics}
\end{figure}

Simulation experiments are conducted using a rational model $(1+k_1r^2+k_2r^4+k_3r^6)/(1+k_4r^2+k_5r^4+k_6r^6)$ to represent the radial distortion, where $k_1 \sim k_6$ are coefficients and $r$ denotes the radial distance on the normalized image plane. In this way, we ensure a fair comparison across the evaluated camera models by preventing overfitting. For example, the representation of Model 2 \cite{Zhang2000}, $1+k_1r^2+k_2r^4+k_3r^6$, does not exactly match the simulated distortion unless $k_4 \sim k_6$ are set to be zeros. Three lens types (linear, wide-angle, and fisheye) are simulated through the rational model above to represent different distortion levels, as illustrated in Fig. \ref{fig4:distortion_characteristics}. The aspect ratios for linear, wide-angle, and fisheye lenses are 1.0013, 1.0020, and 1.0051, respectively.

Real-world experiments involve four datasets we captured with cameras of increasing distortion: a Sony A6400 (224 images), a GoPro in linear mode (215 images), a GoPro in wide mode (217 images), and an Insta360 (136 images).

\subsection{Simulation Test}
\subsubsection{Validation of Pose Ambiguity}
This experiment is designed to validate the two distinct types of pose ambiguity outlined in Proposition \ref{prop1}, and to verify that the initialization from Zhang's method \cite{Zhang2000} reliably facilitates convergence to the rotation ambiguity, as stated in Proposition \ref{prop5}. We conduct 50 Monte Carlo simulations for each image set size (from 10 to 30), employing two different initializations for extrinsic parameter optimization in Eq. \eqref{eq46:loss_function}. For each view, we obtain an initial estimate from \cite{Zhang2000}, $(\hat{R}_i, \hat{\mathbf{t}}_i)$, which serves as a common starting point for two initialization strategies:

\begin{itemize}
\item \textit{Rotation-based initialization:} use the initial estimate $(\hat{R}_i, \hat{\mathbf{t}}_i)$ directly. 
\item \textit{Reflection-based initialization:} select a rotation matrix $\hat{\Lambda}$, then modify the initial estimate by $\hat{R}_i \leftarrow \hat{\Lambda} \hat{R}_i$ and $\hat{\mathbf{t}}_i \leftarrow \hat{\Lambda} \hat{R}_i \operatorname{diag}(1,1,-1) \hat{R}_i^{T} \hat{\mathbf{t}}_i$.
\end{itemize}

After obtaining the optimized poses $(\tilde{R}_i, \tilde{\mathbf{t}}_i)$ using Eq. \eqref{eq46:loss_function}, we evaluate the property of the resulting ambiguity. First, we compute the actual pose ambiguity matrix $\Lambda = \tilde{R}_i R_i^T$, the discrepancy between the optimized rotation and its corresponding ground-truth. Then, two rectification strategies are used to recover the translation: a rotation-based strategy yielding $\Lambda^{-1}\tilde{\mathbf{t}}_i$, and a reflection-based strategy yielding $R_i \operatorname{diag}(1,1,-1) R_i^T \Lambda^{-1}\tilde{\mathbf{t}}_i$. Finally, the translation errors of the rectified outputs are assessed against the ground-truth.

As shown in Fig. \ref{fig5:ambiguity validation}(a), when the optimization is initialized with a reflection strategy, the translation error is minimized under the reflection-based rectification strategy, confirming that the optimized pose $(\tilde{R}_i, \tilde{\mathbf{t}}_i)$ corresponds to the reflection ambiguity. Conversely, Fig. \ref{fig5:ambiguity validation}(b) shows that when using Zhang's initialization, the translation error is minimized under the rotation-based rectification strategy, confirming the rotation ambiguity case. This experiment not only confirms the existence of rotation ambiguity and reflection ambiguity, but also validates that the initialization from \cite{Zhang2000} can act as a deterministic factor for converging to the rotation ambiguity.

\begin{figure}
\centering
\setlength{\tabcolsep}{0pt}
\begin{tabular}{c}
\footnotesize
\begin{tabular}{@{}l@{\hspace{1em}}l@{}}
\hspace*{-0.2in}
Reflection-based initialization: & \hspace*{-0.15in}
\raisebox{-0.2ex}{\includegraphics[height=1.8ex]{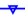}} Linear
\raisebox{-0.2ex}{\includegraphics[height=1.8ex]{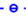}} Wide-angle
\raisebox{-0.2ex}{\includegraphics[height=1.8ex]{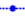}} Fisheye\\ [0.5em]
\vspace{1em}
\hspace*{-0.2in}
Rotation-based initialization:  & \hspace*{-0.15in}
\raisebox{-0.2ex}{\includegraphics[height=1.8ex]{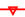}} Linear
\raisebox{-0.2ex}{\includegraphics[height=1.8ex]{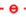}} Wide-angle
\raisebox{-0.2ex}{\includegraphics[height=1.8ex]{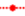}} Fisheye\\
\end{tabular} \\

\hspace*{-0.2in}
\begin{tabular}{@{}c@{\hspace{-3.5pt}}c@{}}
\includegraphics[width=0.25\textwidth]{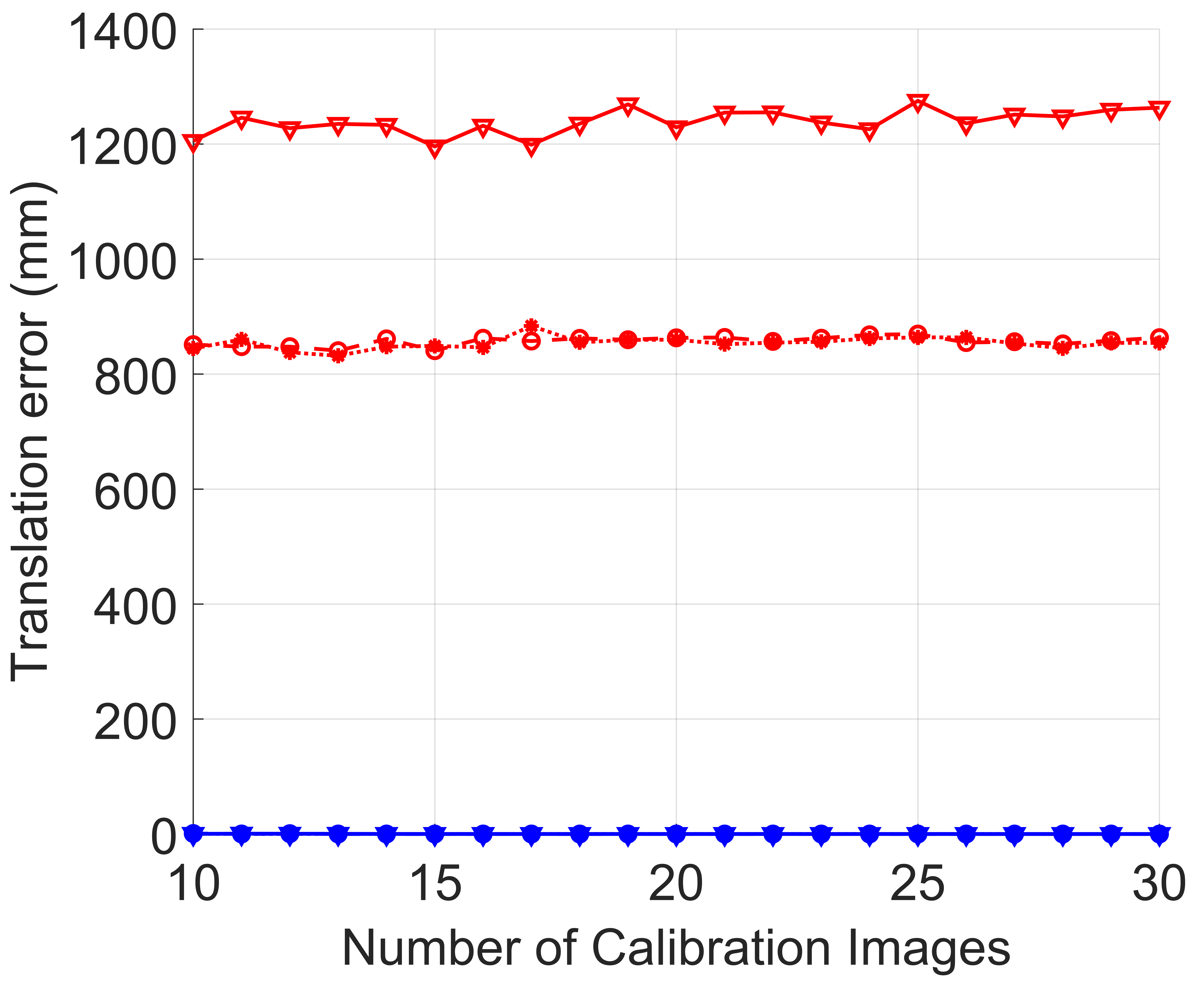}\label{fig5a:reflection} &
\includegraphics[width=0.25\textwidth]{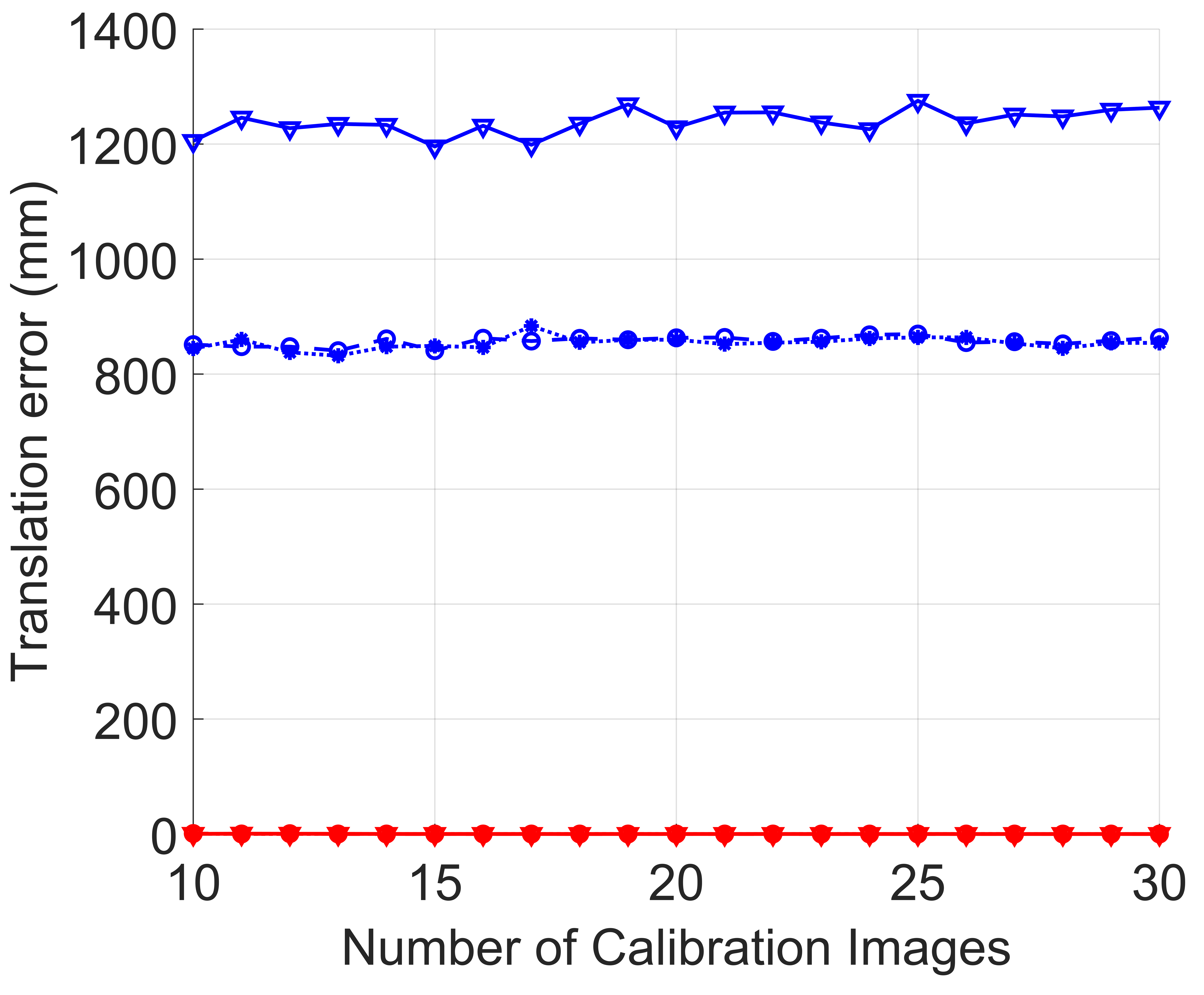}\label{fig5b:rotation} \\
\footnotesize (a) Reflection-based rectification & \footnotesize \hspace*{0.1in}(b) Rotation-based rectification \\
\end{tabular}
\end{tabular}
\caption{Behavior of extrinsic parameter estimation. Translation error under (a) the reflection-based rectification strategy: $\|R_i \operatorname{diag}(1,1,-1) R_i^T\Lambda^{-1}\tilde{\mathbf{t}}_i-\mathbf{t}_i\|$. (b) the rotation-based rectification strategy: $\|\Lambda^{-1}\tilde{\mathbf{t}}_i-\mathbf{t}_i\|$.}
\label{fig5:ambiguity validation}
\end{figure}

\begin{figure*}[!t]
\centering
\setlength{\tabcolsep}{0pt}
\begin{tabular}{c}
\scriptsize
\begin{tabular}{@{}l@{\hspace{0.3in}}l@{\hspace{0.3in}}l@{}}
\hspace*{-0.15in}\raisebox{0ex}{\color{blue}$\bullet$} Model4: Schöps (Unaligned) &
\raisebox{0ex}{\color{red}$\bullet$} Model4: Proposed (Unaligned) &
\raisebox{0ex}{\color{green}$\bullet$} Truth \\ \vspace{0.2em}
\end{tabular} \\

\begin{tabular}{@{}c@{\hspace{-0.1in}}c@{\hspace{-0.1in}}c@{}}
\scriptsize Linear lens & \scriptsize Wide-angle lens & \scriptsize Fisheye lens \\
\includegraphics[width=0.33\textwidth]{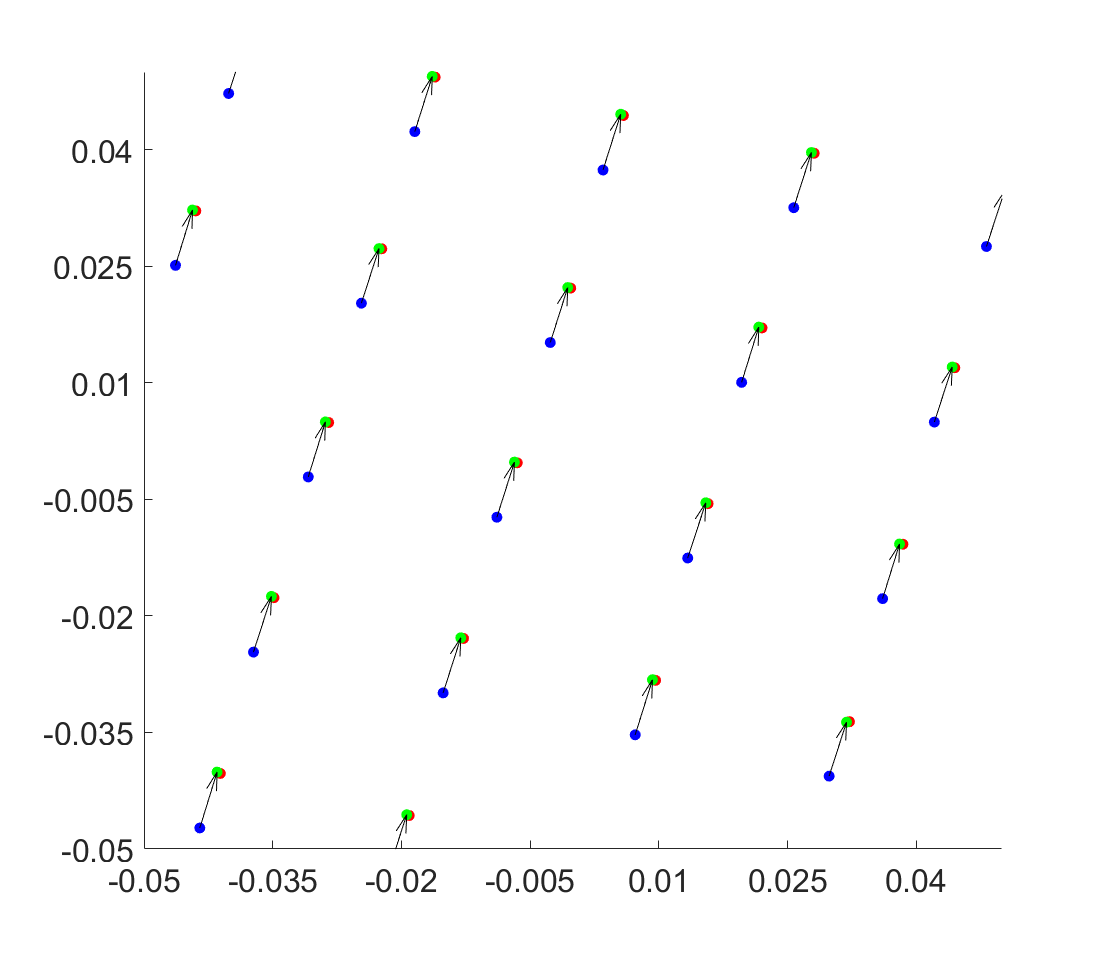} &
\includegraphics[width=0.35\textwidth]{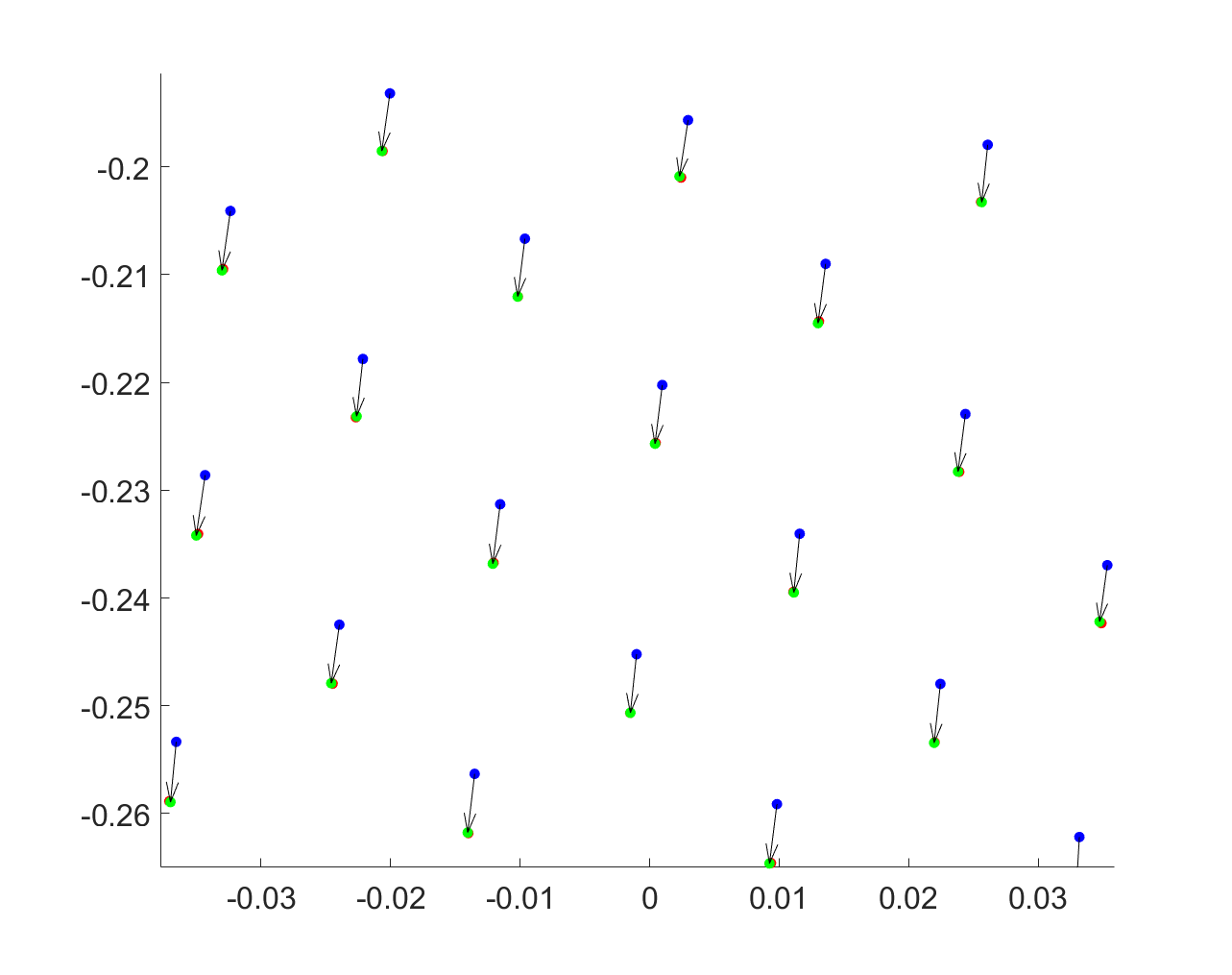} &
\includegraphics[width=0.33\textwidth]{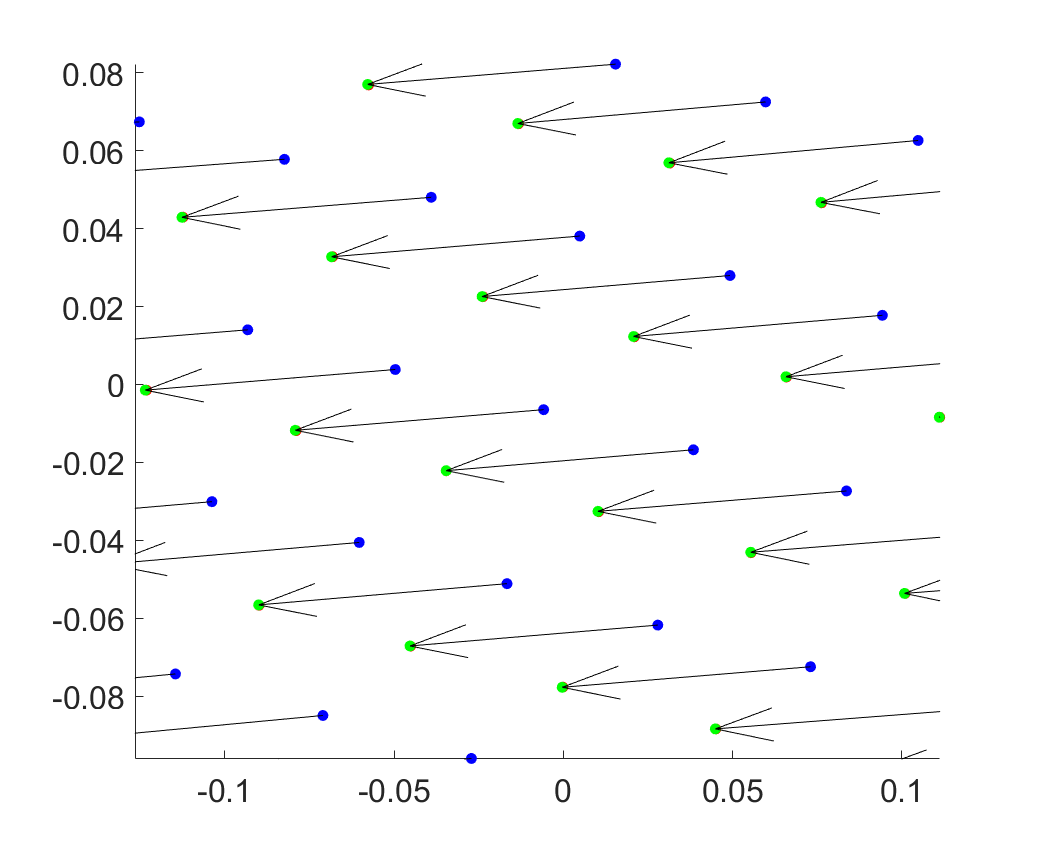} \\
\end{tabular}
\end{tabular}
\caption{Deviation of estimated normalized image coordinates. Arrows point from normalized image coordinates determined by Schöps \cite{Schops2020} to the ground truth.}
\label{fig6:unprojection_error}
\end{figure*}

\subsubsection{Performance Analysis}
We evaluate performance across three key aspects: extrinsic parameter accuracy, intrinsic parameter accuracy and two-view pose estimation accuracy. Note that intrinsic parameter accuracy is quantified by the deviation of the estimated normalized image coordinates from the ground truth.

50 independent calibration trials are carried out for each of three lens types: linear, wide-angle, and fisheye. In each trial, 30 images of the calibration pattern at various poses are generated. We then add Gaussian noise ($\sigma=0.2$ pixels) to the ideal projected image points. All spatial units are measured in millimeters. 

To assess the impact of the calibration on downstream tasks, we evaluate the accuracy of two-view relative pose estimation using 100 Monte Carlo synthetic scenes. Each scene consists of 300 random 3D points ($X, Y \in[-500,500]$, $Z \in[0,500]$). These points are projected onto a pair of images, with added Gaussian noise of standard deviation from 0.5 to 5.0 pixels. The left view is fixed at the world origin, while the right view is randomly positioned in a thin box defined by $X, Y \in[-300,300], Z \in[-100,100]$. The Euler angles of the right view are subject to distribution $N([0 \; 0 \; 0]^{T}, \operatorname{diag}[10^{2} \; 10^{2} \; 20^{2}])$. To ensure a robust evaluation that accounts for the inherent variability in calibration outcomes \cite{Hagemann2022}, the two-view estimation for each synthetic scene is performed using 10 distinct sets of camera intrinsic parameters, each obtained from an independent calibration procedure.

\textbf{Pose Ambiguity.}
To analyze the impact of pose ambiguity in generic calibration, we evaluate the performance of our method and Schöps's before ("Unaligned") and after ("Aligned") a corrective step. Specifically, we introduce an "alignment" adjustment, a corrective rotation using the ground truth, to rectify the ambiguous extrinsic parameters. The difference between the two approaches in the "Unaligned" scenario lies in their calibration strategies for handling the ambiguity matrix. Schöps (Unaligned) approximates by selecting $R_0$ via Eq. \eqref{eq: R0 in open-source toolbox}, whereas our method (Unaligned) determines $\Lambda$ using the method described in Section \ref{sec4}.

\begin{table*}[!t]
\renewcommand{\arraystretch}{1.3}
\caption{Errors in extrinsic parameters: mean / standard deviation.}
\label{table2:extrinsic_params}
\centering
\setlength{\tabcolsep}{0.3pt}
\begin{tabular}{cccccccc}

\toprule
\multicolumn{2}{c}{} & \multicolumn{2}{c}{Linear lens} & \multicolumn{2}{c}{Wide-angle lens} & \multicolumn{2}{c}{Fisheye lens} \\
\cmidrule(lr){3-4} \cmidrule(lr){5-6} \cmidrule(l){7-8}
\multicolumn{2}{c}{} & $\varepsilon_{\text{rot}}$ ($\cdot10^{-3}$ rad) & $\varepsilon_{\text{pos}}$ (mm) & $\varepsilon_{\text{rot}}$ ($\cdot10^{-3}$ rad) & $\varepsilon_{\text{pos}}$ (mm) & $\varepsilon_{\text{rot}}$ ($\cdot10^{-3}$ rad) & $\varepsilon_{\text{pos}}$ (mm) \\
\midrule
\multirow{2}{*}{Schöps \cite{Schops2020}}
 & Unaligned & 8.66 / 2.20 & 5.40 / 1.88 & 5.74 / 1.21 & 2.60 / 0.95 & 73.7 / 1.80 & 35.33 / 8.11 \\
& Aligned & 1.58 / 1.61 & 1.84 / 2.65 & 1.09 / 0.98 & 1.17 / 1.01 & 1.88 / 1.54 & 1.66 / 1.46 \\
\cmidrule(l){2-8}
\multirow{2}{*}{Proposed}
 & Unaligned & 0.35 / 0.39 & 0.23 / 0.31 & 0.18 / 0.16 & 0.08 / 0.08 & 0.44 / 0.38 & 0.25 / 0.17 \\
& Aligned & $\bm{0.13}$ / $\bm{0.12}$ & $\bm{0.05}$ / $\bm{0.04}$ & $\bm{0.11}$ / $\bm{0.09}$ & $\bm{0.04}$ / $\bm{0.03}$ & $\bm{0.20}$ / $\bm{0.13}$ & $\bm{0.13}$ / $\bm{0.06}$ \\
\bottomrule
\end{tabular}
\end{table*}

\noindent\textbullet~\textit{Impact on intrinsic and extrinsic parameters.} Fig. \ref{fig6:unprojection_error} presents the accuracy comparison between Proposed (M4) and Schöps (M4) \cite{Schops2020} for undistorting image points to their normalized image coordinates. The proposed method yields the normalized image coordinates closely matching the ground truth, whereas Sch\"{o}ps (M4) exhibits significant deviation, especially in fisheye scenarios. This confirms the existence of the pose ambiguity in Sch\"{o}ps' approach. Table \ref{table2:extrinsic_params} further quantifies the impact of ambiguity on extrinsic parameter accuracy across various distortion lenses. The results reveal that Sch\"{o}ps (M4) exhibits the largest mean errors due to the pose ambiguity. After applying an alignment adjustment, the pose errors of Sch\"{o}ps (M4) decrease significantly, by more than an order of magnitude particularly in fisheye camera scenarios.

\noindent\textbullet~\textit{Negative impact on two-view estimation.} As shown in Fig. \ref{fig7:two_view_geometry}, the alignment adjustment significantly improves Sch\"{o}ps' results. Especially in high-distortion fisheye lens, two-view relative poses of Sch\"{o}ps (Unaligned) deviate substantially from the ground truth. In contrast, the pose errors of the proposed method (Unaligned) remain stable across various lens scenarios, together with the aligned version.

\noindent\textbullet~\textit{Reliability of the proposed solution for pose ambiguity.} Table \ref{table2:extrinsic_params} shows that the extrinsic parameter errors of the proposed method (Aligned) are consistently low across all lens types, while its unaligned errors are only marginally larger. These results validate the effectiveness and stability of both our global optimization strategy and ambiguity solver.

\begin{figure*}[t]
\centering
\scriptsize
\begin{tabular}{@{}c@{\hspace{0.5em}}c@{\hspace{0.5em}}c@{\hspace{0.5em}}c@{}}
\raisebox{-0.2ex}{\includegraphics[height=2.2ex]{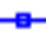}} Schöps (Unaligned) &
\raisebox{-0.2ex}{\includegraphics[height=2.2ex]{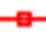}} Proposed (Unaligned) &
\raisebox{-0.2ex}{\includegraphics[height=2.2ex]{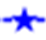}} Schöps (Aligned) &
\raisebox{-0.2ex}{\includegraphics[height=2.2ex]{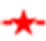}} Proposed (Aligned) \\
\end{tabular}

\vspace{1em}
\begin{tabular}{@{}c@{\hspace{-0.08in}}c@{\hspace{-0.05in}}c@{}}
\scriptsize Linear lens & \scriptsize Wide-angle lens & \scriptsize Fisheye lens \\[-0.1em]
\includegraphics[width=0.3\textwidth, height=3.6cm]{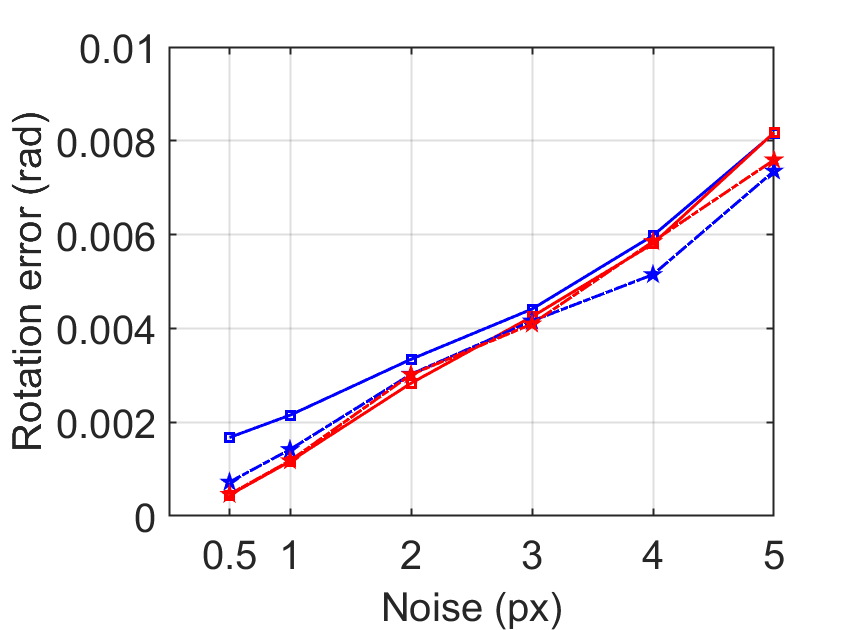} &
\includegraphics[width=0.3\textwidth, height=3.6cm]{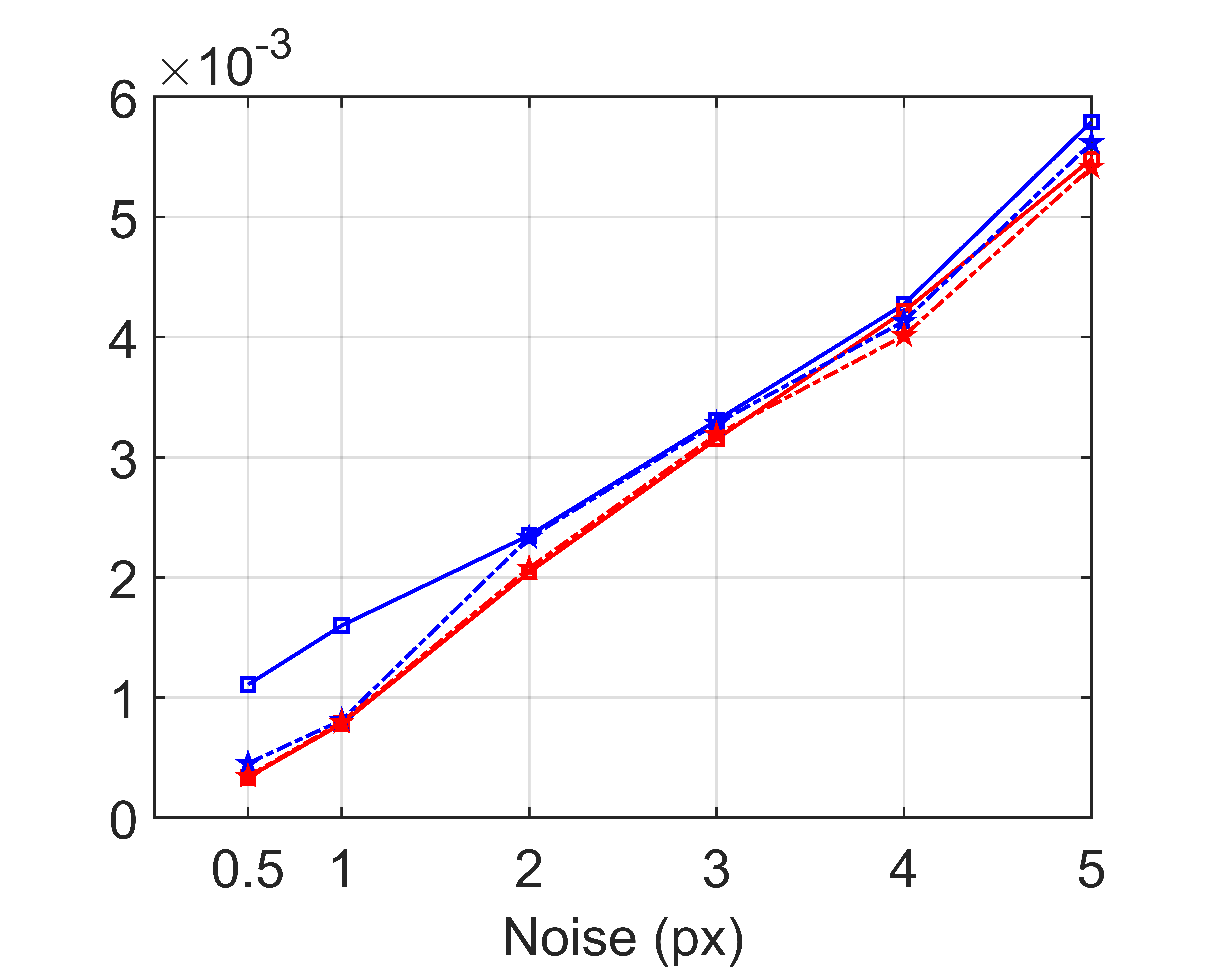} &
\includegraphics[width=0.3\textwidth, height=3.6cm]{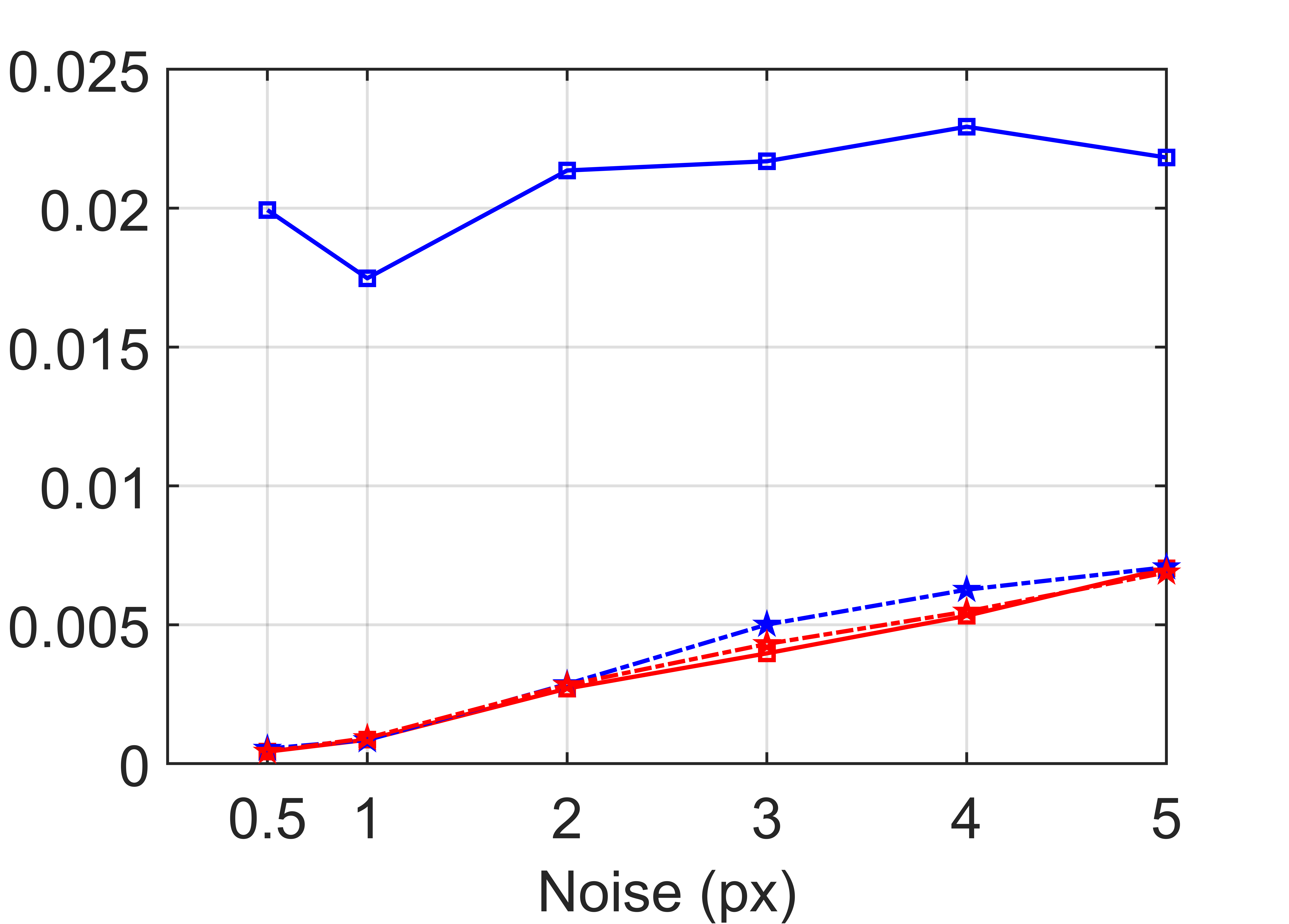} \\[0.2em]
\includegraphics[width=0.3\textwidth, height=3.6cm]{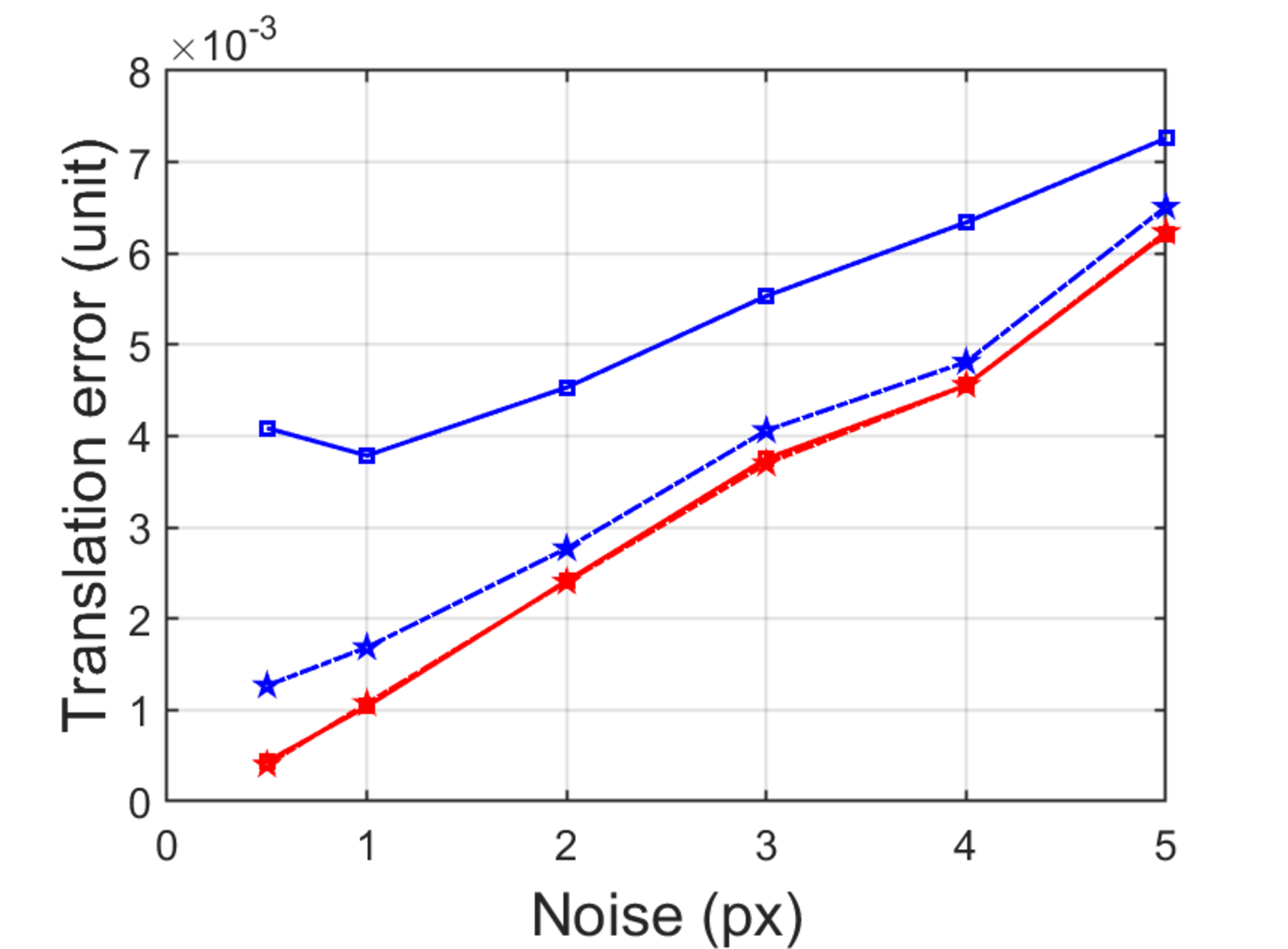} &
\includegraphics[width=0.3\textwidth, height=3.6cm]{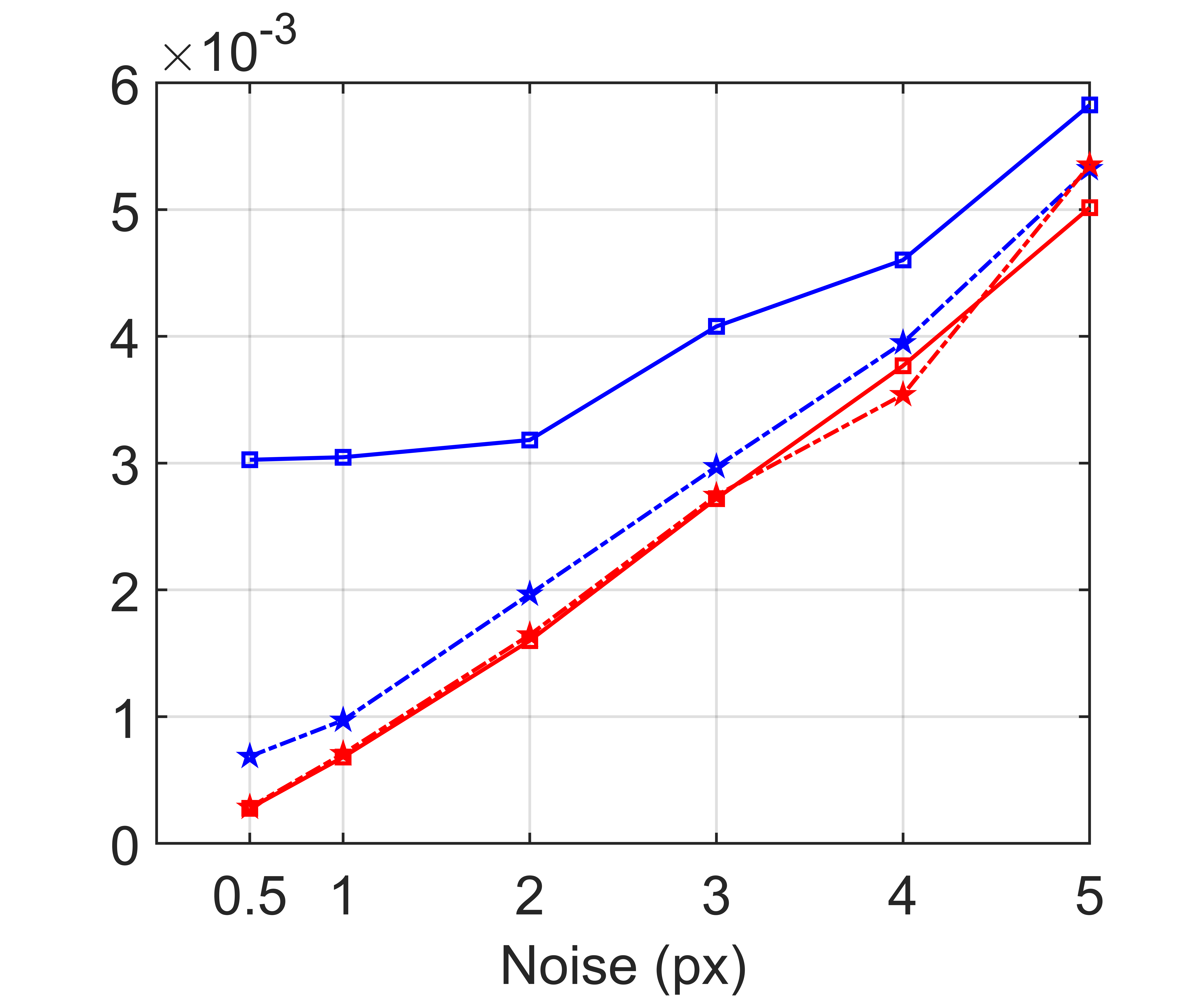} &
\includegraphics[width=0.3\textwidth, height=3.6cm]{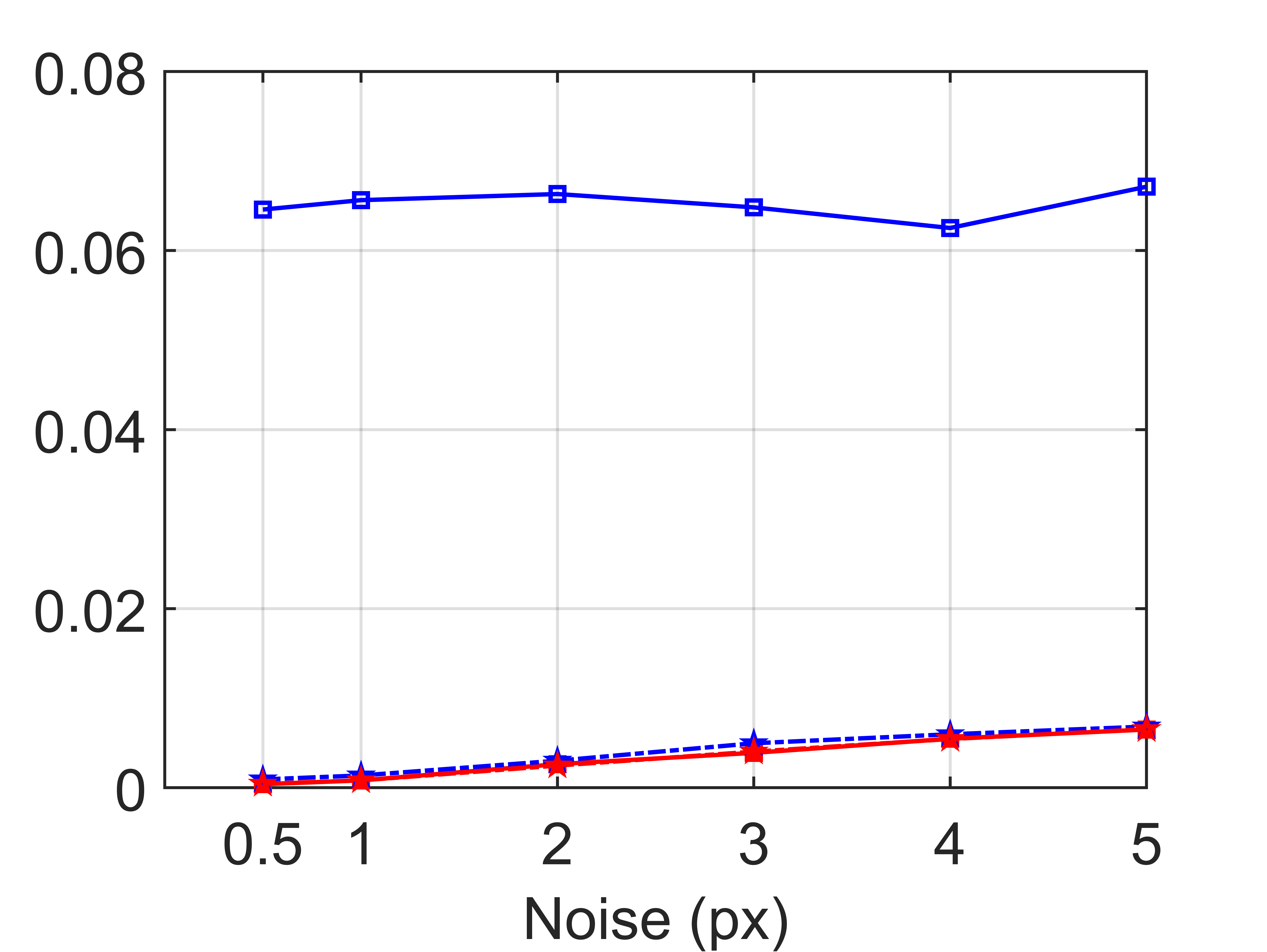}
\end{tabular}
\caption{Two-view pose estimation accuracy: comparison with Schöps' approach, with and without ambiguity correction. Translation error (bottom row) is normalized to unit length.}
\label{fig7:two_view_geometry}
\vspace{-0.6em} 
\end{figure*}

\begin{figure*}[!t]
\centering
\begin{tabular}{@{\hspace{-0.2em}}l@{\hspace{0.6em}}l@{\hspace{0.6em}}l@{}}
  \raisebox{-0.15ex}{\includegraphics[height=2ex]{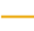}}\;\scriptsize M1: Zhang &
  \raisebox{-0.15ex}{\includegraphics[height=2ex]{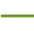}}\;\scriptsize M2: Zhang &
  \raisebox{-0.15ex}{\includegraphics[height=2ex]{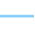}}\;\scriptsize M3: Scaramuzza \\[0.3em]
  \raisebox{-0.15ex}{\includegraphics[height=2ex]{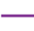}}\;\scriptsize M4: Sch\"{o}ps (Aligned) &
  \raisebox{-0.15ex}{\includegraphics[height=2ex]{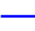}}\;\scriptsize Proposed (Aligned) &
  \raisebox{-0.15ex}{\includegraphics[height=2ex]{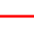}}\;\scriptsize Proposed (Unaligned)\\ \vspace{0.03em}
\end{tabular}
\vspace{0.5em}
\begin{tabular}{@{}c@{\hspace{0in}}c@{\hspace{0in}}c@{}}
  \scriptsize Linear lens & \scriptsize Wide-angle lens & \scriptsize Fisheye lens\\[-0.2em]
  \includegraphics[width=0.3\textwidth, height=3.6cm]{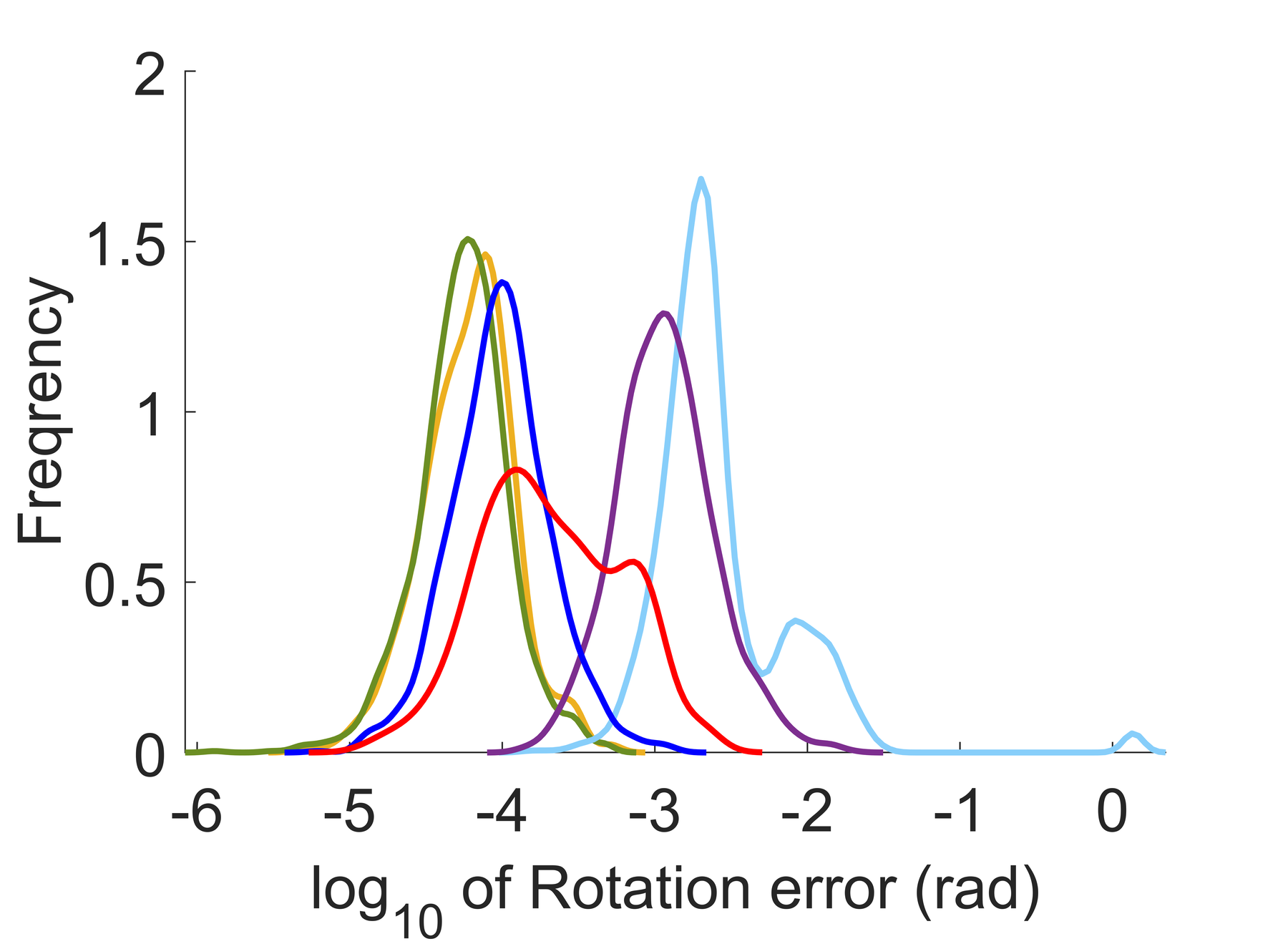} &
  \includegraphics[width=0.3\textwidth, height=3.6cm]{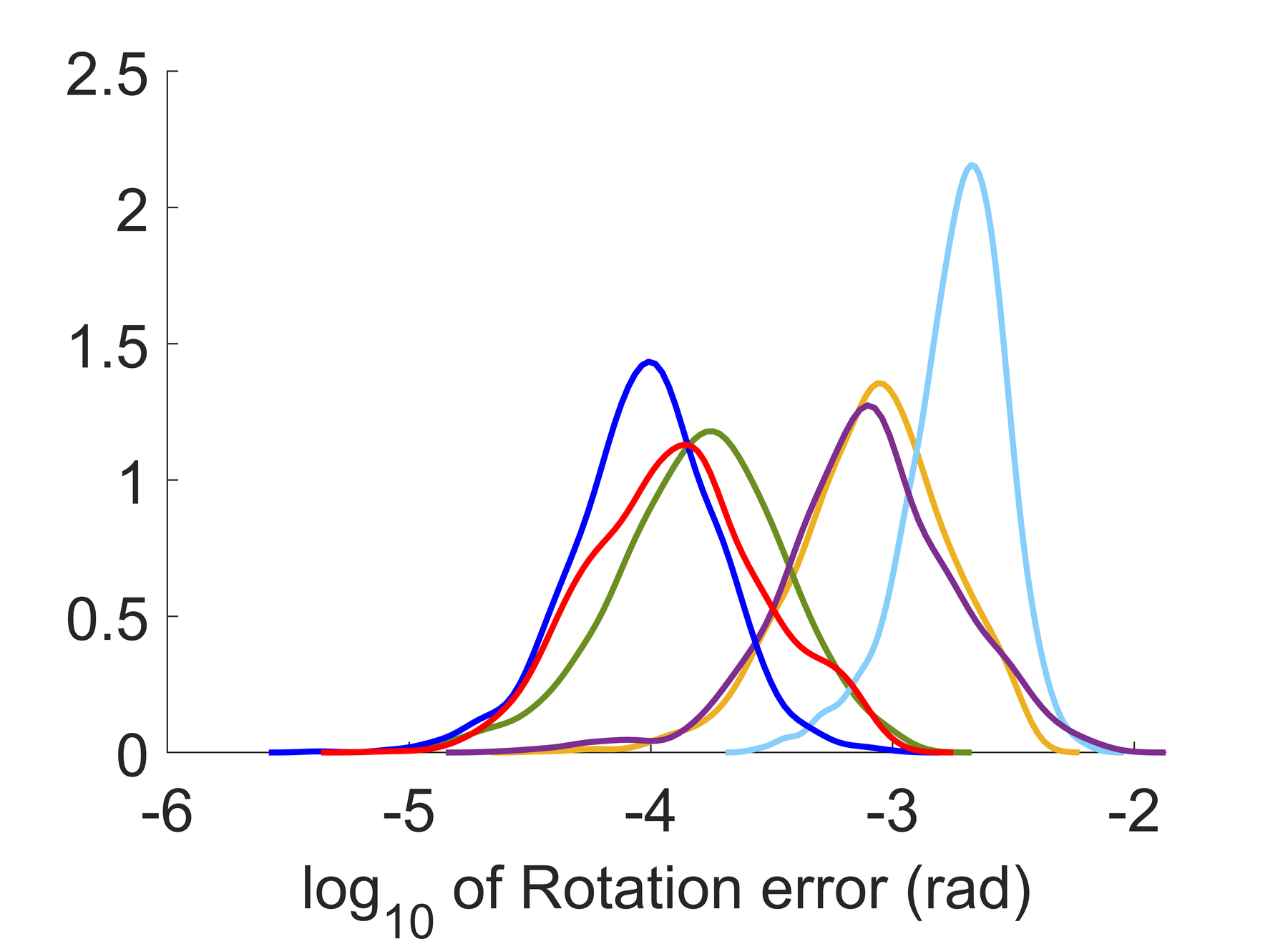} &
  \includegraphics[width=0.3\textwidth, height=3.6cm]{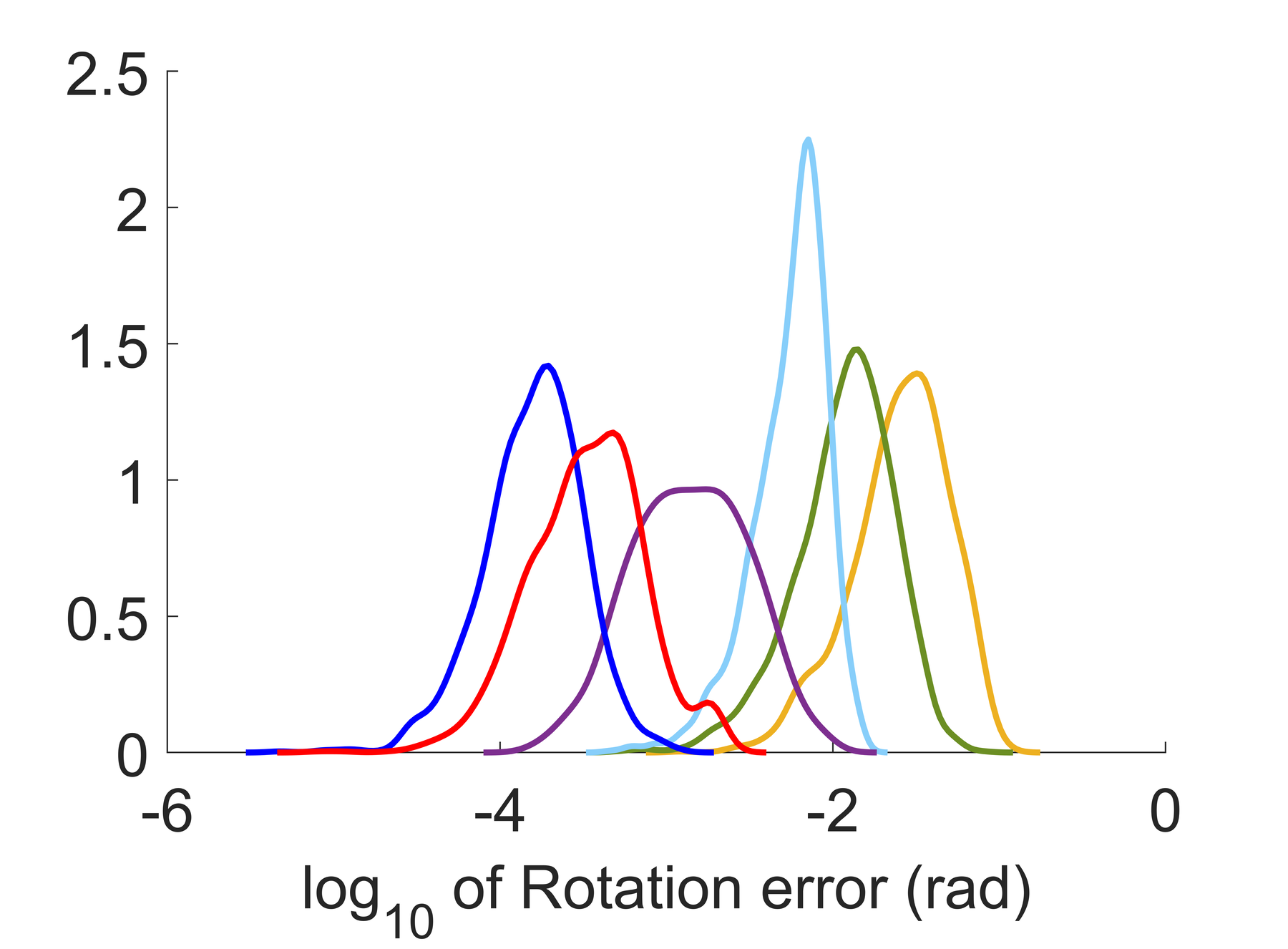} \\[0.15em]
  \includegraphics[width=0.3\textwidth, height=3.6cm]{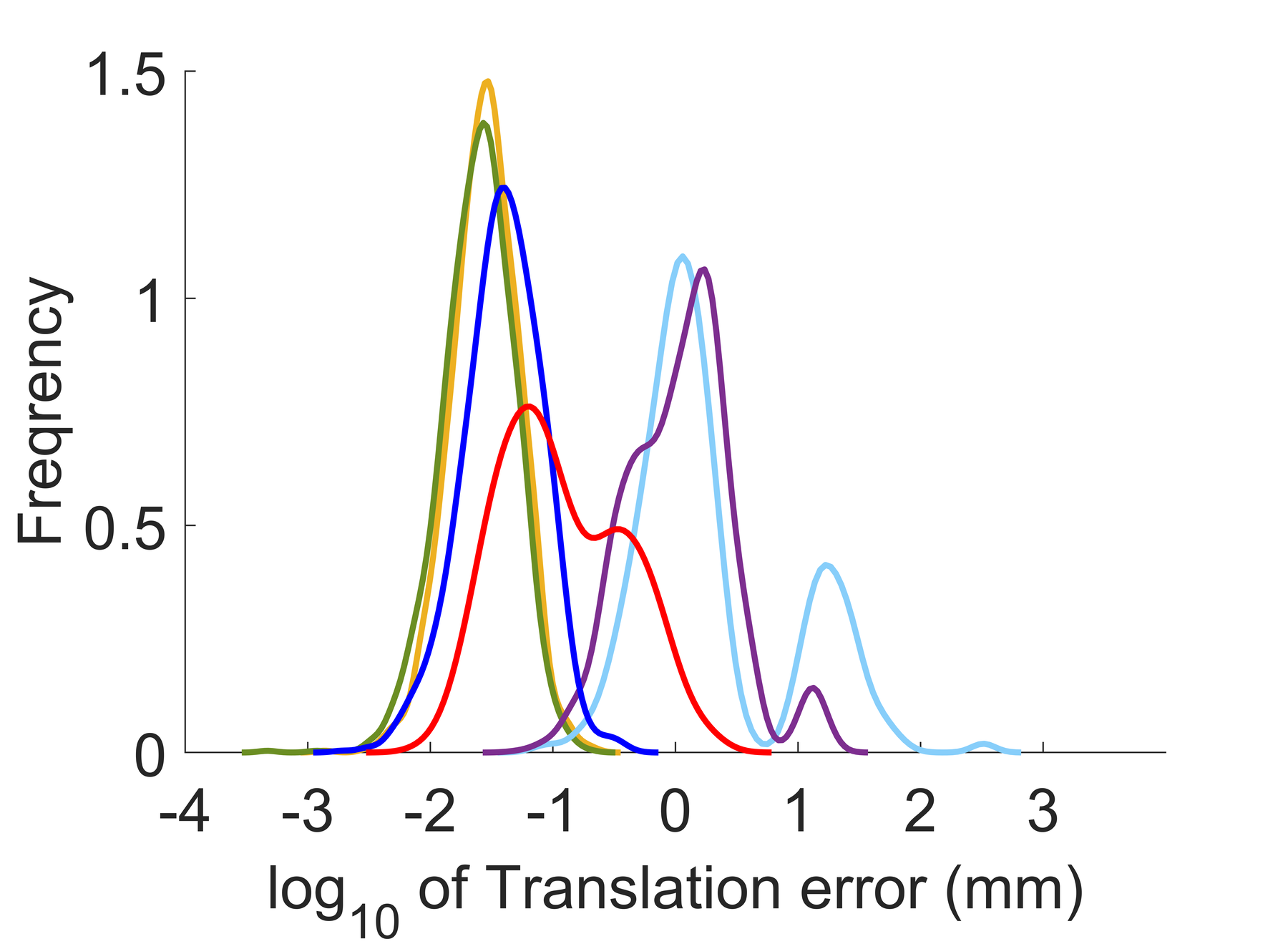} &
  \includegraphics[width=0.3\textwidth, height=3.6cm]{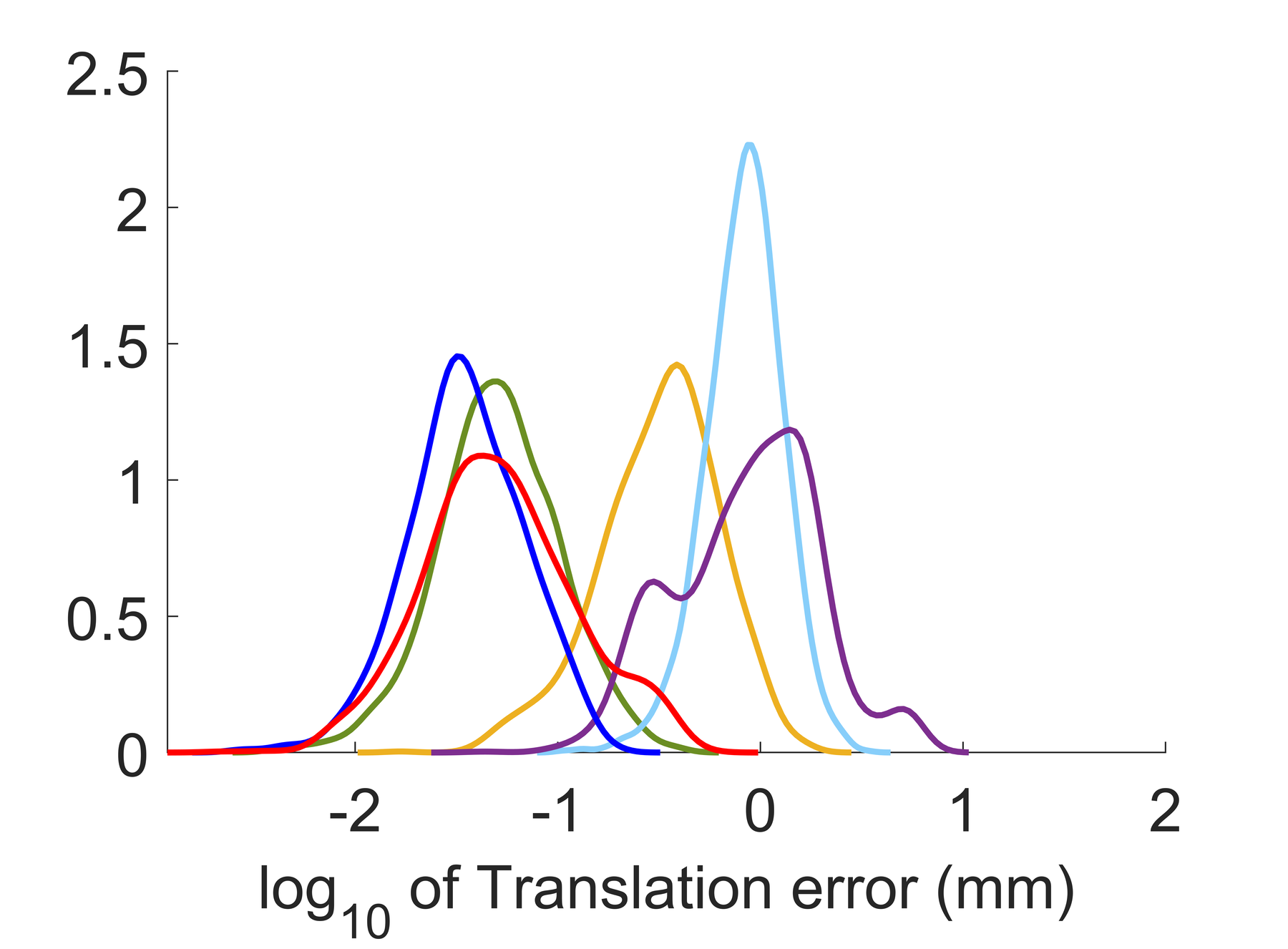} &
  \includegraphics[width=0.3\textwidth, height=3.6cm]{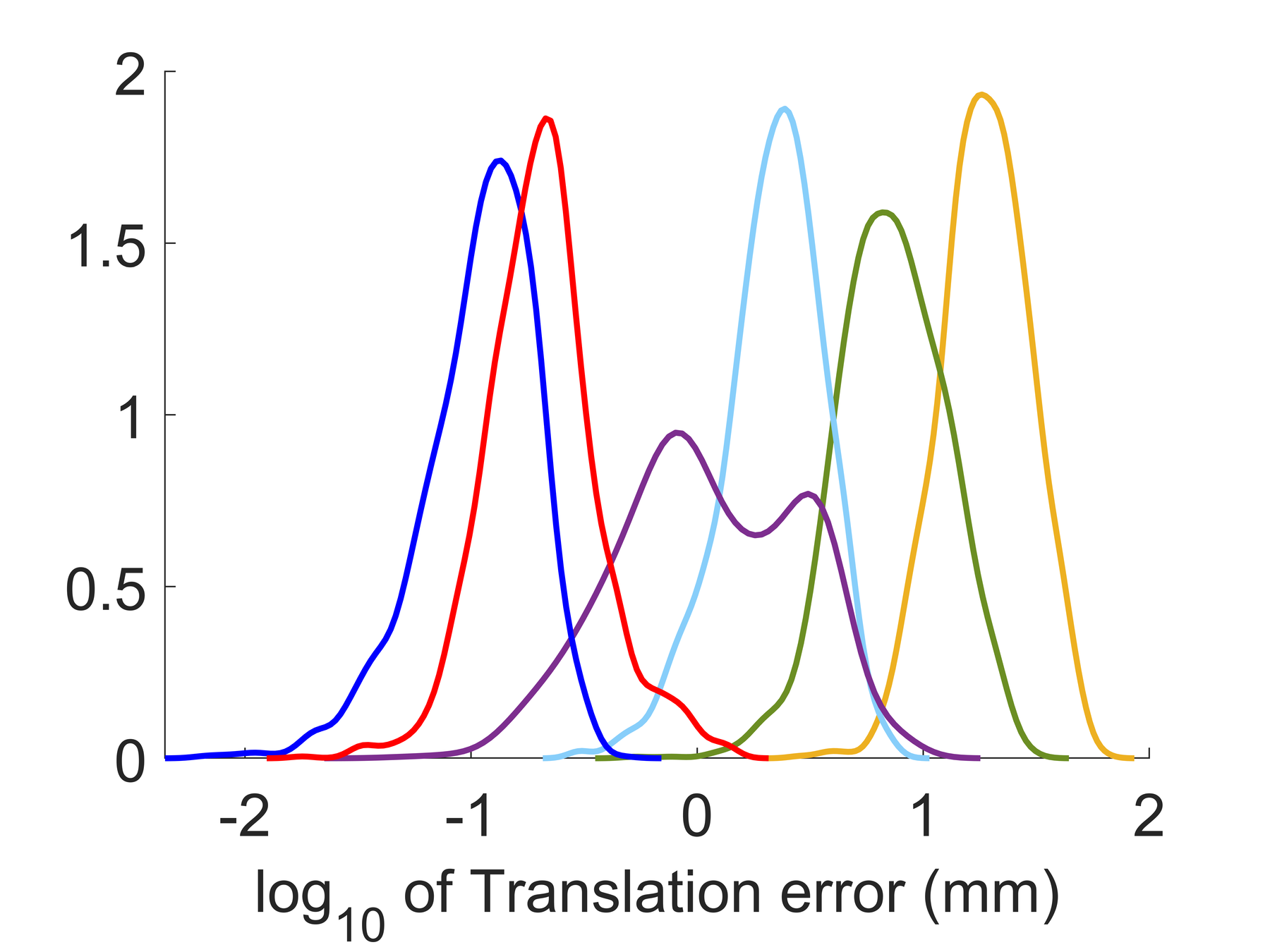}
\end{tabular}
\caption{Probability distribution of extrinsic parameter errors.}
\label{fig8:extrinsic_probability_distribution}
\end{figure*}

\begin{figure*}[!t]
\centering
\begin{tabular}{@{}c@{\hspace{0in}}c@{\hspace{0in}}c@{}}
\setlength{\tabcolsep}{2pt}
\small Linear lens & \small Wide-angle lens & \small Fisheye lens \\[-0.15em]
\includegraphics[width=0.4\textwidth, height=3cm]{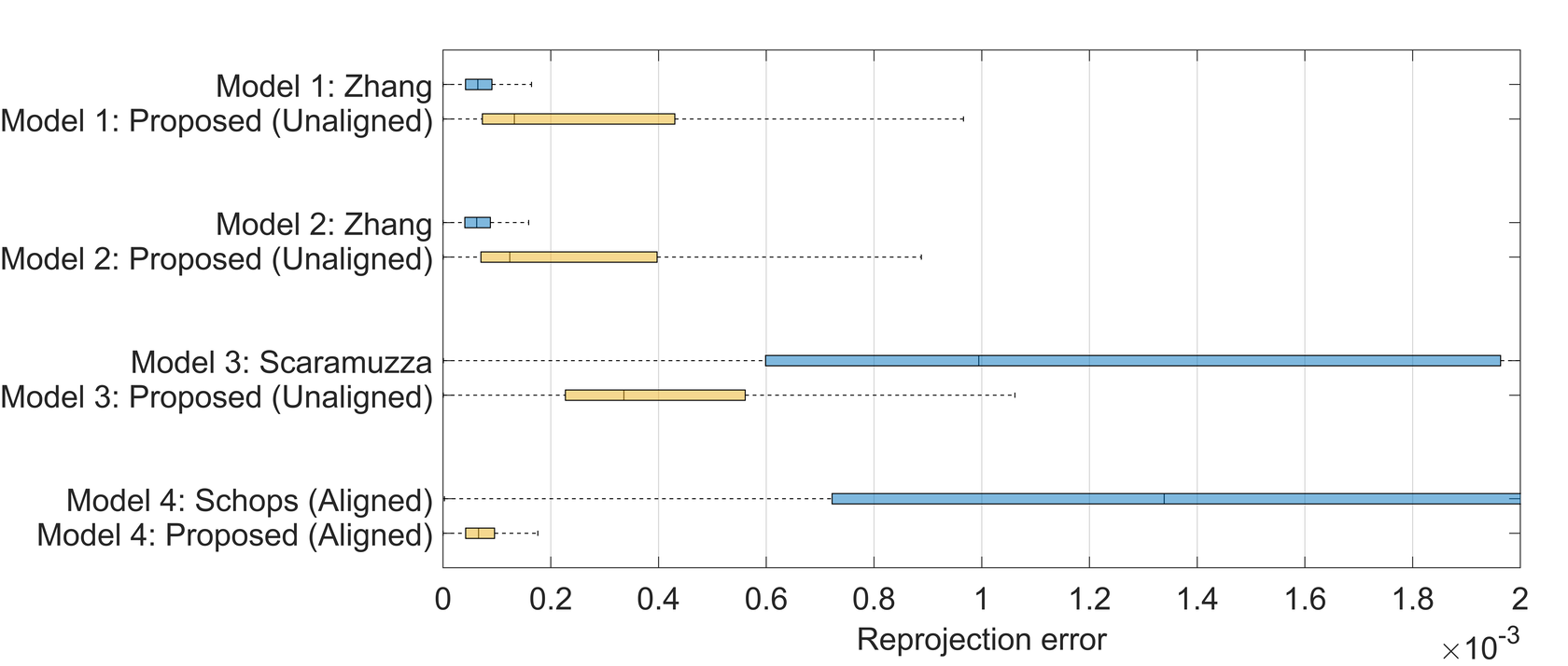} &
\includegraphics[width=0.3\textwidth, height=3cm]{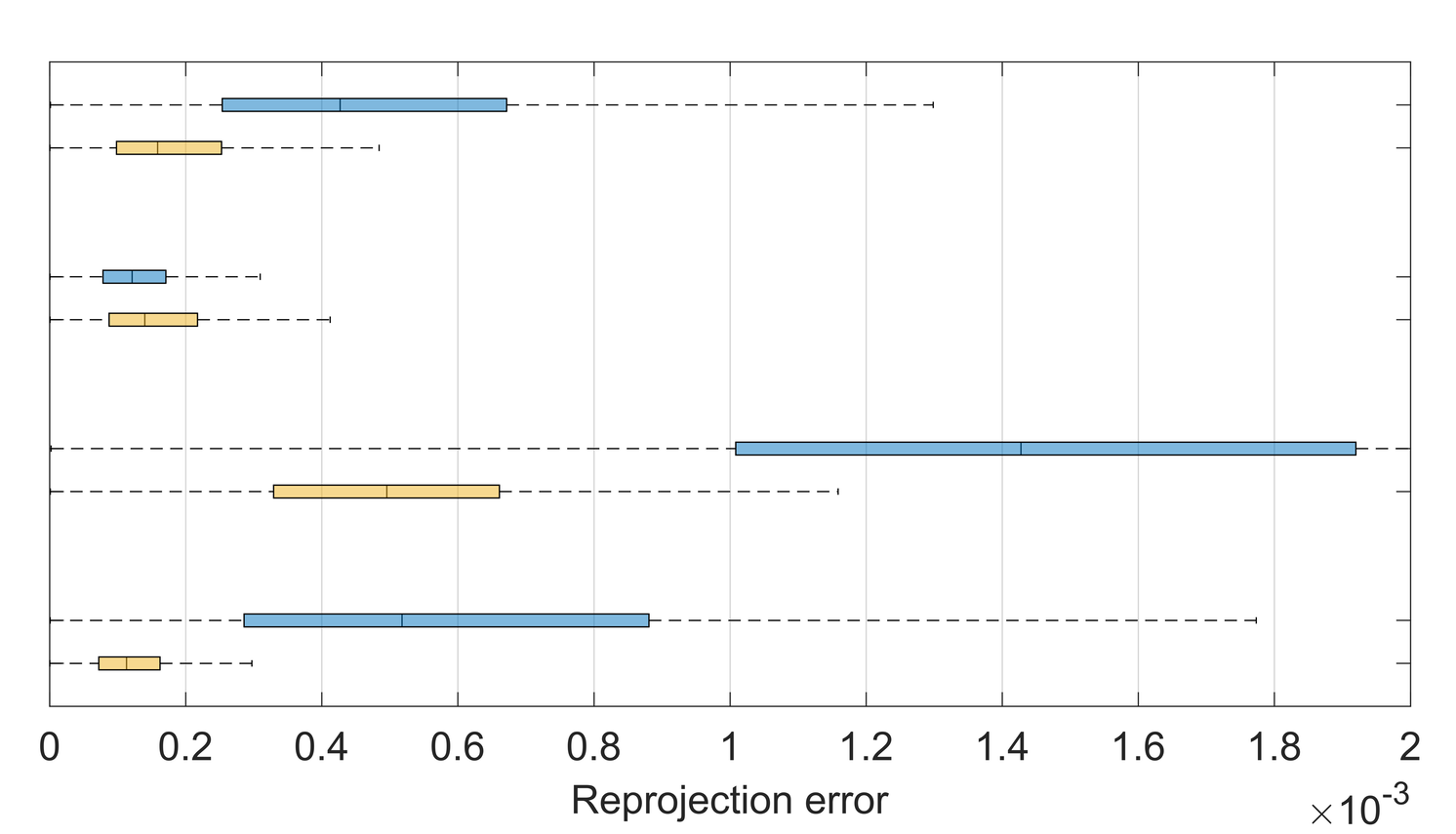} &
\includegraphics[width=0.3\textwidth, height=3cm]{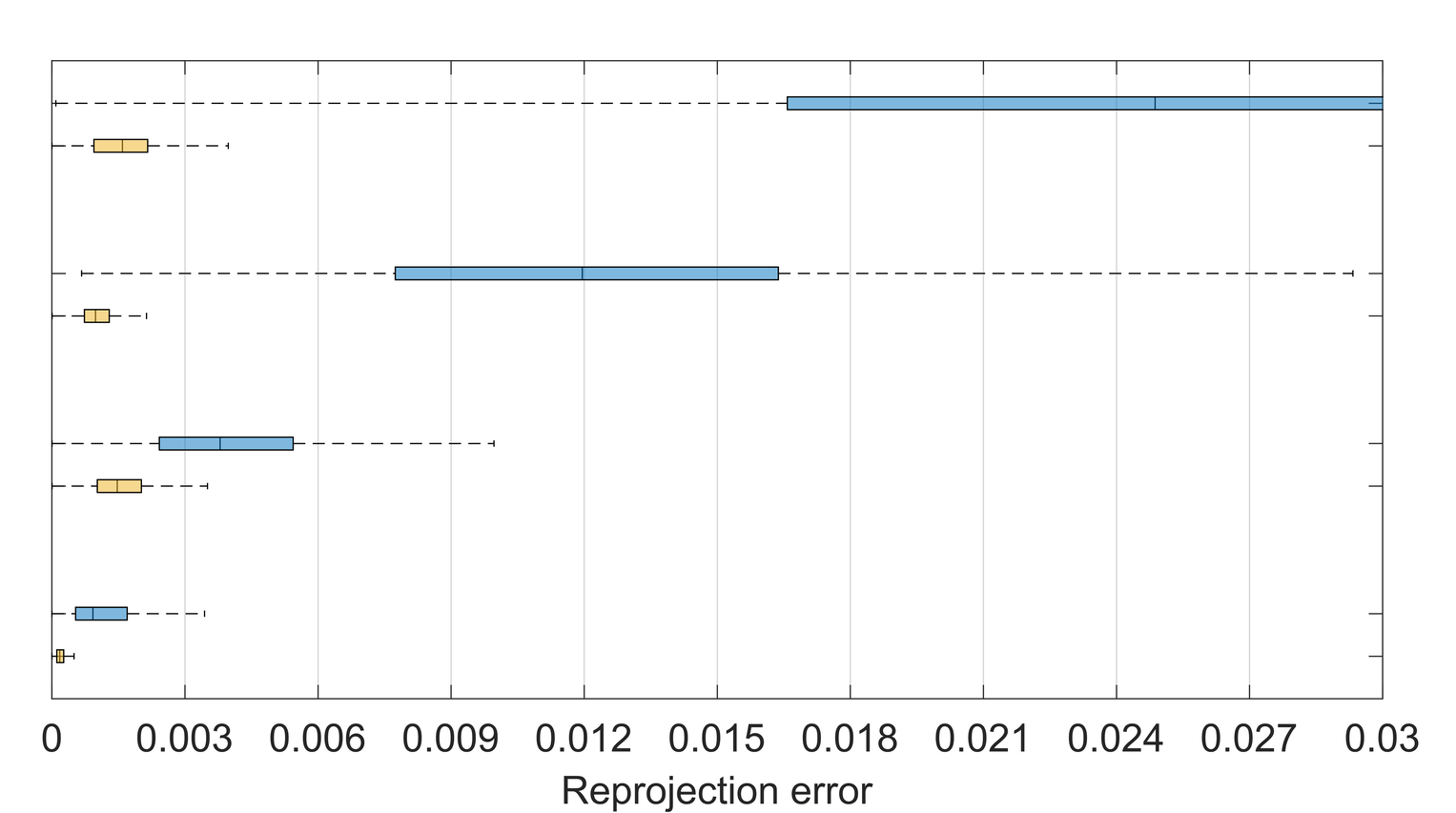} \\
\end{tabular}
\caption{Undistortion error distribution on the normalized image plane using calibrated intrinsic parameters.}
\label{fig9:unprojection_error}
\end{figure*}

\begin{table*}[!t]
\centering
\caption{Resampling-based uncertainty (standard deviation)}
\label{tab3:uncertainty}
\small
\renewcommand{\arraystretch}{1.6}

\begin{tabular*}{\textwidth}{@{\extracolsep{\fill}}l*{8}{c}}
\toprule
\multirow{2}{*}[-0.4ex]{\textbf{Model 1}} & \multicolumn{4}{c}{\textbf{Proposed}} & \multicolumn{4}{c}{\textbf{Zhang}} \\
\cmidrule(lr){2-5}\cmidrule(l){6-9}
 & $f_x$ & $f_y$ & $c_x$ & $c_y$ & $f_x$ & $f_y$ & $c_x$ & $c_y$\\
\midrule
SonyA6400 & $\bm{3.32}$ & $\bm{2.81}$ & 8.03 & 14.24 & 5.14 & 5.42 & $\bm{3.19}$ & $\bm{3.08}$\\
GoPro (Linear) & $\bm{0.84}$ & $\bm{0.93}$ & 0.73 & 1.94 & 1.35 & 1.38 & $\bm{0.72}$ & $\bm{0.79}$\\
GoPro (Wide) & $\bm{4.65}$ & $\bm{4.75}$ & $\bm{3.92}$ & $\bm{1.98}$ & 20.31 & 22.78 & 42.11 & 23.30\\
Insta360 & $\bm{10.32}$ & $\bm{9.69}$ & $\bm{0.56}$ & $\bm{0.74}$ & 26.48 & 27.61 & 33.12 & 29.11\\
\midrule[\heavyrulewidth]

\multirow{2}{*}[-0.4ex]{\textbf{Model 2}} & \multicolumn{4}{c}{\textbf{Proposed}} & \multicolumn{4}{c}{\textbf{Zhang}} \\
\cmidrule(lr){2-5}\cmidrule(l){6-9}
 & $f_x$ & $f_y$ & $c_x$ & $c_y$ & $f_x$ & $f_y$ & $c_x$ & $c_y$\\
\midrule
SonyA6400 & $\bm{3.43}$ & $\bm{2.91}$ & 8.03 & 14.24 & 4.77 & 5.13 & $\bm{3.12}$ & $\bm{3.00}$\\
GoPro (Linear) & $\bm{0.91}$ & $\bm{0.96}$ & 0.75 & 1.94 & 1.33 & 1.38 & $\bm{0.71}$ & $\bm{0.79}$\\
GoPro (Wide) & $\bm{3.82}$ & $\bm{3.94}$ & $\bm{3.92}$ & $\bm{1.98}$ & 8.34 & 9.13 & 12.92 & 5.32\\
Insta360 & $\bm{4.68}$ & $\bm{4.65}$ & $\bm{0.56}$ & $\bm{0.74}$ & 23.6 & 24.51 & 49.19 & 20.02\\
\midrule[\heavyrulewidth]

\multirow{2}{*}[-0.4ex]{\textbf{Model 3}} & \multicolumn{4}{c}{\textbf{Proposed}} & \multicolumn{4}{c}{\textbf{Scaramuzza}} \\
\cmidrule(lr){2-5}\cmidrule(l){6-9}
 & $a_0$ & $a_2$ & $c_x$ & $c_y$ & $a_0$ & $a_2$ & $c_x$ & $c_y$\\
\midrule
SonyA6400 & $\bm{3.18}$ & $\bm{3.51 \cdot 10^{-6}}$ & $\bm{8.00}$ & $\bm{14.33}$ & $4.11 \cdot 10^{4}$ & $2.71 \cdot 10^{-5}$ & 147.18 & 165.57\\
GoPro (Linear) & $\bm{1.36}$ & $\bm{5.00 \cdot 10^{-6}}$ & $\bm{0.76}$ & $\bm{1.95}$ & $1.06 \cdot 10^{3}$ & $1.61 \cdot 10^{-4}$ & 19.32 & 21.17\\
GoPro (Wide) & 3.86 & $\bm{3.36 \cdot 10^{-6}}$ & 3.92 & $\bm{1.97}$ & $\bm{3.70}$ & $3.92 \cdot 10^{-6}$ & $\bm{3.50}$ & 2.04\\
Insta360 & 1.14 & $6.71 \cdot 10^{-6}$ & $\bm{0.57}$ & $\bm{0.75}$ & $\bm{0.76}$ & $\bm{1.47 \cdot 10^{-6}}$ & 0.64 & $\bm{0.76}$\\
\bottomrule
\end{tabular*}

\begin{flushleft}
\small
\noindent Best parameter results are highlighted in bold.
\end{flushleft}
\end{table*}

\begin{figure*}[!t]
\centering
\setlength{\tabcolsep}{0pt}
\footnotesize
\begin{tabular}{@{}l@{\hspace{0.4em}}l@{\hspace{0.4em}}l@{\hspace{0.4em}}l@{}}
  \raisebox{-0.2ex}{\includegraphics[height=1.8ex]{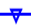}}\;\scriptsize M1: Zhang &
  \raisebox{-0.2ex}{\includegraphics[height=1.8ex]{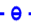}}\;\scriptsize M2: Zhang &
  \raisebox{-0.2ex}{\includegraphics[height=1.8ex]{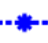}}\;\scriptsize M3: Scaramuzza &
  \raisebox{-0.2ex}{\includegraphics[height=1.8ex]{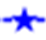}}\;\scriptsize M4: Sch\"{o}ps (Aligned) \\[0.4em]
  \raisebox{-0.2ex}{\includegraphics[height=1.8ex]{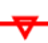}}\;\scriptsize M1: Proposed (Unaligned) &
  \raisebox{-0.2ex}{\includegraphics[height=1.6ex]{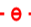}}\;\scriptsize M2: Proposed (Unaligned) &
  \raisebox{-0.2ex}{\includegraphics[height=1.8ex]{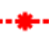}}\;\scriptsize M3: Proposed (Unaligned) &
  \raisebox{-0.2ex}{\includegraphics[height=1.8ex]{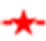}}\;\scriptsize M4: Proposed (Aligned) \vspace{1em}
\end{tabular}

\begin{tabular}{@{\hspace{-0.1in}}c@{\hspace{-0.08in}}c@{\hspace{-0.05in}}c@{}}
  \scriptsize Linear lens & \scriptsize Wide-angle lens & \scriptsize Fisheye lens\\[-0.2em]
  \includegraphics[width=0.3\textwidth, height=3.6cm]{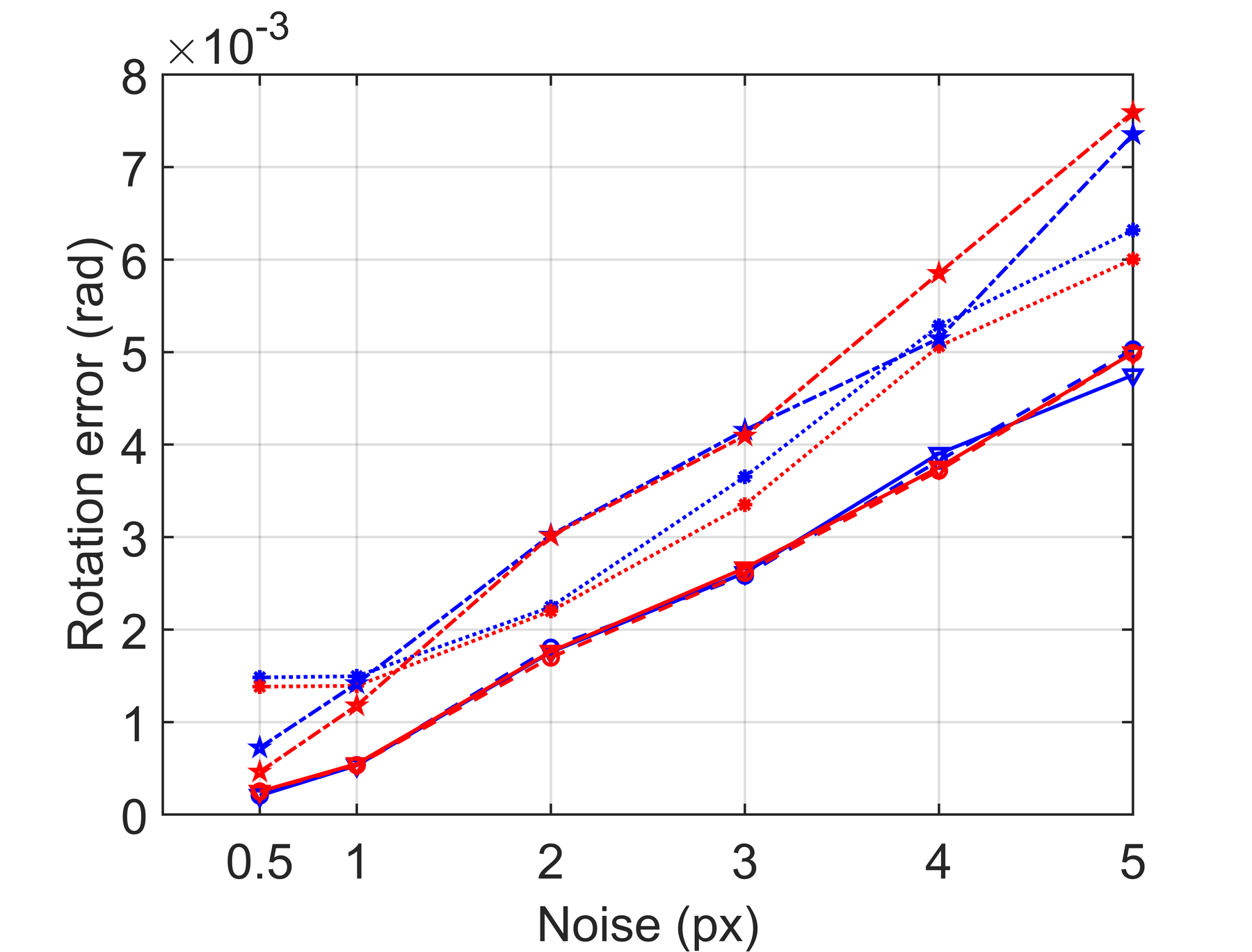} &
  \includegraphics[width=0.3\textwidth, height=3.6cm]{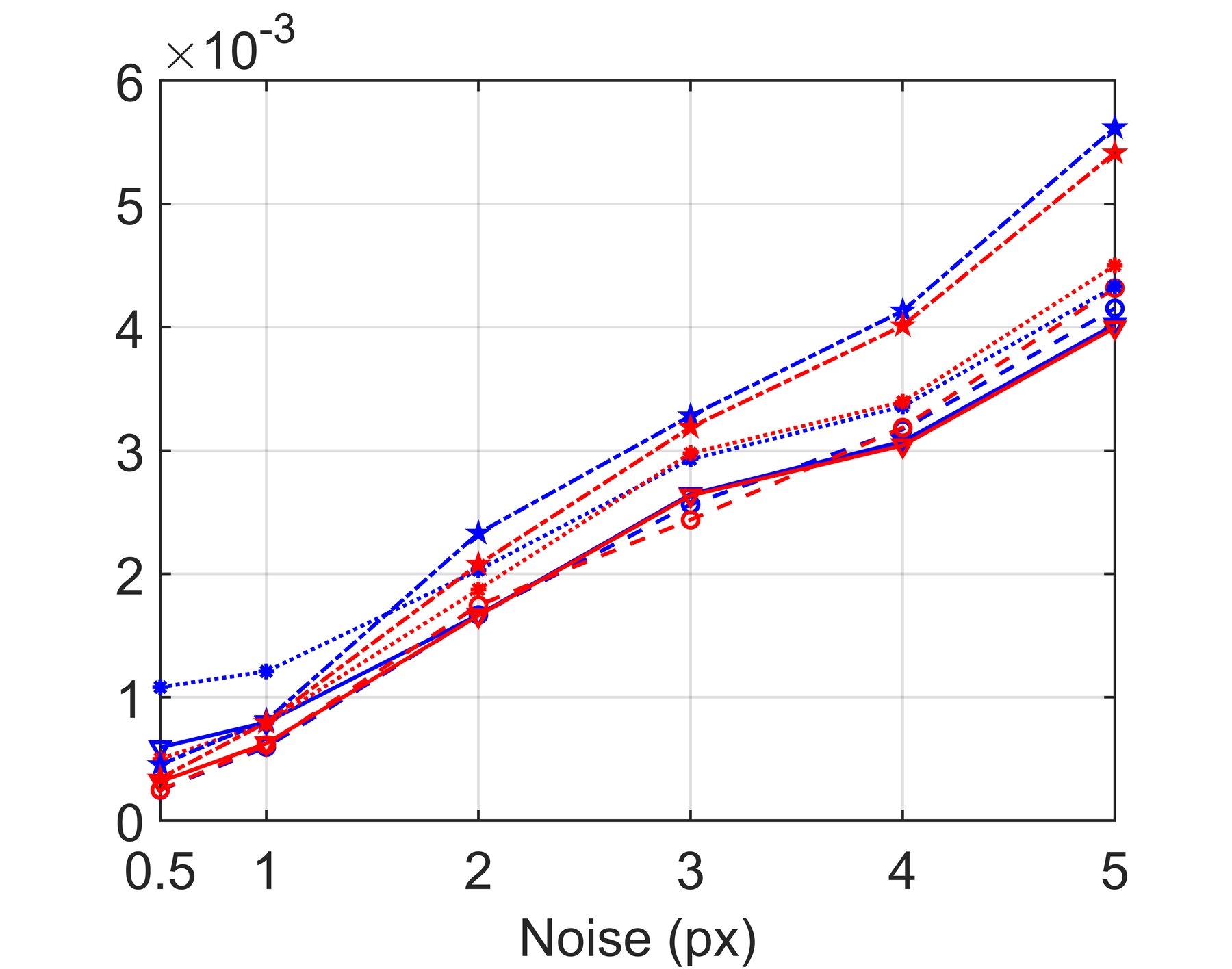} &
  \includegraphics[width=0.3\textwidth, height=3.6cm]{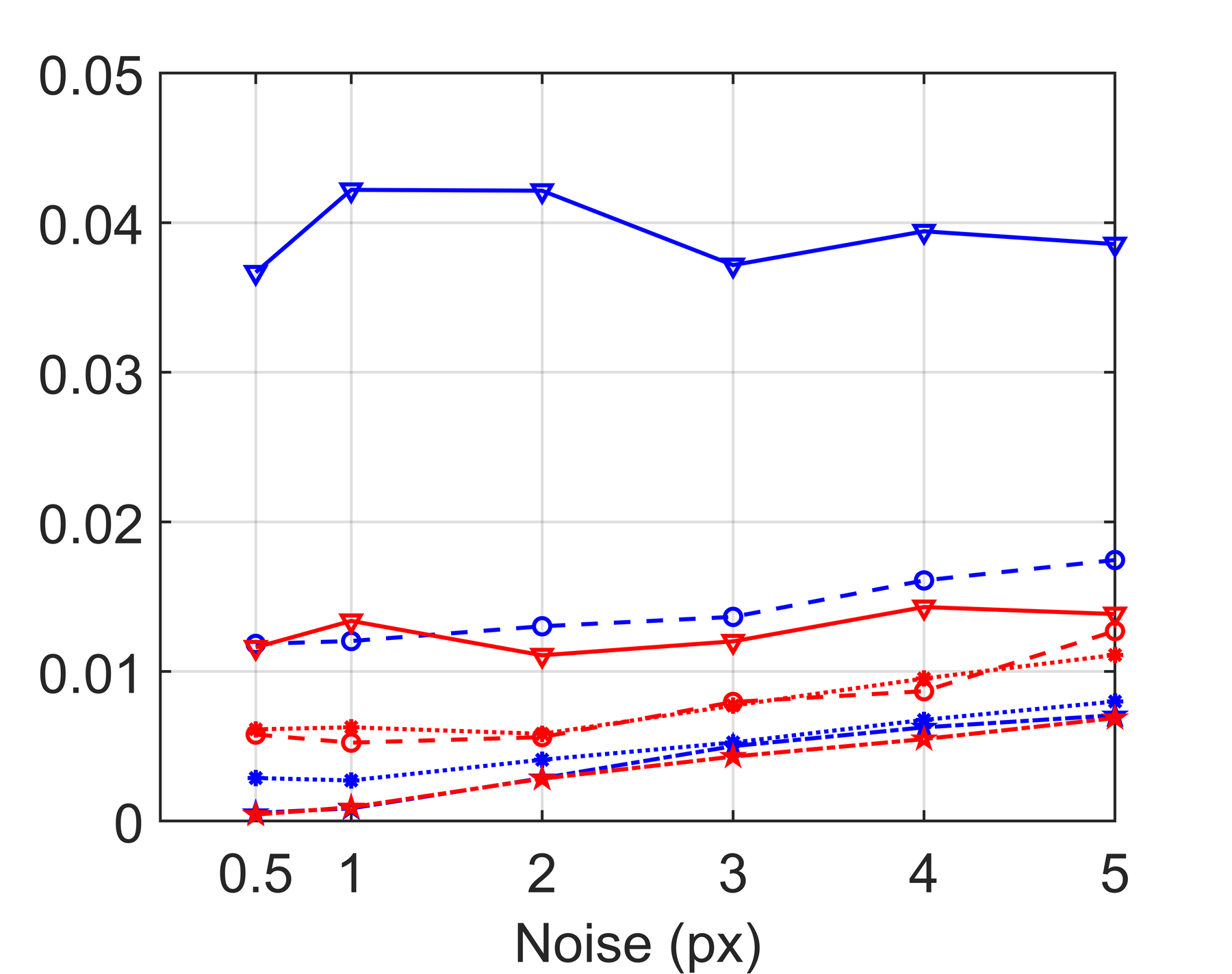} \\[0.15em]
  \includegraphics[width=0.3\textwidth, height=3.6cm]{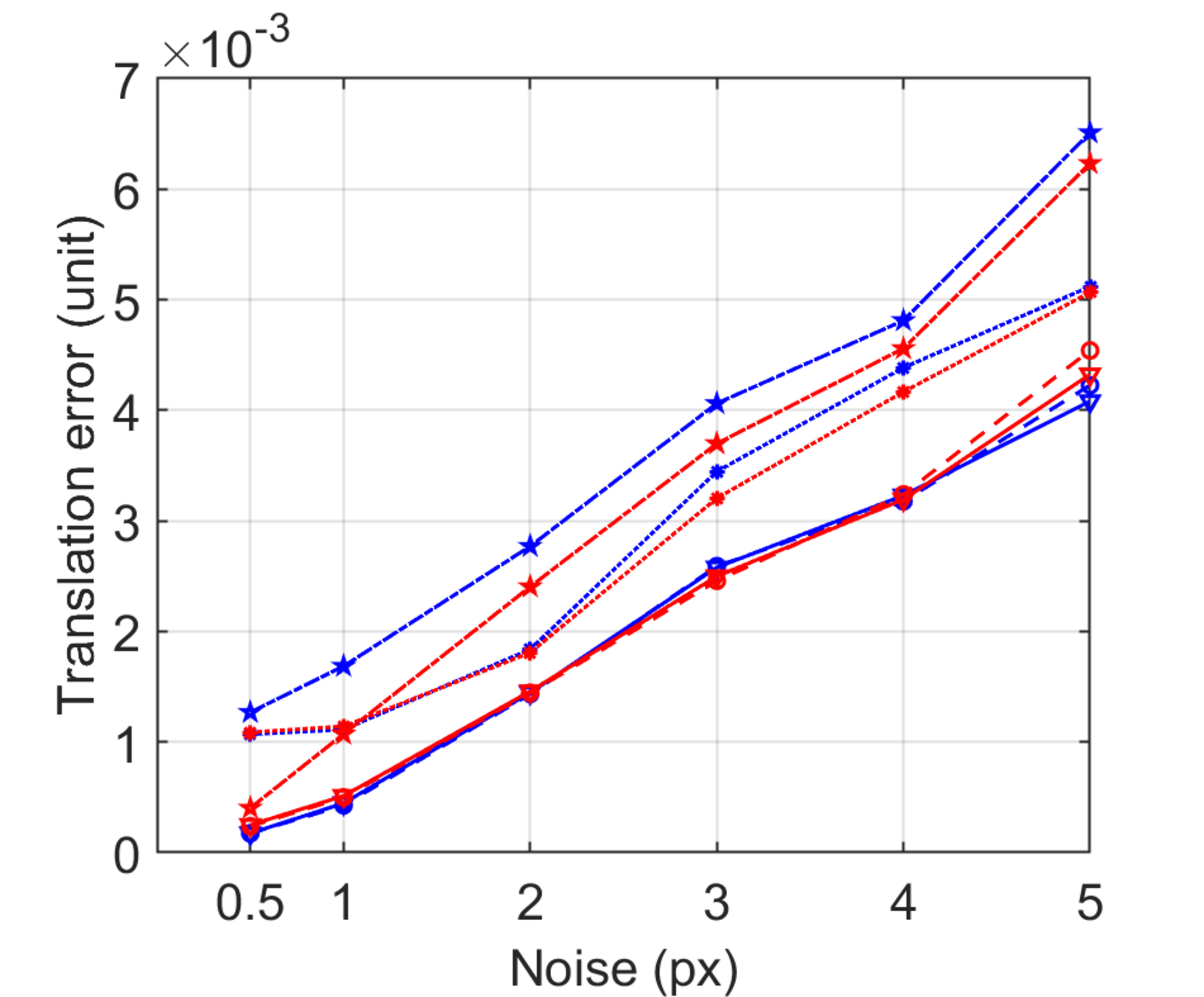} &
  \includegraphics[width=0.3\textwidth, height=3.6cm]{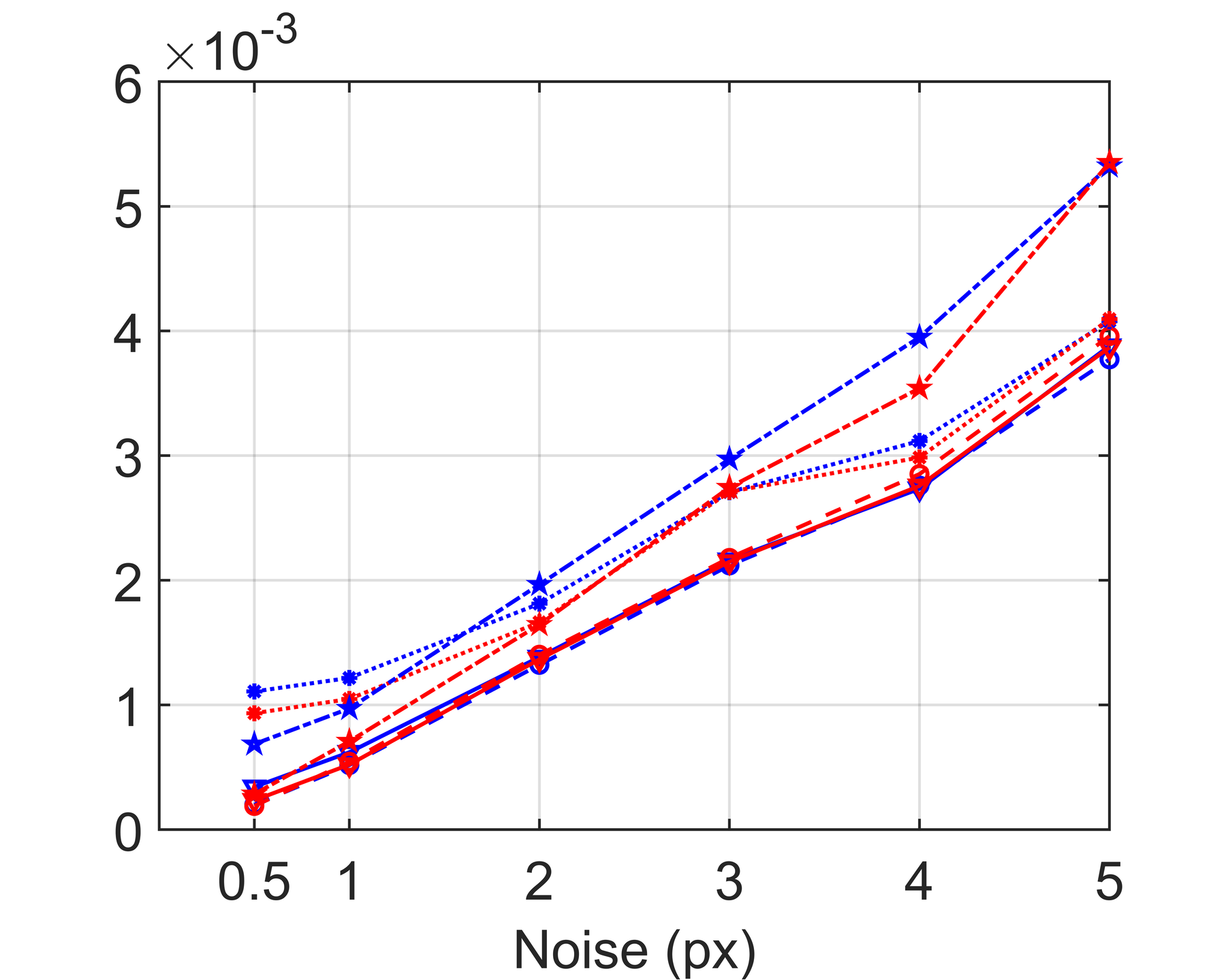} &
  \includegraphics[width=0.3\textwidth, height=3.6cm]{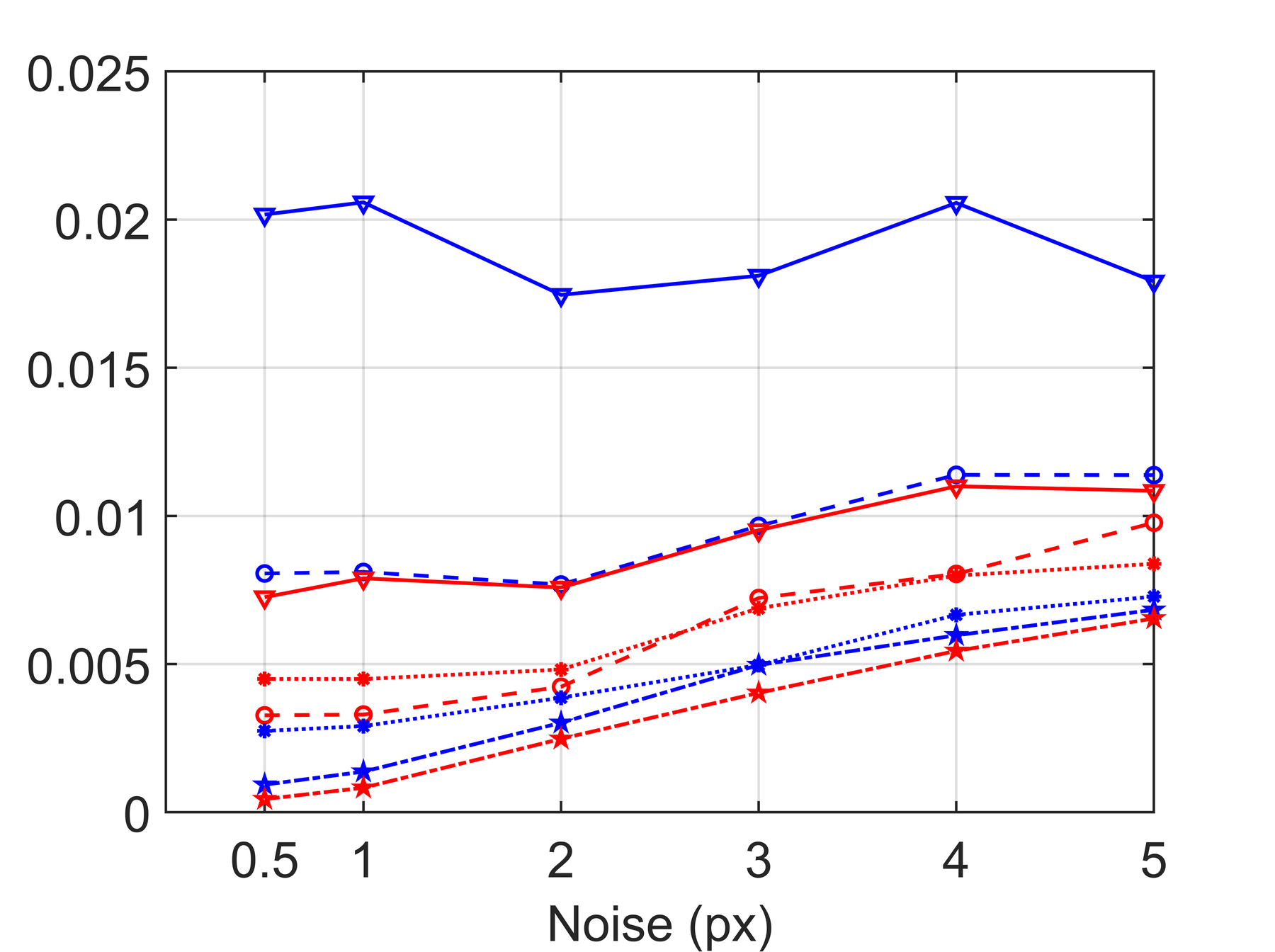}
\end{tabular}
\caption{Two-view pose estimation accuracy. Relative translation errors are defined as Euclidean distances between estimated and ground-truth coordinates after normalizing the position vectors to unit length.}
\label{fig10:accuracy_evaluation}
\end{figure*}


\begin{figure*}[!t]
\centering
\scriptsize
\setlength{\tabcolsep}{0pt}
\begin{tabular}{c}
\begin{tabular}{@{}l@{\hspace{0in}}l@{\hspace{0in}}l@{\hspace{0in}}l@{}}
  \raisebox{-0.2ex}{\includegraphics[height=1.8ex]{fig8_a}}\;\scriptsize M1: Zhang &
  \raisebox{-0.2ex}{\includegraphics[height=1.8ex]{fig8_c}}\;\scriptsize M2: Zhang &
  \raisebox{-0.2ex}{\includegraphics[height=1.8ex]{fig8_e}}\;\scriptsize M3: Scaramuzza &
  \raisebox{-0.2ex}{\includegraphics[height=1.8ex]{fig7_a.png}}\;\scriptsize M4: Sch\"{o}ps (Unaligned) \\[0em]
  \raisebox{-0.2ex}{\includegraphics[height=1.6ex]{fig8_b}}\;\scriptsize M1: Proposed (Unaligned) &
  \raisebox{-0.2ex}{\includegraphics[height=1.8ex]{fig8_d}}\;\scriptsize M2: Proposed (Unaligned) &
  \raisebox{-0.2ex}{\includegraphics[height=1.8ex]{fig8_f}}\;\scriptsize M3: Proposed (Unaligned) &
  \raisebox{-0.2ex}{\includegraphics[height=1.8ex]{fig7_b.png}}\;\scriptsize M4: Proposed (Unaligned) \vspace{1em}
\end{tabular} \\
\vspace{0.3em}

\begin{tabular}{@{\hspace{-0.3in}}c@{\hspace{-0.15in}}c@{\hspace{-0.15in}}c@{\hspace{-0.15in}}c@{}} 
  \scriptsize Sony A6400 & \scriptsize GoPro (Linear) & \scriptsize GoPro (Wide) & \scriptsize Insta360 \\[-0.15em]
  \hspace*{-0.1in} \includegraphics[width=0.26\textwidth]{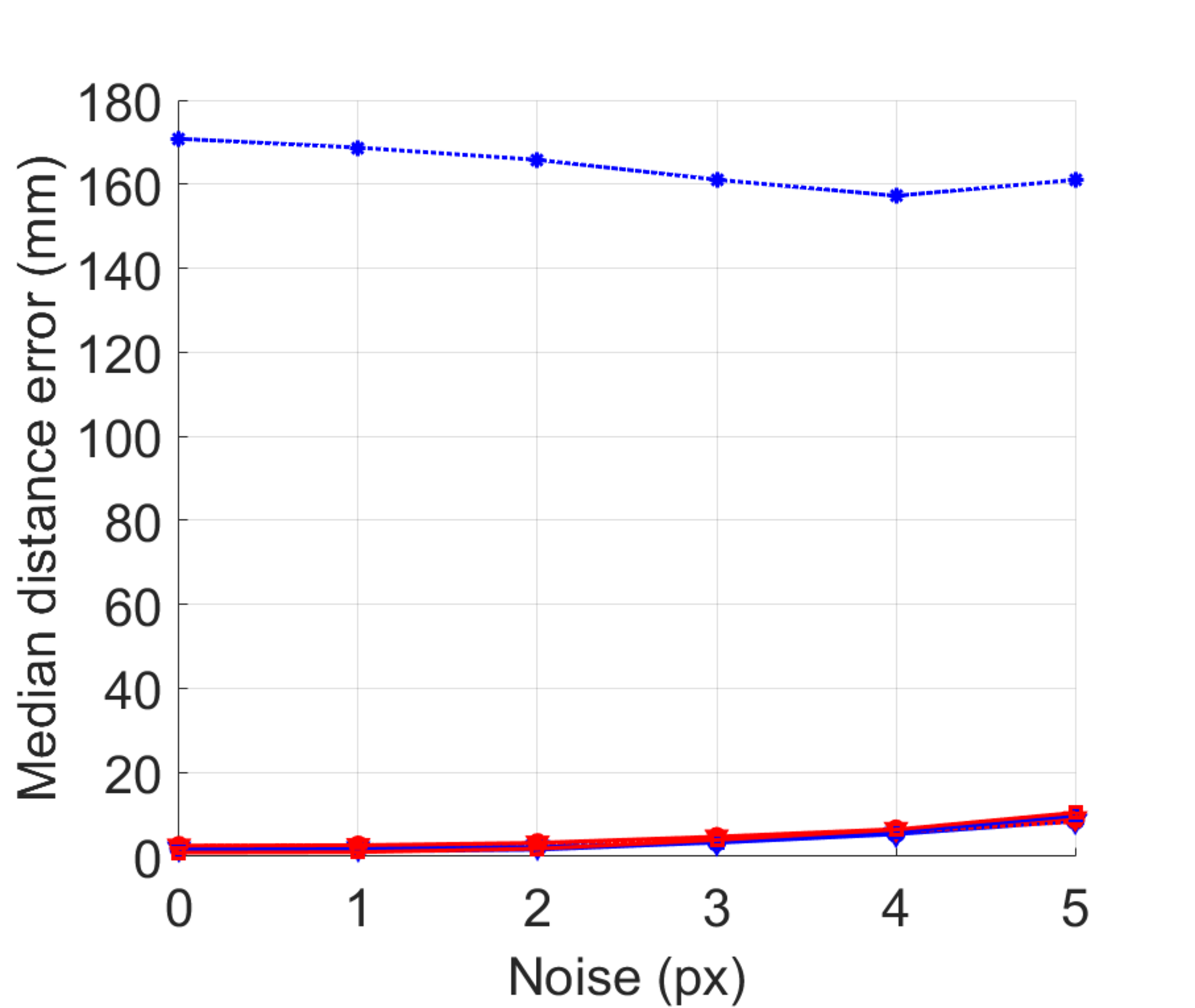} &
  \includegraphics[width=0.26\textwidth]{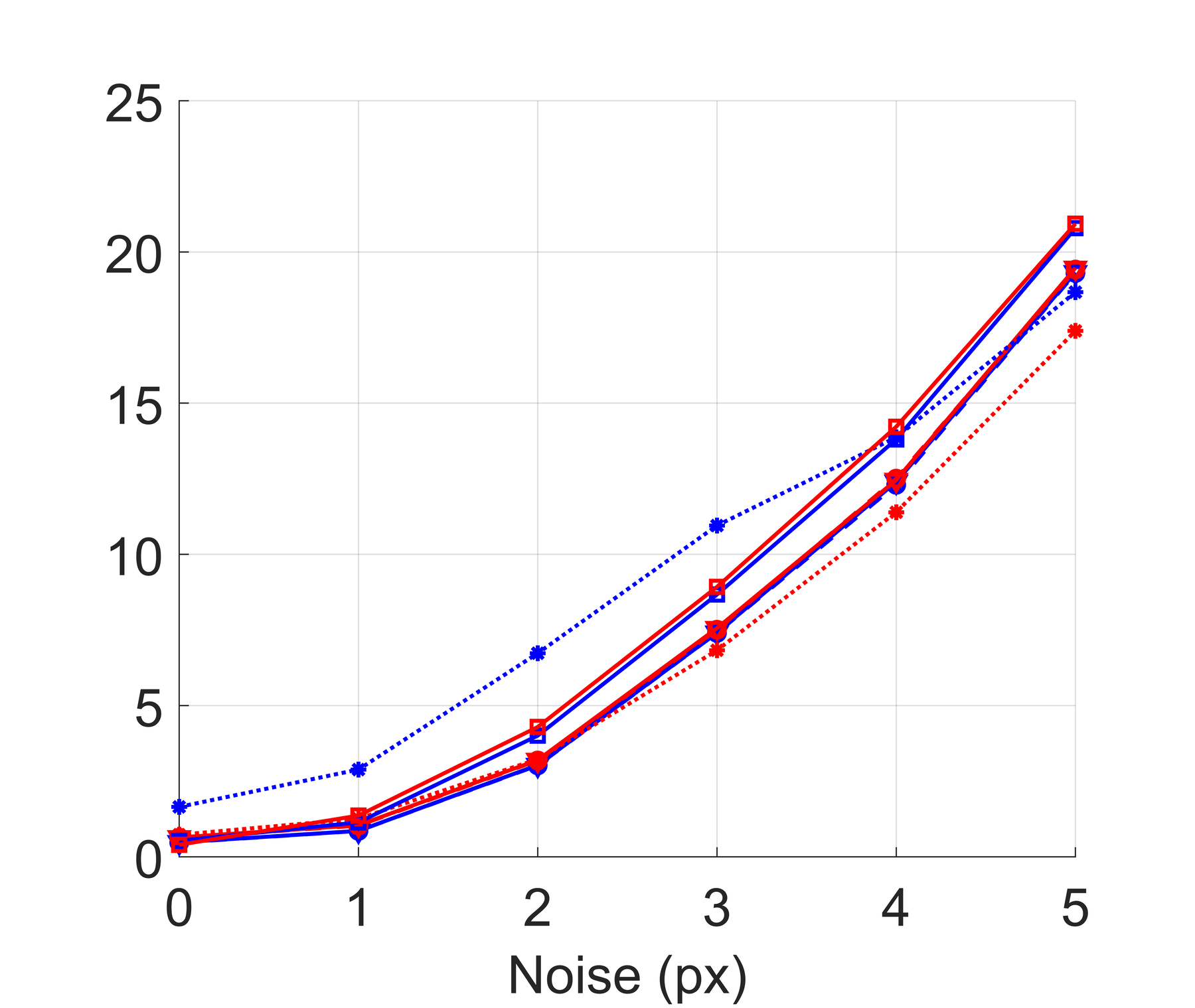} &
  \includegraphics[width=0.26\textwidth]{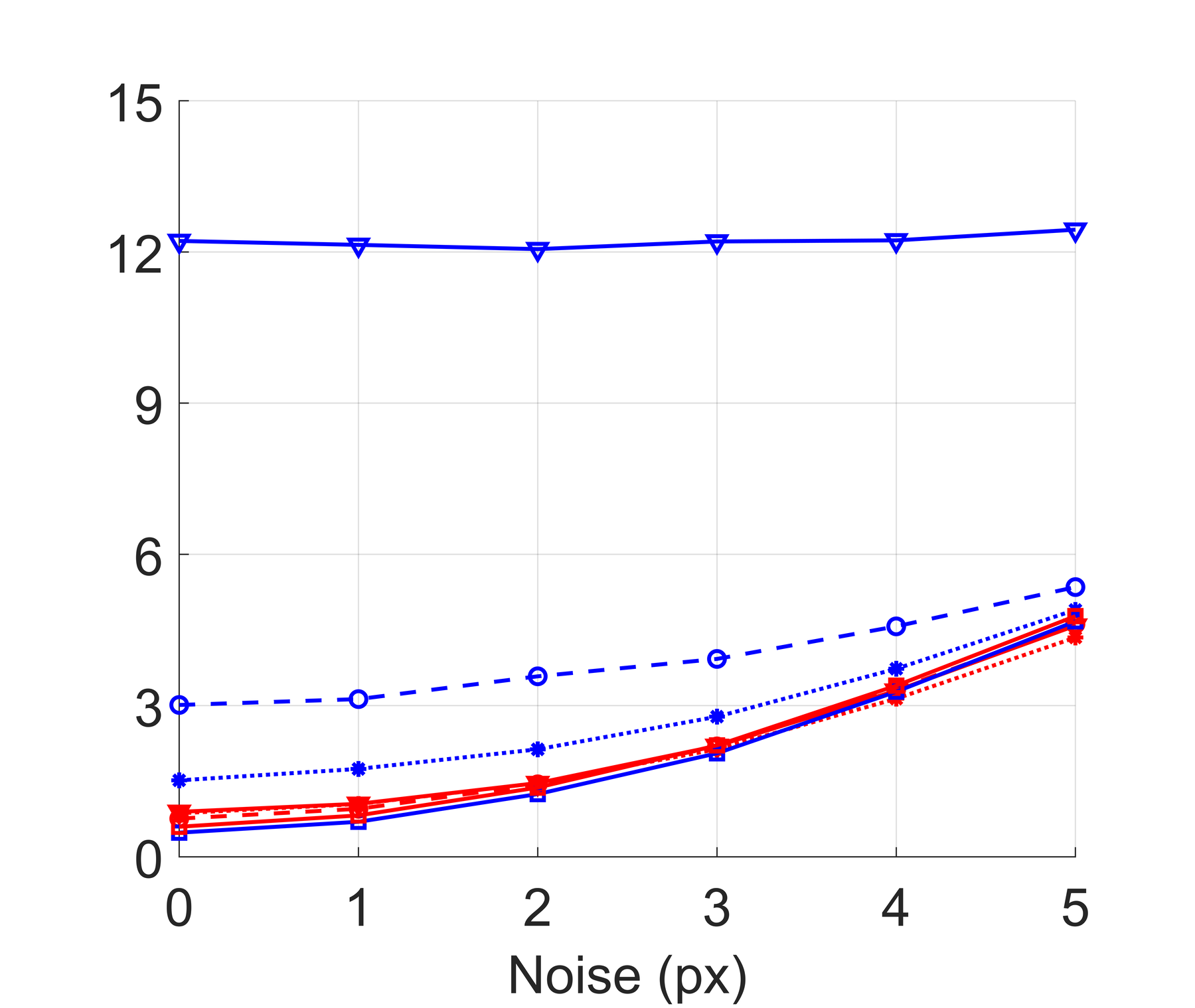} &
  \includegraphics[width=0.26\textwidth]{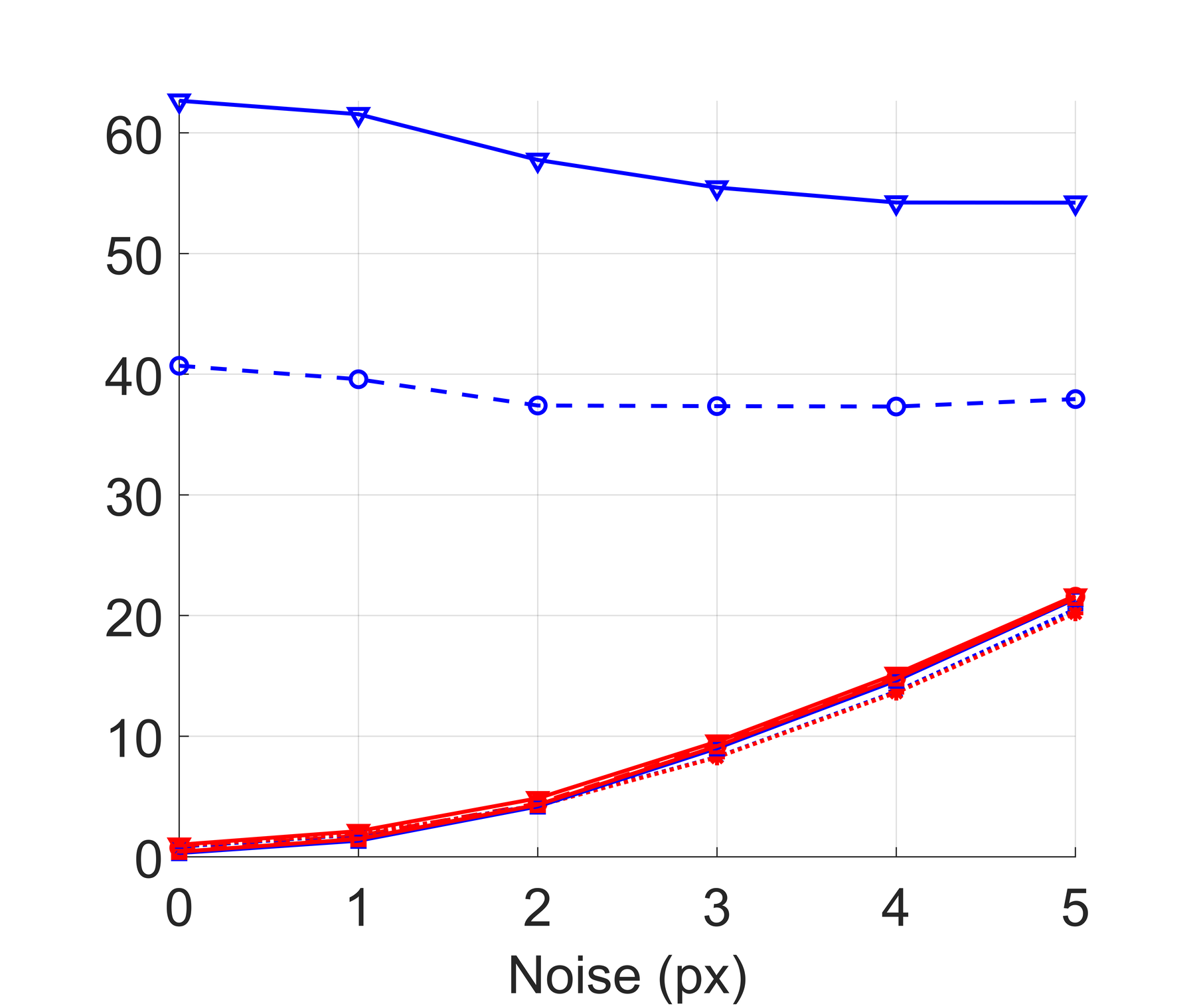}
\end{tabular}
\end{tabular}
\caption{Two-view 3D reconstruction accuracy under varying lenses, calibration approaches and noise levels.}
\label{fig11:reconstruction_accuracy}
\end{figure*}

\begin{figure*}[!t]
\centering
\setlength{\tabcolsep}{0.5pt}  %
\scriptsize
\begin{tabular}{c@{\hspace{4mm}}c@{\hspace{2mm}}c@{\hspace{2mm}}c}
\begin{tabular}[t]{c}
\textbf{Sample images}\\[0.1cm]
\includegraphics[width=3cm]{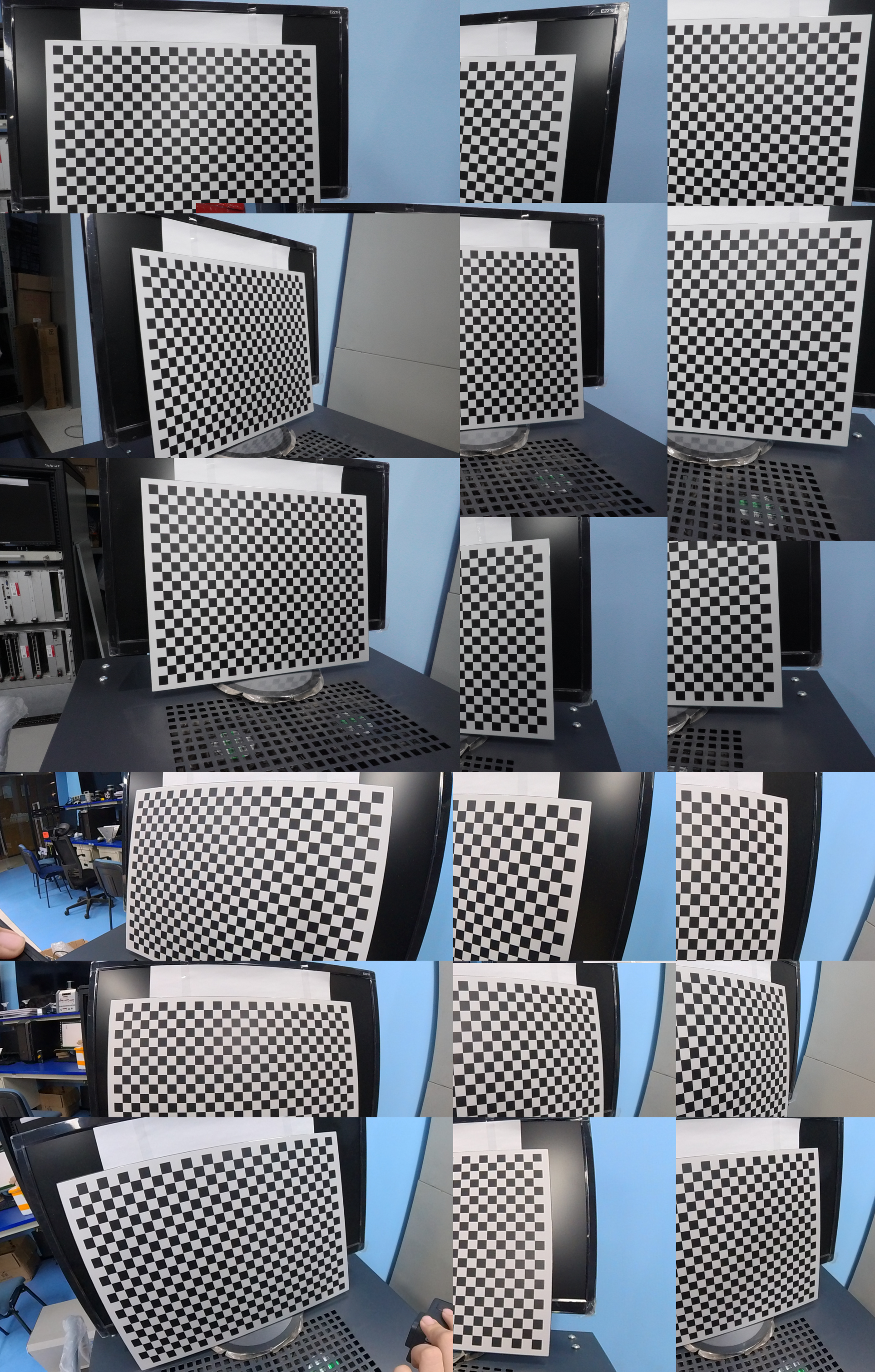}\\ 
\end{tabular}
&
\begin{tabular}[t]{@{}l@{}}
\multicolumn{1}{c}{\hspace{-0.3in} \textbf{Insta360}}\\[0.1cm]
{\color{blue}$\bullet$} M1: Zhang\\
{\color{red}$\bullet$} M1: Proposed (Unaligned)\\[0.15cm]
\includegraphics[width=2.8cm]{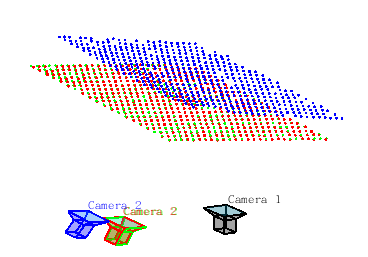}\\[0.3cm]
\includegraphics[width=2.8cm]{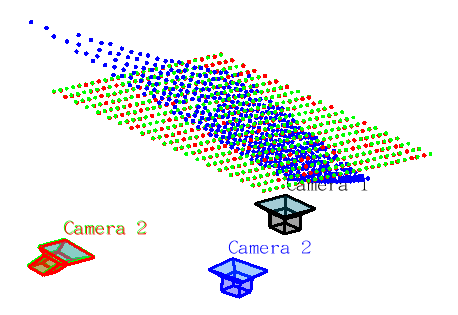}
\end{tabular}
&
\begin{tabular}[t]{@{}l@{}}
\multicolumn{1}{c}{\hspace{-0.3in} \textbf{Insta360}}\\[0.1cm]
{\color{blue}$\bullet$} M2: Zhang\\
{\color{red}$\bullet$} M2: Proposed (Unaligned)\\[0.15cm]
\includegraphics[width=2.8cm]{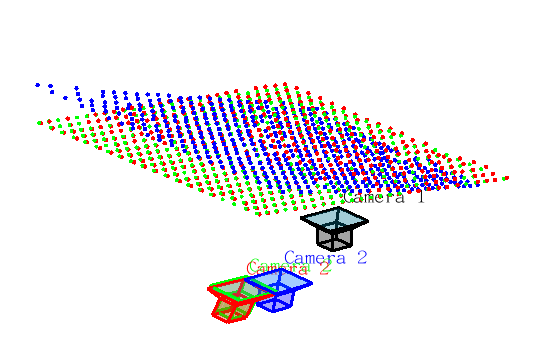}\\[0.3cm]
\includegraphics[width=2.6cm]{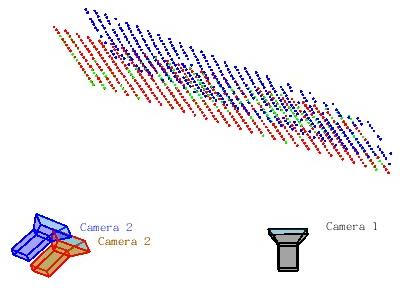}
\end{tabular}
&
\begin{tabular}[t]{@{}l@{}}
\multicolumn{1}{c}{\hspace{-0.3in} \textbf{Sony A6400}}\\[0.1cm]
{\color{blue}$\bullet$} M3: Scaramuzza\\
{\color{red}$\bullet$} M3: Proposed (Unaligned)\\[0.15cm]
\includegraphics[width=2.8cm]{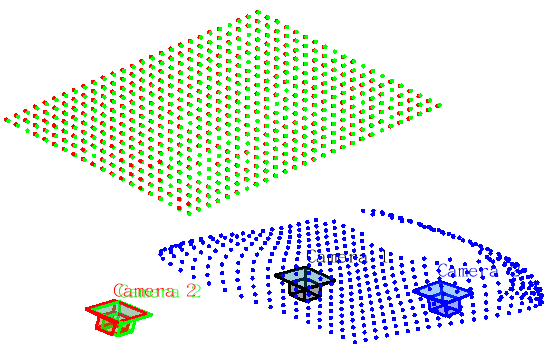}\\[0.45cm]
\includegraphics[width=3.0cm]{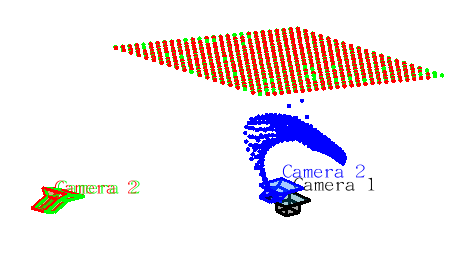}
\end{tabular}
\end{tabular}

\vspace{0.2cm}
\scriptsize
\begin{tabular}{@{}ll@{}}
{\color{black}$\bullet$} Left-view camera pose \;\; & {\color{green}$\bullet$} Reference right-view camera pose and 3D points
\end{tabular}
\caption{Examples of 3D reconstruction results for Insta360 cameras (Models 1 and 2) and Sony A6400 (Model 3). The proposed method shows improved reconstruction quality compared to traditional calibration methods.}
\label{fig12:reconstruction_results}
\end{figure*}

\textbf{Extrinsic Parameter Accuracy.} Fig. \ref{fig8:extrinsic_probability_distribution} compares the extrinsic parameter accuracy of our method against four open-source calibration toolboxes across different lens types. Our approach consistently outperforms the Sch\"{o}ps (M4) and Scaramuzza (M3). While slightly inferior to Zhang (M1) and Zhang (M2) in linear lens scenarios, the peak difference in rotational errors is below $10^{-3}$ rad. As lens distortion increases, Zhang (M1) and Zhang (M2) exhibit a significant degradation in extrinsic accuracy. For instance, in fisheye lens scenarios, their rotation error distribution peaks exceed $10^{-2}$ rad, and the translational error distribution peak for Zhang (M1) even exceeds 10 mm. In contrast, Proposed (M1) and Proposed (M2) maintain rotation error distribution peak below $10^{-3}$ rad and translational error peak below 1 mm. These observations show that the proposed approach offers superior stability and robustness when handling various distortion lenses.

\textbf{Intrinsic Parameter Accuracy.} Fig. \ref{fig9:unprojection_error} compares the intrinsic parameter accuracy of our approach against the baseline toolboxes for Models 1-4 across different lenses. Proposed (M4) consistently ranks the best one or among the best ones across different lens scenarios. Furthermore, the proposed approach exhibits robust performance across various lens distortions. While Proposed (M1/M2) slightly underperforms Zhang (M1/M2) in low-distortion scenarios, it significantly surpasses Zhang (M1/M2) as distortion increases. Moreover, our method effectively reduces systematic errors arising from model-lens mismatches. For instance, in the linear and wide-angle lens scenarios, Proposed (M3) significantly outperforms Scaramuzza (M3).

\textbf{Two-view Pose Estimation Accuracy.} Fig. \ref{fig10:accuracy_evaluation} evaluates the pose estimation accuracy using the calibrated models for two-view geometry. In low-distortion scenarios, including linear and wide-angle lenses, Proposed (M1/M2) achieves performance comparable to Zhang (M1/M2). However, the advantages of our approach become prominent as lens distortion increases. For fisheye lenses, our method significantly outperforms Zhang (M1/M2). This indicates that our approach exhibits higher robustness and reliability when handling various distortion lens scenarios. Additionally, Fig. \ref{fig10:accuracy_evaluation} further reveals the notable advantage of the B-spline surface model (M4) for high-distortion lenses, which is consistent with the conclusion reported in~\cite{Dunne2007comparison}.

\subsection{Real Test}
\textbf{Resampling-based Uncertainty.} We employ a bootstrap resampling strategy~\cite{Davison1997} to quantify the stability of the estimated parameters for the parametric models. For each dataset, we perform 50 calibration trials on random subsets of 30 images and report the standard deviation of the estimated intrinsic parameters in Table \ref{tab3:uncertainty}. The results reveal that our proposed pipeline yields consistently more stable parameter estimates, particularly in challenging scenarios. For instance, Proposed (M1/M2) exhibits significantly lower standard deviations than Zhang (M1/M2) on GoPro (Wide) and Insta360 cameras. While Scaramuzza (M3) tends to diverge on Sony A6400, our Proposed (M3) remains robust. Moreover, the proposed method achieves comparable stability with Scaramuzza (M3) on the Insta360. These observations demonstrate that our method effectively mitigates instabilities arising from model-lens mismatches. Notably, Table \ref{tab3:uncertainty} reports only a subset of Model 3's parameters: i.e., $a_0$, $a_2$, $c_x$ and $c_y$ (refer to \cite{Scaramuzza2006}), as the standard deviations for others ($a_3, a_4$) are negligible (below $10^{-7}$).

\textbf{Reconstruction Accuracy.} To evaluate the practical applicability of the calibration results, we conduct a two-view reconstruction experiment. The evaluation datasets are created by imaging a stationary calibration pattern from multiple camera poses, using the same cameras as in the calibration process. For each camera model (M1–M3), we evaluate 1,000 view pairs to mitigate the influence of outlier data. Additionally, Gaussian noise with a standard deviation ranging from 0 to 5 pixels is artificially added to the original images to evaluate the reconstruction robustness under varying noise contamination. Since no absolute ground truth is available in real-world settings, we establish a reliable reference by using the calibration toolboxes that prove most accurate in simulation: Zhang (M2) for GoPro (Linear) and Sony A6400; Scaramuzza (M3) for GoPro (Wide) and Insta360 cameras. Accuracy is then measured by the 3D Euclidean distance between reconstructed points (with registration through point-cloud registration) and their known reference positions on the pattern.

The quantitative results in Fig. \ref{fig11:reconstruction_accuracy} demonstrate the superior stability of our proposed method. While traditional parametric calibrations suffer from significant performance degradation for mismatched lenses, our approach maintains high accuracy across all scenarios. For instance, when applying Model 3 to the Sony A6400 camera, Scaramuzza (M3) results in substantially higher reconstruction errors, whereas Proposed (M3) achieves accuracy comparable to the reference values. Moreover, for GoPro (Wide) and Insta360 cameras, Proposed (M1/M2) significantly reduces the reconstruction errors by approximately an order of magnitude compared to Zhang (M1/M2). 

These quantitative findings are visually corroborated by the representative reconstructions in Fig. \ref{fig12:reconstruction_results}. For the Sony A6400 camera with Model 3, reconstruction from Scaramuzza (M3) is visibly warped, whereas Proposed (M3) correctly preserves the board's planar integrity. The examples for the Insta360 camera further highlight that the reconstructions from Proposed (M1/M2) closely match the reference and outperform Zhang (M1/M2). These results validate the reliability of our hybrid calibration framework across diverse camera models and complex lens scenarios.

\section{Conclusion}\label{sec7}
This paper identifies a pose ambiguity in pose solutions by the generic calibration and analyzes its adverse influence on pose estimation accuracy. An efficient linear solver and a subsequent nonlinear optimization are proposed to handle the ambiguity issue. Additionally, to simplify the complex calibration workflow of generic calibration approach and enhance the stability and reliability of parametric calibration, we develop a novel global optimization hybrid calibration framework. The framework first determines and fixes the extrinsic parameters without  relying on any specific parametric model before optimizing the intrinsics.

The proposed approach is comprehensively evaluated against the existing generic and parametric calibration methods, across diverse lens types in both simulated and real-world tests. Results demonstrate that generic calibrations are negatively impacted by the inherent pose ambiguity problem, while parametric calibrations exhibit high dependency on specific lens types. The proposed approach is capable of efficiently and accurately estimating extrinsic parameters, resolving pose ambiguity, and mitigating risks related to overfitting and numerical instability during calibration. The approach has great potential for complex and higher-dimensional camera models with enhanced flexibility in model selection. Future work will explore its extension to non-central camera models.

\bmhead{Acknowledgements}

This work was supported in part by the National Key R\&D Program (2022YFB3903802), the National Natural Science Foundation of China (62273228), the Shanghai Jiao Tong University Scientific and Technological Innovation Funds, and the China Postdoctoral Science Foundation and Fellowship Program of CPSF (2024278).

\section*{Declarations}

\noindent \textbf{Data Availability Statement}: The data that support the findings of this study are openly available in Zenodo at \url{https://doi.org/10.5281/zenodo.16728769}.

\bibliography{MyEndNoteLibrary}


\begin{thebibliography}{54}
\ifx \bisbn   \undefined \def \bisbn  #1{ISBN #1}\fi
\ifx \binits  \undefined \def \binits#1{#1}\fi
\ifx \bauthor  \undefined \def \bauthor#1{#1}\fi
\ifx \batitle  \undefined \def \batitle#1{#1}\fi
\ifx \bjtitle  \undefined \def \bjtitle#1{#1}\fi
\ifx \bvolume  \undefined \def \bvolume#1{\textbf{#1}}\fi
\ifx \byear  \undefined \def \byear#1{#1}\fi
\ifx \bissue  \undefined \def \bissue#1{#1}\fi
\ifx \bfpage  \undefined \def \bfpage#1{#1}\fi
\ifx \blpage  \undefined \def \blpage #1{#1}\fi
\ifx \burl  \undefined \def \burl#1{\textsf{#1}}\fi
\ifx \doiurl  \undefined \def \doiurl#1{\url{https://doi.org/#1}}\fi
\ifx \betal  \undefined \def \betal{\textit{et al.}}\fi
\ifx \binstitute  \undefined \def \binstitute#1{#1}\fi
\ifx \binstitutionaled  \undefined \def \binstitutionaled#1{#1}\fi
\ifx \bctitle  \undefined \def \bctitle#1{#1}\fi
\ifx \beditor  \undefined \def \beditor#1{#1}\fi
\ifx \bpublisher  \undefined \def \bpublisher#1{#1}\fi
\ifx \bbtitle  \undefined \def \bbtitle#1{#1}\fi
\ifx \bedition  \undefined \def \bedition#1{#1}\fi
\ifx \bseriesno  \undefined \def \bseriesno#1{#1}\fi
\ifx \blocation  \undefined \def \blocation#1{#1}\fi
\ifx \bsertitle  \undefined \def \bsertitle#1{#1}\fi
\ifx \bsnm \undefined \def \bsnm#1{#1}\fi
\ifx \bsuffix \undefined \def \bsuffix#1{#1}\fi
\ifx \bparticle \undefined \def \bparticle#1{#1}\fi
\ifx \barticle \undefined \def \barticle#1{#1}\fi
\bibcommenthead
\ifx \bconfdate \undefined \def \bconfdate #1{#1}\fi
\ifx \botherref \undefined \def \botherref #1{#1}\fi
\ifx \url \undefined \def \url#1{\textsf{#1}}\fi
\ifx \bchapter \undefined \def \bchapter#1{#1}\fi
\ifx \bbook \undefined \def \bbook#1{#1}\fi
\ifx \bcomment \undefined \def \bcomment#1{#1}\fi
\ifx \oauthor \undefined \def \oauthor#1{#1}\fi
\ifx \citeauthoryear \undefined \def \citeauthoryear#1{#1}\fi
\ifx \endbibitem  \undefined \def \endbibitem {}\fi
\ifx \bconflocation  \undefined \def \bconflocation#1{#1}\fi
\ifx \arxivurl  \undefined \def \arxivurl#1{\textsf{#1}}\fi
\csname PreBibitemsHook\endcsname

\bibitem[\protect\citeauthoryear{Schöps et~al.}{2020}]{Schops2020}
\begin{bchapter}
\bauthor{\bsnm{Schöps}, \binits{T.}},
\bauthor{\bsnm{Larsson}, \binits{V.}},
\bauthor{\bsnm{Pollefeys}, \binits{M.}},
\bauthor{\bsnm{Sattler}, \binits{T.}}:
\bctitle{Why having 10,000 parameters in your camera model is better than twelve}.
In: \bbtitle{Proc. IEEE Conf. Comput. Vis. Pattern Recognit.},
pp. \bfpage{2535}--\blpage{2544}
(\byear{2020})
\end{bchapter}
\endbibitem

\bibitem[\protect\citeauthoryear{Zhang}{2000}]{Zhang2000}
\begin{barticle}
\bauthor{\bsnm{Zhang}, \binits{Z.}}:
\batitle{A flexible new technique for camera calibration}.
\bjtitle{IEEE Trans. Pattern Anal. Mach. Intell.}
\bvolume{22}(\bissue{11}),
\bfpage{1330}--\blpage{1334}
(\byear{2000})
\end{barticle}
\endbibitem

\bibitem[\protect\citeauthoryear{Brown}{1971}]{Brown1971}
\begin{barticle}
\bauthor{\bsnm{Brown}, \binits{D.}}:
\batitle{Close-range camera calibration}.
\bjtitle{Photo. Eng.}
\bvolume{37}(\bissue{8}),
\bfpage{855}--\blpage{866}
(\byear{1971})
\end{barticle}
\endbibitem

\bibitem[\protect\citeauthoryear{Kannala and Brandt}{2006}]{Kannala2006}
\begin{barticle}
\bauthor{\bsnm{Kannala}, \binits{J.}},
\bauthor{\bsnm{Brandt}, \binits{S.S.}}:
\batitle{A generic camera model and calibration method for conventional, wide-angle, and fish-eye lenses}.
\bjtitle{IEEE Trans. Pattern Anal. Mach. Intell.}
\bvolume{28}(\bissue{8}),
\bfpage{1335}--\blpage{1340}
(\byear{2006})
\end{barticle}
\endbibitem

\bibitem[\protect\citeauthoryear{Scaramuzza et~al.}{2006}]{Scaramuzza2006}
\begin{bchapter}
\bauthor{\bsnm{Scaramuzza}, \binits{D.}},
\bauthor{\bsnm{Martinelli}, \binits{A.}},
\bauthor{\bsnm{Siegwart}, \binits{R.}}:
\bctitle{A toolbox for easily calibrating omnidirectional cameras}.
In: \bbtitle{Proc. IEEE Int. Conf. Intell. Rob. Syst.},
pp. \bfpage{5695}--\blpage{5701}
(\byear{2006})
\end{bchapter}
\endbibitem

\bibitem[\protect\citeauthoryear{Urban et~al.}{2015}]{Urban2015}
\begin{barticle}
\bauthor{\bsnm{Urban}, \binits{S.}},
\bauthor{\bsnm{Leitloff}, \binits{J.}},
\bauthor{\bsnm{Hinz}, \binits{S.}}:
\batitle{Improved wide-angle, fisheye and omnidirectional camera calibration}.
\bjtitle{ISPRS J. Photogramm. Remote Sens.}
\bvolume{108}(\bissue{7}),
\bfpage{72}--\blpage{79}
(\byear{2015})
\end{barticle}
\endbibitem

\bibitem[\protect\citeauthoryear{Mei and Rives}{2007}]{Mei2007}
\begin{bchapter}
\bauthor{\bsnm{Mei}, \binits{C.}},
\bauthor{\bsnm{Rives}, \binits{P.}}:
\bctitle{Single view point omnidirectional camera calibration from planar grids}.
In: \bbtitle{Proc. IEEE Int. Conf. Robot. Autom.},
pp. \bfpage{3945}--\blpage{3950}
(\byear{2007})
\end{bchapter}
\endbibitem

\bibitem[\protect\citeauthoryear{Usenko et~al.}{2018}]{Usenko2018}
\begin{bchapter}
\bauthor{\bsnm{Usenko}, \binits{V.}},
\bauthor{\bsnm{Demmel}, \binits{N.}},
\bauthor{\bsnm{Cremers}, \binits{D.}}:
\bctitle{The double sphere camera model}.
In: \bbtitle{Proc. Int. Conf. 3D Vis. 3DV},
pp. \bfpage{552}--\blpage{560}
(\byear{2018})
\end{bchapter}
\endbibitem

\bibitem[\protect\citeauthoryear{Khomutenko et~al.}{2016}]{Khomutenko2016}
\begin{barticle}
\bauthor{\bsnm{Khomutenko}, \binits{B.}},
\bauthor{\bsnm{Garcia}, \binits{G.}},
\bauthor{\bsnm{Martinet}, \binits{P.}}:
\batitle{An enhanced unified camera model}.
\bjtitle{IEEE Rob. Autom. Lett.}
\bvolume{1}(\bissue{1}),
\bfpage{137}--\blpage{144}
(\byear{2016})
\end{barticle}
\endbibitem

\bibitem[\protect\citeauthoryear{Pan et~al.}{2022}]{Pan2022}
\begin{bchapter}
\bauthor{\bsnm{Pan}, \binits{L.}},
\bauthor{\bsnm{Pollefeys}, \binits{M.}},
\bauthor{\bsnm{Larsson}, \binits{V.}}:
\bctitle{Camera pose estimation using implicit distortion models}.
In: \bbtitle{Proc. IEEE Conf. Comput. Vis. Pattern Recognit.},
pp. \bfpage{12809}--\blpage{12818}
(\byear{2022})
\end{bchapter}
\endbibitem

\bibitem[\protect\citeauthoryear{Dunne et~al.}{2007}]{Dunne2007Efficient}
\begin{bchapter}
\bauthor{\bsnm{Dunne}, \binits{A.K.}},
\bauthor{\bsnm{Mallon}, \binits{J.}},
\bauthor{\bsnm{Whelan}, \binits{P.F.}}:
\bctitle{Efficient generic calibration method for general cameras with single centre of projection}.
In: \bbtitle{Proc. IEEE Int. Conf. Comput. Vis.},
pp. \bfpage{1}--\blpage{8}
(\byear{2007})
\end{bchapter}
\endbibitem

\bibitem[\protect\citeauthoryear{Rosebrock and Wahl}{2012a}]{Rosebrock2012Generic}
\begin{bchapter}
\bauthor{\bsnm{Rosebrock}, \binits{D.}},
\bauthor{\bsnm{Wahl}, \binits{F.M.}}:
\bctitle{Generic camera calibration and modeling using spline surfaces}.
In: \bbtitle{Proc. IEEE Intell. Vehicles Symp.},
pp. \bfpage{51}--\blpage{56}
(\byear{2012})
\end{bchapter}
\endbibitem

\bibitem[\protect\citeauthoryear{Rosebrock and Wahl}{2012b}]{Rosebrock2012Complete}
\begin{bchapter}
\bauthor{\bsnm{Rosebrock}, \binits{D.}},
\bauthor{\bsnm{Wahl}, \binits{F.M.}}:
\bctitle{Complete generic camera calibration and modeling using spline surfaces}.
In: \bbtitle{Proc. Asian Conf. Comput. Vis.},
pp. \bfpage{487}--\blpage{498}.
\bpublisher{Springer},
\blocation{Berlin, Heidelberg}
(\byear{2012})
\end{bchapter}
\endbibitem

\bibitem[\protect\citeauthoryear{Hagemann et~al.}{2022}]{Hagemann2022}
\begin{barticle}
\bauthor{\bsnm{Hagemann}, \binits{A.}},
\bauthor{\bsnm{Knorr}, \binits{M.}},
\bauthor{\bsnm{Janssen}, \binits{H.}},
\bauthor{\bsnm{Stiller}, \binits{C.}}:
\batitle{Inferring bias and uncertainty in camera calibration}.
\bjtitle{Int. J. Comput. Vis.}
\bvolume{130}(\bissue{1}),
\bfpage{1}--\blpage{16}
(\byear{2022})
\end{barticle}
\endbibitem

\bibitem[\protect\citeauthoryear{Wei and Ma}{1994}]{Guoqing1994}
\begin{barticle}
\bauthor{\bsnm{Wei}, \binits{G.}},
\bauthor{\bsnm{Ma}, \binits{S.D.}}:
\batitle{Implicit and explicit camera calibration: Theory and experiments}.
\bjtitle{IEEE Trans. Pattern Anal. Mach. Intell.}
\bvolume{16}(\bissue{5}),
\bfpage{469}--\blpage{480}
(\byear{1994})
\end{barticle}
\endbibitem

\bibitem[\protect\citeauthoryear{Hartley and Kang}{2007}]{Hartley2007}
\begin{barticle}
\bauthor{\bsnm{Hartley}, \binits{R.}},
\bauthor{\bsnm{Kang}, \binits{S.B.}}:
\batitle{Parameter-free radial distortion correction with center of distortion estimation}.
\bjtitle{IEEE Trans. Pattern Anal. Mach. Intell.}
\bvolume{29}(\bissue{8}),
\bfpage{1309}--\blpage{1321}
(\byear{2007})
\end{barticle}
\endbibitem

\bibitem[\protect\citeauthoryear{Lochman et~al.}{2021}]{Lochman2021}
\begin{bchapter}
\bauthor{\bsnm{Lochman}, \binits{Y.}},
\bauthor{\bsnm{Liepieshov}, \binits{K.}},
\bauthor{\bsnm{Chen}, \binits{J.}},
\bauthor{\bsnm{Perdoch}, \binits{M.}},
\bauthor{\bsnm{Zach}, \binits{C.}},
\bauthor{\bsnm{Pritts}, \binits{J.}}:
\bctitle{Babelcalib: A universal approach to calibrating central cameras}.
In: \bbtitle{Proc. IEEE Int. Conf. Comput. Vis.},
pp. \bfpage{15253}--\blpage{15262}
(\byear{2021})
\end{bchapter}
\endbibitem

\bibitem[\protect\citeauthoryear{Tang et~al.}{2017}]{Tang2017}
\begin{barticle}
\bauthor{\bsnm{Tang}, \binits{Z.}},
\bauthor{\bsnm{Gioi}, \binits{R.}},
\bauthor{\bsnm{Monasse}, \binits{P.}},
\bauthor{\bsnm{Morel}, \binits{J.-M.}}:
\batitle{A precision analysis of camera distortion models}.
\bjtitle{IEEE Trans. Image Process.}
\bvolume{26}(\bissue{6}),
\bfpage{2694}--\blpage{2704}
(\byear{2017})
\end{barticle}
\endbibitem

\bibitem[\protect\citeauthoryear{Larsson et~al.}{2019}]{Larsson2019}
\begin{bchapter}
\bauthor{\bsnm{Larsson}, \binits{V.}},
\bauthor{\bsnm{Sattler}, \binits{T.}},
\bauthor{\bsnm{Kukelova}, \binits{Z.}},
\bauthor{\bsnm{Pollefeys}, \binits{M.}}:
\bctitle{Revisiting radial distortion absolute pose}.
In: \bbtitle{Proc. IEEE Int. Conf. Comput. Vis.},
pp. \bfpage{1062}--\blpage{1071}
(\byear{2019})
\end{bchapter}
\endbibitem

\bibitem[\protect\citeauthoryear{Polic et~al.}{2020}]{Polic2020}
\begin{bchapter}
\bauthor{\bsnm{Polic}, \binits{M.}},
\bauthor{\bsnm{Steidl}, \binits{S.}},
\bauthor{\bsnm{Albl}, \binits{C.}},
\bauthor{\bsnm{Kukelova}, \binits{Z.}},
\bauthor{\bsnm{Pajdla}, \binits{T.}}:
\bctitle{Uncertainty based camera model selection}.
In: \bbtitle{Proc. IEEE Conf. Comput. Vis. Pattern Recognit.},
pp. \bfpage{5991}--\blpage{6000}
(\byear{2020})
\end{bchapter}
\endbibitem

\bibitem[\protect\citeauthoryear{Davison and Hinkley}{1997}]{Davison1997}
\begin{bbook}
\bauthor{\bsnm{Davison}, \binits{A.}},
\bauthor{\bsnm{Hinkley}, \binits{D.}}:
\bbtitle{Bootstrap {Methods} and {Their} {Application}}.
\bpublisher{Cambridge Univ. Press},
\blocation{Cambridge}
(\byear{1997})
\end{bbook}
\endbibitem

\bibitem[\protect\citeauthoryear{Hartley}{2003}]{Hartley2003}
\begin{bbook}
\bauthor{\bsnm{Hartley}, \binits{R.}}:
\bbtitle{{Multi-View Geometry in Computer Vision}}.
\bpublisher{Cambridge Univ. Press},
\blocation{New York}
(\byear{2003})
\end{bbook}
\endbibitem

\bibitem[\protect\citeauthoryear{Beck and Stiller}{2018}]{Beck2018}
\begin{bchapter}
\bauthor{\bsnm{Beck}, \binits{J.}},
\bauthor{\bsnm{Stiller}, \binits{C.}}:
\bctitle{Generalized b-spline camera model}.
In: \bbtitle{Proc. IEEE Intell. Vehicles Symp.},
pp. \bfpage{2137}--\blpage{2142}
(\byear{2018})
\end{bchapter}
\endbibitem

\bibitem[\protect\citeauthoryear{Rojtberg and Kuijper}{2018}]{Rojtberg2018}
\begin{bchapter}
\bauthor{\bsnm{Rojtberg}, \binits{P.}},
\bauthor{\bsnm{Kuijper}, \binits{A.}}:
\bctitle{Efficient pose selection for interactive camera calibration}.
In: \bbtitle{Proc. IEEE Int. Symp. Mix. Augment. Real. ISMAR},
pp. \bfpage{31}--\blpage{36}
(\byear{2018})
\end{bchapter}
\endbibitem

\bibitem[\protect\citeauthoryear{Peng and Sturm}{2019}]{Peng2019}
\begin{bchapter}
\bauthor{\bsnm{Peng}, \binits{S.}},
\bauthor{\bsnm{Sturm}, \binits{P.}}:
\bctitle{Calibration wizard: A guidance system for camera calibration based on modelling geometric and corner uncertainty}.
In: \bbtitle{Proc. IEEE Int. Conf. Comput. Vis.},
pp. \bfpage{1497}--\blpage{1505}
(\byear{2019})
\end{bchapter}
\endbibitem

\bibitem[\protect\citeauthoryear{Ren and Hu}{2021}]{Ren2021}
\begin{bchapter}
\bauthor{\bsnm{Ren}, \binits{Y.}},
\bauthor{\bsnm{Hu}, \binits{F.}}:
\bctitle{Camera calibration with pose guidance}.
In: \bbtitle{Proc. IEEE Int. Conf. Acoust. Speech Signal Process.},
pp. \bfpage{2180}--\blpage{2184}
(\byear{2021})
\end{bchapter}
\endbibitem

\bibitem[\protect\citeauthoryear{Dunne et~al.}{2007}]{Dunne2007comparison}
\begin{bchapter}
\bauthor{\bsnm{Dunne}, \binits{A.}},
\bauthor{\bsnm{Mallon}, \binits{J.}},
\bauthor{\bsnm{Whelan}, \binits{P.}}:
\bctitle{A comparison of new generic camera calibration with the standard parametric approach}.
In: \bbtitle{10th IAPR Conference on Machine Vision Applications (MVA)},
pp. \bfpage{114}--\blpage{117}
(\byear{2007})
\end{bchapter}
\endbibitem

\bibitem[\protect\citeauthoryear{Conrady}{1919}]{Conrady1919}
\begin{barticle}
\bauthor{\bsnm{Conrady}, \binits{A.E.}}:
\batitle{Lens-systems, decentered}.
\bjtitle{Monthly Notices Roy. Astron. Soc.}
\bvolume{79}(\bissue{5}),
\bfpage{384}--\blpage{390}
(\byear{1919})
\end{barticle}
\endbibitem

\bibitem[\protect\citeauthoryear{Brown}{1966}]{Brown1966}
\begin{barticle}
\bauthor{\bsnm{Brown}, \binits{D.C.}}:
\batitle{Decentering the distortion of lenses}.
\bjtitle{Photogramm. Eng.}
\bvolume{32},
\bfpage{444}--\blpage{462}
(\byear{1966})
\end{barticle}
\endbibitem

\bibitem[\protect\citeauthoryear{Caprile and Torre}{1990}]{Caprile1990}
\begin{barticle}
\bauthor{\bsnm{Caprile}, \binits{B.}},
\bauthor{\bsnm{Torre}, \binits{V.}}:
\batitle{Using vanishing points for camera calibration}.
\bjtitle{Int. J. Comput. Vis.}
\bvolume{4}(\bissue{3}),
\bfpage{127}--\blpage{139}
(\byear{1990})
\end{barticle}
\endbibitem

\bibitem[\protect\citeauthoryear{Hartley}{1994}]{Hartley1994}
\begin{bchapter}
\bauthor{\bsnm{Hartley}, \binits{R.I.}}:
\bctitle{Self-calibration from multiple views with a rotating camera}.
In: \bbtitle{Proc. Eur. Conf. Comput. Vis.},
vol. \bseriesno{I 3},
pp. \bfpage{471}--\blpage{478}
(\byear{1994})
\end{bchapter}
\endbibitem

\bibitem[\protect\citeauthoryear{Tsai}{1987}]{Tsai1987}
\begin{barticle}
\bauthor{\bsnm{Tsai}, \binits{R.}}:
\batitle{A versatile camera calibration technique for high-accuracy {3D} machine vision metrology using off-the-shelf tv cameras and lenses}.
\bjtitle{IEEE J. Robot. Autom.}
\bvolume{3}(\bissue{4}),
\bfpage{323}--\blpage{344}
(\byear{1987})
\end{barticle}
\endbibitem

\bibitem[\protect\citeauthoryear{Fitzgibbon}{2001}]{Fitzgibbon2001}
\begin{bchapter}
\bauthor{\bsnm{Fitzgibbon}, \binits{A.W.}}:
\bctitle{Simultaneous linear estimation of multiple view geometry and lens distortion}.
In: \bbtitle{Proc. IEEE Conf. Comput. Vis. Pattern Recognit.},
vol. \bseriesno{1},
pp. \bfpage{125}--\blpage{132}
(\byear{2001})
\end{bchapter}
\endbibitem

\bibitem[\protect\citeauthoryear{Brito et~al.}{2013}]{Brito2013}
\begin{bchapter}
\bauthor{\bsnm{Brito}, \binits{J.H.}},
\bauthor{\bsnm{Angst}, \binits{R.}},
\bauthor{\bsnm{Köser}, \binits{K.}},
\bauthor{\bsnm{Zach}, \binits{C.}},
\bauthor{\bsnm{Branco}, \binits{P.}},
\bauthor{\bsnm{Ferreira}, \binits{M.J.}}, \betal:
\bctitle{Unknown radial distortion centers in multiple view geometry problems}.
In: \bbtitle{Proc. Asian Conf. Comput. Vis.},
pp. \bfpage{136}--\blpage{149}.
\bpublisher{Springer},
\blocation{Berlin, Heidelberg}
(\byear{2013})
\end{bchapter}
\endbibitem

\bibitem[\protect\citeauthoryear{Brito}{2017}]{Brito2017}
\begin{barticle}
\bauthor{\bsnm{Brito}, \binits{J.}}:
\batitle{Autocalibration for structure from motion}.
\bjtitle{Comput. Vis. Image Underst.}
\bvolume{157}(\bissue{3}),
\bfpage{240}--\blpage{254}
(\byear{2017})
\end{barticle}
\endbibitem

\bibitem[\protect\citeauthoryear{Bujnak et~al.}{2010}]{Bujnak2010}
\begin{bchapter}
\bauthor{\bsnm{Bujnak}, \binits{M.}},
\bauthor{\bsnm{Kukelova}, \binits{Z.}},
\bauthor{\bsnm{Pajdla}, \binits{T.}}:
\bctitle{New efficient solution to the absolute pose problem for camera with unknown focal length and radial distortion}.
In: \bbtitle{Proc. Asian Conf. Comput. Vis.},
pp. \bfpage{11}--\blpage{24}.
\bpublisher{Springer},
\blocation{Berlin, Heidelberg}
(\byear{2010})
\end{bchapter}
\endbibitem

\bibitem[\protect\citeauthoryear{Kukelova et~al.}{2013}]{Kukelova2013}
\begin{bchapter}
\bauthor{\bsnm{Kukelova}, \binits{Z.}},
\bauthor{\bsnm{Bujnak}, \binits{M.}},
\bauthor{\bsnm{Pajdla}, \binits{T.}}:
\bctitle{Real-time solution to the absolute pose problem with unknown radial distortion and focal length}.
In: \bbtitle{Proc. IEEE Int. Conf. Comput. Vis.},
pp. \bfpage{2816}--\blpage{2823}
(\byear{2013})
\end{bchapter}
\endbibitem

\bibitem[\protect\citeauthoryear{Kukelova et~al.}{2015}]{Kukelova2015}
\begin{bchapter}
\bauthor{\bsnm{Kukelova}, \binits{Z.}},
\bauthor{\bsnm{Heller}, \binits{J.}},
\bauthor{\bsnm{Bujnak}, \binits{M.}},
\bauthor{\bsnm{Fitzgibbon}, \binits{A.}},
\bauthor{\bsnm{Pajdla}, \binits{T.}}:
\bctitle{Efficient solution to the epipolar geometry for radially distorted cameras}.
In: \bbtitle{Proc. IEEE Int. Conf. Comput. Vis.},
pp. \bfpage{2309}--\blpage{2317}
(\byear{2015})
\end{bchapter}
\endbibitem

\bibitem[\protect\citeauthoryear{Camposeco et~al.}{2015}]{Camposeco2015}
\begin{bchapter}
\bauthor{\bsnm{Camposeco}, \binits{F.}},
\bauthor{\bsnm{Sattler}, \binits{T.}},
\bauthor{\bsnm{Pollefeys}, \binits{M.}}:
\bctitle{Non-parametric structure-based calibration of radially symmetric cameras}.
In: \bbtitle{Proc. IEEE Int. Conf. Comput. Vis.},
pp. \bfpage{2192}--\blpage{2200}
(\byear{2015})
\end{bchapter}
\endbibitem

\bibitem[\protect\citeauthoryear{Miraldo and Araujo}{2013}]{Miraldo2013}
\begin{barticle}
\bauthor{\bsnm{Miraldo}, \binits{P.}},
\bauthor{\bsnm{Araujo}, \binits{H.}}:
\batitle{Calibration of smooth camera models}.
\bjtitle{IEEE Trans. Pattern Anal. Mach. Intell.}
\bvolume{35}(\bissue{9}),
\bfpage{2091}--\blpage{2103}
(\byear{2013})
\end{barticle}
\endbibitem

\bibitem[\protect\citeauthoryear{Grossberg and Nayar}{2001}]{Grossberg2001}
\begin{bchapter}
\bauthor{\bsnm{Grossberg}, \binits{M.D.}},
\bauthor{\bsnm{Nayar}, \binits{S.K.}}:
\bctitle{A general imaging model and a method for finding its parameters}.
In: \bbtitle{Proc. IEEE Int. Conf. Comput. Vis.},
vol. \bseriesno{2},
pp. \bfpage{108}--\blpage{115}
(\byear{2001})
\end{bchapter}
\endbibitem

\bibitem[\protect\citeauthoryear{Sturm and Ramalingam}{2004}]{Sturm2004}
\begin{bchapter}
\bauthor{\bsnm{Sturm}, \binits{P.F.}},
\bauthor{\bsnm{Ramalingam}, \binits{S.}}:
\bctitle{A generic concept for camera calibration}.
In: \bbtitle{Proc. Eur. Conf. Comput. Vis.},
pp. \bfpage{1}--\blpage{13}
(\byear{2004})
\end{bchapter}
\endbibitem

\bibitem[\protect\citeauthoryear{Ramalingam and Sturm}{2017}]{Ramalingam2017}
\begin{barticle}
\bauthor{\bsnm{Ramalingam}, \binits{S.}},
\bauthor{\bsnm{Sturm}, \binits{P.}}:
\batitle{A unifying model for camera calibration}.
\bjtitle{IEEE Trans. Pattern Anal. Mach. Intell.}
\bvolume{39}(\bissue{7}),
\bfpage{1309}--\blpage{1319}
(\byear{2017})
\end{barticle}
\endbibitem

\bibitem[\protect\citeauthoryear{Sturm and Ramalingam}{2003}]{Sturm2003}
\begin{botherref}
\oauthor{\bsnm{Sturm}, \binits{P.}},
\oauthor{\bsnm{Ramalingam}, \binits{S.}}:
A generic calibration concept: Theory and algorithms.
Thesis,
INRIA
(2003)
\end{botherref}
\endbibitem

\bibitem[\protect\citeauthoryear{Ramalingam}{2006}]{Ramalingam2006}
\begin{botherref}
\oauthor{\bsnm{Ramalingam}, \binits{S.}}:
Generic imaging models: Calibration and {3D} reconstruction algorithms.
Thesis,
Institut National Polytechnique de Grenoble-INPG
(2006)
\end{botherref}
\endbibitem

\bibitem[\protect\citeauthoryear{Brousseau and Roy}{2019}]{Brousseau2019}
\begin{bchapter}
\bauthor{\bsnm{Brousseau}, \binits{P.-A.}},
\bauthor{\bsnm{Roy}, \binits{S.}}:
\bctitle{Calibration of axial fisheye cameras through generic virtual central models}.
In: \bbtitle{Proc. IEEE Int. Conf. Comput. Vis.},
pp. \bfpage{4040}--\blpage{4048}
(\byear{2019})
\end{bchapter}
\endbibitem

\bibitem[\protect\citeauthoryear{Bergamasco et~al.}{2017}]{Bergamasco2017}
\begin{bchapter}
\bauthor{\bsnm{Bergamasco}, \binits{F.}},
\bauthor{\bsnm{Cosmo}, \binits{L.}},
\bauthor{\bsnm{Gasparetto}, \binits{A.}},
\bauthor{\bsnm{Albarelli}, \binits{A.}},
\bauthor{\bsnm{Torsello}, \binits{A.}}:
\bctitle{Parameter-free lens distortion calibration of central cameras}.
In: \bbtitle{Proc. IEEE Int. Conf. Comput. Vis.},
pp. \bfpage{3867}--\blpage{3875}
(\byear{2017})
\end{bchapter}
\endbibitem

\bibitem[\protect\citeauthoryear{Uhlig and Heizmann}{2020}]{Uhlig2020}
\begin{bchapter}
\bauthor{\bsnm{Uhlig}, \binits{D.}},
\bauthor{\bsnm{Heizmann}, \binits{M.}}:
\bctitle{A calibration method for the generalized imaging model with uncertain calibration target coordinates}.
In: \bbtitle{Proc. Asian Conf. Comput. Vis.},
pp. \bfpage{541}--\blpage{559}
(\byear{2020})
\end{bchapter}
\endbibitem

\bibitem[\protect\citeauthoryear{Lin et~al.}{2020}]{Lin2020}
\begin{bchapter}
\bauthor{\bsnm{Lin}, \binits{Y.}},
\bauthor{\bsnm{Larsson}, \binits{V.}},
\bauthor{\bsnm{Geppert}, \binits{M.}},
\bauthor{\bsnm{Kukelova}, \binits{Z.}},
\bauthor{\bsnm{Pollefeys}, \binits{M.}},
\bauthor{\bsnm{Sattler}, \binits{T.}}:
\bctitle{Infrastructure-based multi-camera calibration using radial projections}.
In: \bbtitle{Proc. Eur. Conf. Comput. Vis.},
pp. \bfpage{327}--\blpage{344}.
\bpublisher{Springer},
\blocation{Cham}
(\byear{2020})
\end{bchapter}
\endbibitem

\bibitem[\protect\citeauthoryear{Stewénius et~al.}{2006}]{Stewenius2006}
\begin{barticle}
\bauthor{\bsnm{Stewénius}, \binits{H.}},
\bauthor{\bsnm{Engels}, \binits{C.}},
\bauthor{\bsnm{Nistér}, \binits{D.}}:
\batitle{Recent developments on direct relative orientation}.
\bjtitle{ISPRS J. Photogramm. Remote Sens.}
\bvolume{60}(\bissue{4}),
\bfpage{284}--\blpage{294}
(\byear{2006})
\end{barticle}
\endbibitem

\bibitem[\protect\citeauthoryear{Kneip et~al.}{2012}]{Kneip2012}
\begin{bchapter}
\bauthor{\bsnm{Kneip}, \binits{L.}},
\bauthor{\bsnm{Siegwart}, \binits{R.}},
\bauthor{\bsnm{Pollefeys}, \binits{M.}}:
\bctitle{Finding the exact rotation between two images independently of the translation}.
In: \bbtitle{Proc. Eur. Conf. Comput. Vis.},
pp. \bfpage{696}--\blpage{709}.
\bpublisher{Springer},
\blocation{Berlin, Heidelberg}
(\byear{2012})
\end{bchapter}
\endbibitem

\bibitem[\protect\citeauthoryear{Levenberg}{1944}]{Levenberg1944}
\begin{barticle}
\bauthor{\bsnm{Levenberg}, \binits{K.}}:
\batitle{A method for the solution of certain non-linear problems in least squares}.
\bjtitle{Q. Appl. Math.}
\bvolume{2}(\bissue{2}),
\bfpage{164}--\blpage{168}
(\byear{1944})
\end{barticle}
\endbibitem

\bibitem[\protect\citeauthoryear{Marquardt}{1963}]{Marquardt1963}
\begin{barticle}
\bauthor{\bsnm{Marquardt}, \binits{D.W.}}:
\batitle{An algorithm for least-squares estimation of nonlinear parameters}.
\bjtitle{SIAM J. Appl. Math.}
\bvolume{11}(\bissue{2}),
\bfpage{431}--\blpage{441}
(\byear{1963})
\end{barticle}
\endbibitem

\bibitem[\protect\citeauthoryear{Kneip and Furgale}{2014}]{Kneip2014}
\begin{bchapter}
\bauthor{\bsnm{Kneip}, \binits{L.}},
\bauthor{\bsnm{Furgale}, \binits{P.}}:
\bctitle{{OpenGV}: A unified and generalized approach to real-time calibrated geometric vision}.
In: \bbtitle{Proc. IEEE Int. Conf. Robot. Autom.},
pp. \bfpage{1}--\blpage{8}
(\byear{2014})
\end{bchapter}
\endbibitem

\end{thebibliography}

\end{document}